\documentclass[10pt,journal,cspaper,compsoc]{IEEEtran}

\ifCLASSOPTIONcompsoc
\else
  \usepackage{cite}
\fi

\ifCLASSINFOpdf
\else
  \usepackage[dvips]{graphicx}
\fi
\usepackage{epsfig}

\usepackage[cmex10]{amsmath}
\usepackage{algorithmic}

\usepackage{array}

\usepackage{mdwmath}
\usepackage{mdwtab}
\usepackage{eqparbox}
\usepackage{multirow}

\usepackage{url}
\usepackage{amssymb}
\usepackage{amsthm,url}
\usepackage{color}
\usepackage[T1]{fontenc}

\newtheorem{prop}{Proposition}
\newtheorem{mydef}{Definition}

\begin{document}
%
\title{Object Proposal Generation using Two-Stage Cascade SVMs}

\author{Ziming~Zhang,
        and Philip~H.S.~Torr,~\IEEEmembership{Senior Member,~IEEE}
\IEEEcompsocitemizethanks{\IEEEcompsocthanksitem Dr. Z. Zhang is currently with the Department of Electrical and Computer Engineering, Boston University, Boston, MA 02215, US. Prof. P.H.S. Torr is with the Department of Engineering Science, University of Oxford, Oxford OX1 3PJ, UK.\protect\\
E-mail: zzhang14@bu.edu, philip.torr@eng.ox.ac.uk}
\thanks{}}

%
%

\markboth{IEEE Transaction on Pattern Analysis and Machine Intelligence}%
{Zhang & Torr: Object Proposal Generation using Two-Stage Cascade SVMs}
%


\IEEEcompsoctitleabstractindextext{%
\begin{abstract}
Object proposal algorithms have shown great promise as a first step for object recognition and detection. Good object proposal generation algorithms require high object recall rate as well as low computational cost, because generating object proposals is usually utilized as a preprocessing step. The problem of how to accelerate the object proposal generation and evaluation process without decreasing recall is thus of great interest. In this paper, we propose a new object proposal generation method using two-stage cascade SVMs, where in the first stage linear filters are learned for predefined quantized scales/aspect-ratios independently, and in the second stage a global linear classifier is learned across all the quantized scales/aspect-ratios for calibration, so that all the proposals can be compared properly. The proposals with highest scores are our final output. Specifically, we explain our scale/aspect-ratio quantization scheme, and investigate the effects of combinations of $\ell_1$ and $\ell_2$ regularizers in cascade SVMs with/without ranking constraints in learning. Comprehensive experiments on VOC2007 dataset are conducted, and our results achieve the state-of-the-art performance with high object recall rate and high computational efficiency. Besides, our method has been demonstrated to be suitable for not only class-specific but also generic object proposal generation.

\end{abstract}

}

\maketitle

\IEEEdisplaynotcompsoctitleabstractindextext

%
\IEEEpeerreviewmaketitle

\section{Introduction}
\IEEEPARstart{F}or object proposal generation, we are interested in providing a small set of windows ({\it i.e.} bounding boxes) containing object instances probably with high object recall as well as high computational efficiency. Recent research has demonstrated that object proposal, as a data pre-process step, can be involved successfully in complex computer vision systems to help reduce the computational cost significantly while achieving state-of-the-art performance, {\it e.g.}, in object recognition \cite{wei2014cnn} and object detection \cite{wang2013regionlets}. In these methods, a small number of object proposals are needed to summarize all the objects in images that will be utilized further by the methods. Therefore, the need to accelerate the evaluation process as well as achieving high object recall is thus becoming more important for a successful computer vision system, and this problem has been attracting more and more attention~\cite{Rahtu_iccv11,Alexe2012pami,vedaldi09multiple,bb25233,bb85658,bb44332,bb25231}. 
\begin{figure*}[t]
\begin{minipage}[b]{0.5\linewidth}
 \begin{center}
 \centerline{\includegraphics[width=\columnwidth]{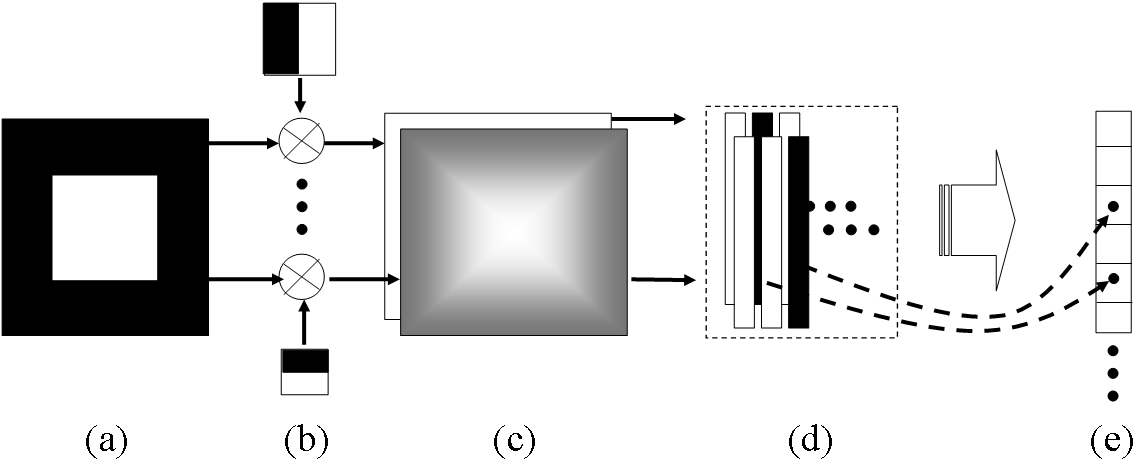}}
 \centerline{\footnotesize{(\lowercase\expandafter{\romannumeral1})}}
 \end{center}
\end{minipage}
\begin{minipage}[b]{0.5\linewidth}
 \begin{center}
 \centerline{\includegraphics[width=0.9\columnwidth]{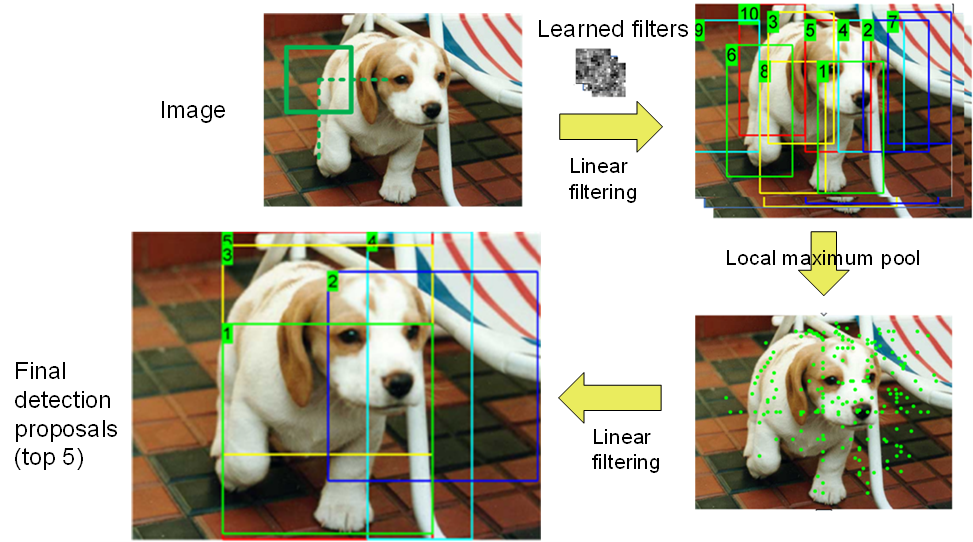}}
 \centerline{\footnotesize{(\lowercase\expandafter{\romannumeral2})}}
 \end{center}
\end{minipage}
\caption{{\footnotesize (\lowercase\expandafter{\romannumeral1}) Summary of our cascaded method. An image (a) is first convolved with a set of linear classifiers at varying scales/aspect-ratios (b) producing response images (c). Local maxima are extracted with non-max suppression from each response image, and the corresponding windows with top ranking scores are forwarded to the second stage of the cascade. Each proposed window is associated with a feature vector (d), and a second round of ranking orders these proposals (e) so that the true positives (marked as black) are pushed towards the top during training. Our method outputs the top ranking windows in this final ordering. (\lowercase\expandafter{\romannumeral2}) An example of generating proposals for detecting the dog in the image is shown, which explains the steps in (\lowercase\expandafter{\romannumeral1}). The numbers at the corners of windows in the bottom-left image indicate the ranks of windows.}}\label{fig:0}
\vspace{-0mm}
\end{figure*}

The main difficulties in object proposal generation are three-fold. First, the search space for localizing object proposals may be huge: Take a $(W\times H)$-pixel image for example. Considering all possible locations and scales/aspect-ratios in the image, the number of proposal candidates is roughly $O(W^2H^2)$. Second, finding a proper object representation is challenging, because of the change of imaging factors, huge intra-class and inter-class variations, many object categories, {\it etc}. Third, there may be multiple correct proposals for a single object instance of interest, leading to unnecessary spatial clusters of proposals. Thus, developing a highly computationally efficient yet accurate object proposal generation algorithm becomes very challenging.

Our previous work appeared as \cite{zhang_cvpr11}, where we proposed a {\em ranking} based two-stage cascade model for {\em class-specific} object proposal generation. To reduce the search space, we first proposed a scale/aspect-ratio quantization scheme in log-space, which guarantees any possible instance of objects in images can be located using at least one bounding box defined in the scheme. Then we learn linear classifiers at each stage in our cascade, all of whose scores can be utilized for ranking purposes. Ranking support vector machines (SVMs)~\cite{herbrich00ordinal} are used for ranking the proposals, which are normal SVMs with additional ranking constraints added into the learning to guarantee that some data should be classified with a higher score than others based on the ground-truth ranking order ({\it e.g.} those windows that better overlap the object ground-truth bounding boxes). In this way, our two-stage cascade enables us to incorporate variability in scale and aspect ratio by training a linear classifier for each quantized scale/aspect-ratio in the first stage, and another linear classifier in the second stage to calibrate the scores of the windows proposed from the first stage for final proposals. Finally, the usage of simple gradient features, linear convolution, and non-max suppression makes our method achieve the state-of-the-art performance in terms of object recall {\it vs.} number of proposals with high computational efficiency. Fig.~\ref{fig:0} summarizes the cascaded model and gives an example of generating proposals using this method.

This paper extends our work in \cite{zhang_cvpr11}. Specifically, we explain in detail our scale/aspect-ratio quantization scheme, investigate more general usage of cascade SVMs, and particularly demonstrate the capability of our method for {\em generic} object proposal generation. We explore the effects of combinations of $\ell_1$ and $\ell_2$ regularizers in two-stage cascade SVMs with/without ranking constraints in learning. Interestingly, our comprehensive comparison on the VOC2007 \cite{pascal-voc-2007} dataset suggests that {\em in general, the cascade where $\ell_1$-SVMs, which perform feature selection~\cite{bb61565}, are utilized in both stages without ranking constraints consistently works best}.

The rest of the paper is organized as follows. We first review some related work in Section \ref{sec:rw}. Then we explain the details of scale/aspect-ratio quantization scheme in Section \ref{sec:qs}. Next we formulate our two-stage cascade SVMs based on the proposed scale/aspect-ratio quantization scheme in Section \ref{sec:app}, and list some implementation details in Section \ref{sec:imp}. Finally Section~\ref{sec:exp} shows our experimental results and Section~\ref{sec:con} concludes the paper.

\section{Related Work}\label{sec:rw}
Various methods have been proposed to handle the proposal generation problem. Branch and bound techniques~\cite{bb25233,bb25231} for instance limit the number of windows that must be evaluated by pruning sets of windows whose response can be bounded. The efficiency of such methods is highly dependent on the strength of the bound, and the ease with which it can be evaluated, which can cause the method to offer limited speed-up for non-linear classifiers. Alternatively, cascade approaches~\cite{bb85658,bb44332,vedaldi09multiple} use weaker but faster classifiers in the initial stages to prune out negative examples, and only apply slower non-linear classifiers at the final stages. In~\cite{vedaldi09multiple} a fast linear SVM is used as a first step, while the jumping window approach~\cite{bb85658} builds an initial linear classifier by selecting pairs of discriminative visual words from their associated rectangle regions. Felzenszwalb et. al.~\cite{bb44331} propose a part-based cascaded model using a latent SVM in which part filters are only evaluated if a sufficient response is obtained from a global ``root'' filter, and \cite{bb44332} propose a combination of cascade and branch and bound techniques. Such approaches have been proved to be efficient, and have generated state-of-the-art results \cite{bb44331}. However, the fact that in \cite{bb44332} the decision scores for detections must be compared across the training data may limit the efficiency of the early cascade stages, where we only need to compare the scores of a classifier at any level of the cascade within a single image. Further, such approaches learn a single model which is applied at varying resolutions. Recent work~\cite{park2010multiresolution} strongly suggests that we should explicitly learn different detectors for different scales.

Several recent works \cite{Rahtu_iccv11,Alexe2012pami,Endres:2010:CIO:1888150.1888195,Uijlings13,Yanulevskaya14,Rantalankila14} are closely related to ours. Objectness measure \cite{Alexe2012pami} combines multiple visual cues to score the windows, and then produces the object proposals by sampling windows with high scores. Based on \cite{Alexe2012pami}, Rahtu {\it et. al.} \cite{Rahtu_iccv11} proposed another category-independent cascaded method for proposal generation, where the proposal candidates are sampled from super-pixels, which are generated using a segmentation method, according to a prior object localization distribution and then ranked using structured learning with learned features. The idea of grouping super-pixels/segments to generate proposals is also used in \cite{Endres:2010:CIO:1888150.1888195,Uijlings13,Yanulevskaya14,Rantalankila14} with different grouping criteria. More empirical comparisons of different proposed object proposal generation methods are presented in \cite{hosang2014good}.

The major differences between our method and these related work above are:
\begin{itemize}
\item From the view of features, our method only takes simple image gradients as features for learning and testing, while all of the related work above utilize multiple visual cues in images;
\item From the view of ranking proposals, our method utilizes the classification scores ({\it i.e.} margins) generated by the learned linear classifiers, rather than the scores from super-pixels \cite{Endres:2010:CIO:1888150.1888195}, prior object localization distributions \cite{Rahtu_iccv11}, or the combination of multiple visual cues \cite{Alexe2012pami,Rahtu_iccv11}, which involves more heuristics in general;
\item From the view of learning, our method formulates the problem using the cascade SVM framework, which is much easier to understand and implement.
\end{itemize}

As a result, our method achieves state-of-the-art performance with high object recall and high computational efficiency.

\section{Scale/Aspect-ratio Quantization Scheme}\label{sec:qs}
\subsection{Preliminaries}
Before explaining the details of our scale/aspect-ratio quantization scheme, we first introduce some definitions that are used later.

\begin{mydef}[\textbf{Bounding Box Overlap Score}]\label{def:ch1-overlap}
The {\em overlap score} between a bounding box $s$ and a ground-truth bounding box of an object $t$, $o(s,t)$, is defined as their intersection area divided by their union area. Clearly, $0\leq o(s,t)\leq1$, and the higher $o(s,t)$ is, the better the localization of the object $t$ with the bounding box $s$ is.
\end{mydef}
\begin{mydef}[\textbf{$\eta$-Accuracy}]\label{def:ch1-accuracy}
We say that a window $s\in \mathcal{S}$ can be localized by another window $t\in \mathcal{T}$ to {\em $\eta$-accuracy} if $o(s,t)\geq\eta$, $(0\leq\eta\leq1)$.
\end{mydef}
\begin{mydef}[\textbf{Maximum Overlap}]\label{def:ch3-maximum-overlap}
Given an image $I$ and the ground-truth bounding boxes of multiple objects $g_{1\cdots m_I}$ in $I$, the {\em maximum overlap} of a window $s$ in $I$ is defined as $o_s = \max_{i\in\{1,\cdots,m_I\}}o(s,g_i)$, where $o(s,g_i)$ denotes the overlap score between $s$ and $g_i$.
\end{mydef}
\begin{mydef}[\textbf{Correct Object Proposals}]\label{def:ch3-correct-proposal}
Given an overlap score threshold $\eta$, a window $s$ is considered as a {\em correct object proposal} in an image if and only if $o_s\geq\eta$.
\end{mydef}
\begin{mydef}[\textbf{Quantized Scale/Aspect-ratio}]\label{def:ch3-qs}
Given an overlap score threshold $\eta$, a window $s$ in an image can be quantized into a quantized scale/aspect-ratio $\mathcal{T}$ if and only if $\exists t\in\mathcal{T}$ such that $s$ can be localized to $\eta$-accuracy, where $t$ is a window with the quantized scale/aspect-ratio.
\end{mydef}

\subsection{Quantization Scheme}

\begin{figure}[t]
\begin{center}
\centerline{\includegraphics[width=\columnwidth]{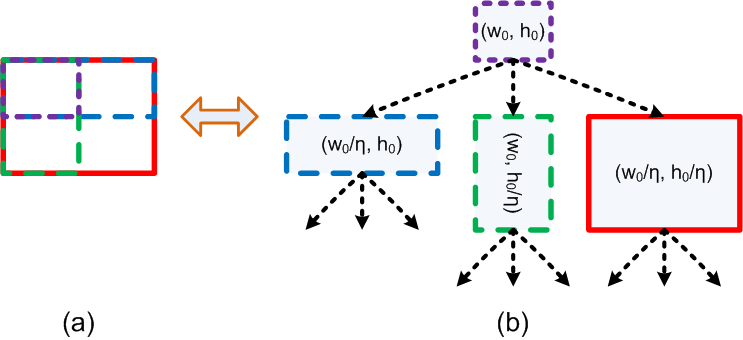}}\vspace{0mm}
\caption{{\footnotesize Illustration of hierarchical representation of our scale/aspect-ratio quantization scheme with overlap threshold $\eta=0.5$. (a) superimposes the four window scales in a mini-quantization scheme, and (b) unfolds the scales into a tree structure. The relative widths and heights of the windows are represented by the $(w,h)$ pairs. Such a hierarchy can represent all windows to $\eta$-accuracy.}}\label{fig:1}
\end{center}
\vspace{-0mm}
\end{figure}

\begin{figure*}[t]
\begin{minipage}[b]{0.33\linewidth}
 \begin{center}
 \centerline{\includegraphics[width=1.12\columnwidth]{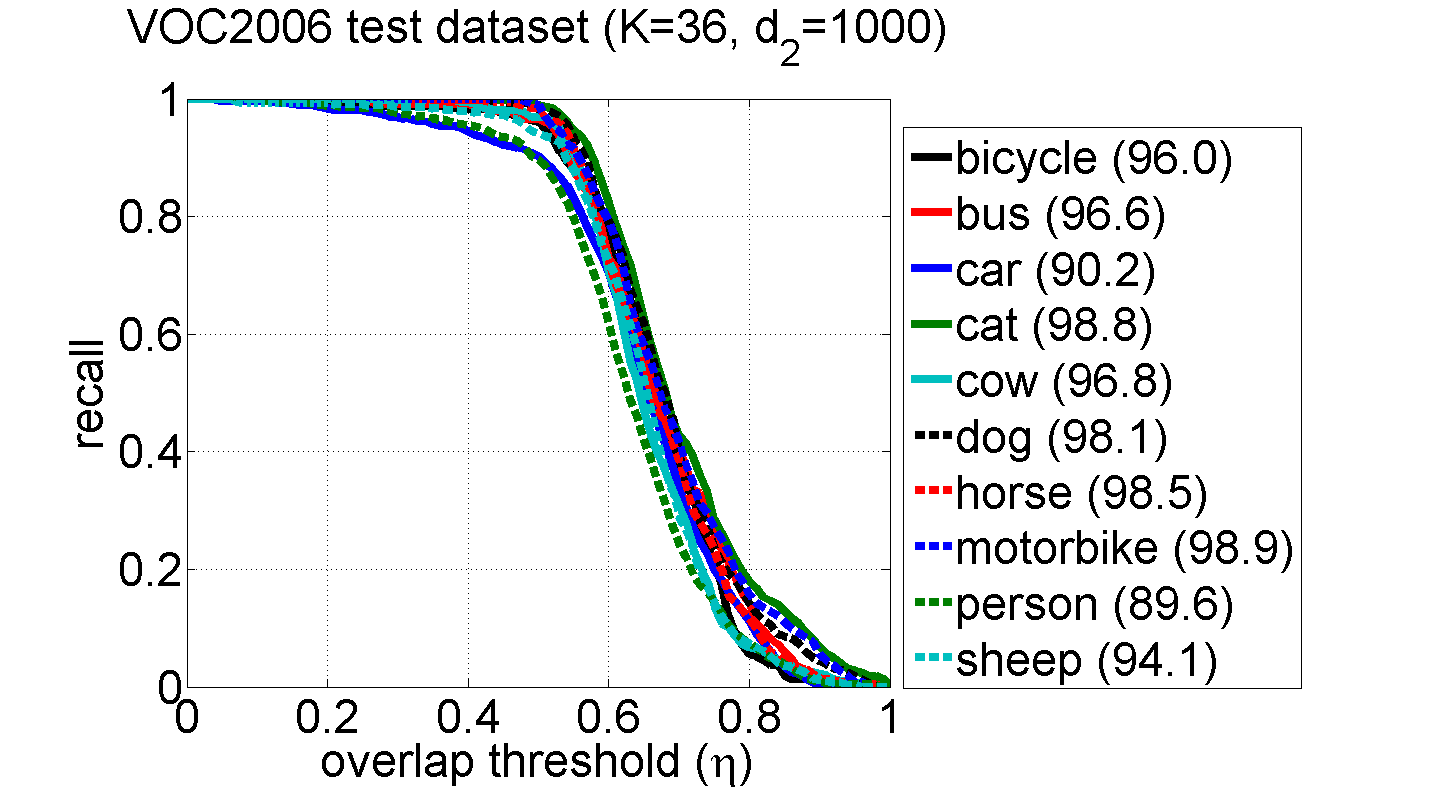}}
 \end{center}
\end{minipage}
\begin{minipage}[b]{0.33\linewidth}
 \begin{center}
 \centerline{\includegraphics[width=1.12\columnwidth]{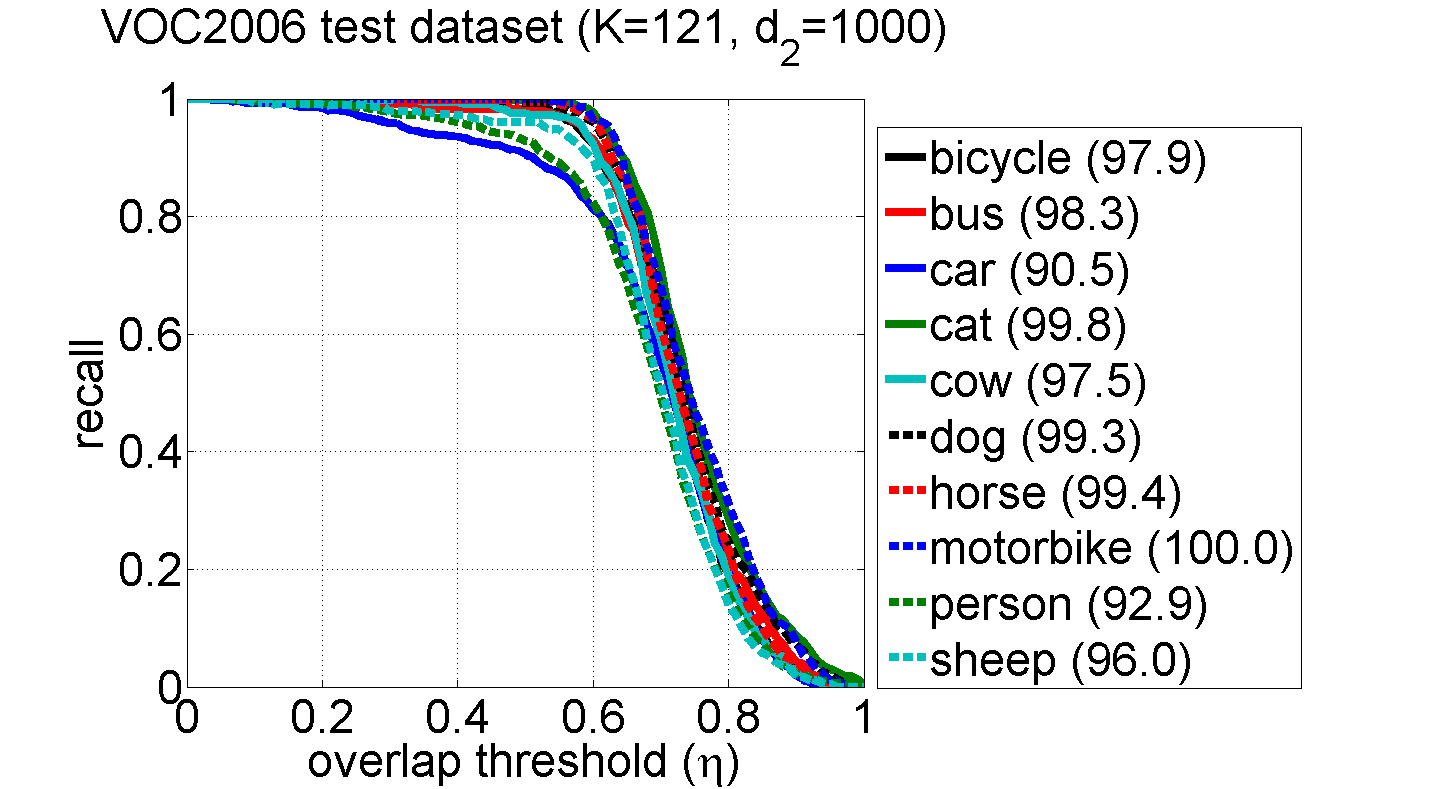}}
 \end{center}
\end{minipage}
\begin{minipage}[b]{0.33\linewidth}
 \begin{center}
 \centerline{\includegraphics[width=1.12\columnwidth]{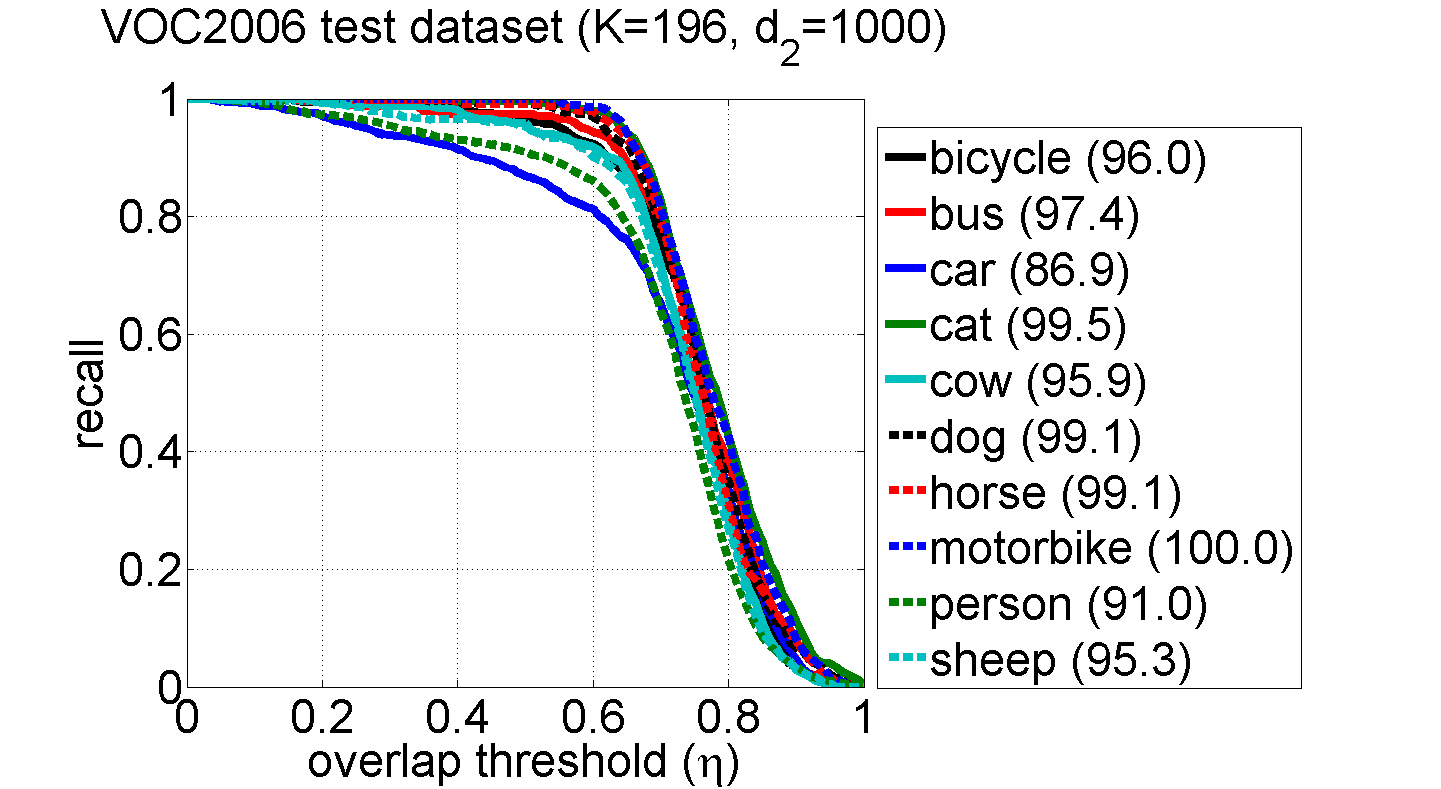}}
 \end{center}
\end{minipage}
\vspace{-5mm}
\caption{\footnotesize{An example of our method \cite{zhang_cvpr11} on demonstrating the localization quality with increase of the number of quantized scales/aspect-ratios, $K$, on the VOC2006 \cite{pascal-voc-2006} dataset using the object recall-overlap evaluation. Recall-overlap curves are plotted for individual classes using $d_2=1000$ final proposals from left to right, and $K\in\{36, 121, 196\}$ from top to bottom. The numbers shown in the legends are the recall percentages when the overlap score threshold for correct localization, $\eta$, is set to 0.5. For more details, please refer to \cite{zhang_cvpr11}.}}\label{fig:cvpr11}
\vspace{0mm}
\end{figure*}

We design our quantization scheme so that in each image any window $t \in \mathcal{T}$ can be represented by at least one window $s \in \mathcal{S}$ in our quantization scheme. 

Fig.~\ref{fig:1} gives an intuitive representation of our scheme. Given the smallest size (width and height) of windows in the scheme $(w_0,h_0)$, we include in our scheme all quantization levels of the form $\mathcal{S}(w_0/\eta^a,h_0/\eta^b)$, where $a \in \{0,1,\cdots, A\}$ and $b \in \{0,1,\cdots, B\}$ are naturally limited by the image size, and $\mathcal{S}(\cdot,\cdot)$ denotes the set of windows with the specific width and height. As a result, the quantization levels can be thought of as forming a tree structure, as illustrated in Fig. \ref{fig:1}.

Next, we will introduce some very important properties of our quantization scheme to explain its essence of reducing the search space for object proposal generation.

\begin{prop}[\textbf{Existence of Quantization Scheme}]\label{prop:ch3-existence}
Given an overlap score threshold $\eta_0$ and a minimum size of objects $(w_0,h_0)$ that can be found in images, any window $s$ with window size $(w_s,h_s)$ can be localized to $\eta_0$-accuracy by at least one window $t$ in our scale/aspect-ratio quantization scheme with parameter $\eta\geq\eta_0$.
\end{prop}
\begin{proof}
According to Fig. \ref{fig:1}, we can construct a subset of windows in our quantized scheme by computing $a\in\left\{\lfloor\log_{\eta}\frac{w_s}{w_0}\rfloor,\cdots,\lceil\log_{\eta}\frac{w_s}{w_0}\rceil\right\}$ and $b\in\left\{\lfloor\log_{\eta}\frac{h_s}{h_0}\rfloor,\cdots,\lceil\log_{\eta}\frac{h_s}{h_0}\rceil\right\}$, where $\lfloor\cdot\rfloor$ and $\lceil\cdot\rceil$ denote the floor and ceiling operations, respectively. Letting $t(w_t,h_t)$ be a window with quantized scale/aspect-ratio $(w_t,h_t)$, the overlap between $s$ and $t$ can be calculated as follows:
\begin{eqnarray}
\lefteqn{\exists a,b,\; o\left(s,t\left(w_0\eta^{a},h_0\eta^{b}\right)\right)}\\
&& =\frac{\min\left\{w_s,w_0\eta^a\right\}\cdot\min\left\{h_s,h_0\eta^b\right\}}{\max\left\{w_s,w_0\eta^a\right\}\cdot\max\left\{h_s,h_0\eta^b\right\}}\nonumber\\
&& =\eta^{\left|a-\log_{\eta}\frac{w_s}{w_0}\right|+\left|b-\log_{\eta}\frac{h_s}{h_0}\right|}\geq\eta^{0.5+0.5}=\eta\geq\eta_0.\nonumber
\end{eqnarray}
That is, $s$ can be localized to $\eta_0$-accuracy by $t$.
\end{proof}

\begin{prop}[\textbf{Sufficient Number of Quantized Scales/Aspect-ratios}]\label{prop:ch3-min-num}
Given an overlap score threshold $\eta$, a minimum size $(w_0,h_0)$ and a maximum size $(w,h)$ of objects that can be found in images, the number of quantized scales/aspect-ratios that is sufficient to localize any object is bounded by $\left(1+\lceil\log_{\eta}\frac{w_0}{w}\rceil\right)\left(1+\lceil\log_{\eta}\frac{h_0}{h}\rceil\right)$.
\end{prop}
\begin{proof}
Let the smallest quantized scale/aspect-ratio in our scheme is $(w_0,h_0)$. Based on the proof in Proposition \ref{prop:ch3-existence}, we can construct a scale/aspect-ratio quantization scheme which limits $a\in\{0,\cdots,\lceil\log_{\eta}\frac{w_0}{w}\rceil\}$ and $b\in\{0,\cdots,\lceil\log_{\eta}\frac{h_0}{h}\rceil\}$. Therefore, the number of quantized scales/aspect-ratios that is sufficient to localize all possible objects in images is bounded by $\left(1+\lceil\log_{\eta}\frac{w_0}{w}\rceil\right)\left(1+\lceil\log_{\eta}\frac{h_0}{h}\rceil\right)$.
\end{proof}

\begin{prop}[\textbf{Search Space for Object Localization}]\label{prop:3}
Given an overlap score threshold $\eta$, the minimum size of quantized scale/aspect-ratio $(w_0,h_0)$, and the maximum image size $(W,H)$, the search space for localizing an arbitrary object in images using quantized windows is $O\left(W\cdot\lceil\log_{\eta}\frac{w_0}{W}\rceil\cdot H\cdot\lceil\log_{\eta}\frac{h_0}{H}\rceil\right)$.
\end{prop}
\begin{proof}
According to Proposition \ref{prop:ch3-min-num}, the search space for scales/aspect-ratios of objects is reduced to $O(\lceil\log_{\eta}\frac{w_0}{W}\rceil\lceil\log_{\eta}\frac{h_0}{H}\rceil)$ using our quantization scheme, while the search space for positions of objects keeps the same $O(W\cdot H)$ as sliding window methods. Therefore, the search space for object localization using our scheme is $O\left(W\cdot\lceil\log_{\eta}\frac{w_0}{W}\rceil\cdot H\cdot\lceil\log_{\eta}\frac{h_0}{H}\rceil\right)$.
\end{proof}

\begin{figure*}[t]
\centerline{\includegraphics[width=2\columnwidth]{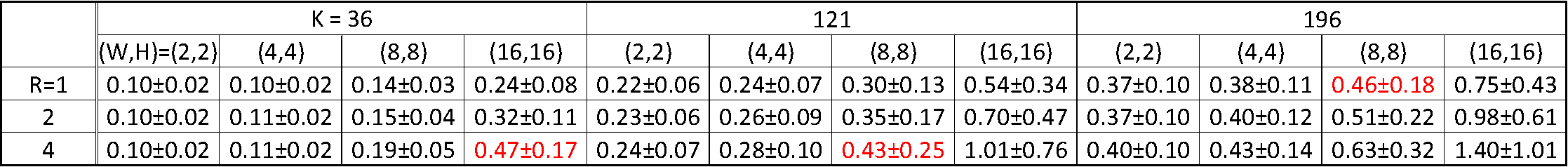}}
\caption{\footnotesize{Comparing the speed of our method \cite{zhang_cvpr11} in seconds at various parameter settings in forms of ``mean$\pm$standard deviation'', where $K$ denotes the number of quantized scales/aspect-ratios, $(W,H)$ denotes the filter size, and $R$ denotes the number of feature channels. The code is written using a mixture of Matlab and C++, and run on a single core with 3.33 GHz. The highlighted (red) numbers are close to the running time in \cite{bb25233}, one of the state-of-the-art cascaded classifiers. For more details, please refer to \cite{zhang_cvpr11}.}}\label{fig:ct}
\vspace{0mm}
\end{figure*}

\subsection{Discussion}
\subsubsection{Object Representation}
Instead of constructing larger and larger quantized scales/aspect-ratios in the quantization scheme, {\em we utilize a same small window size ({\it i.e.} $8\times8$ pixel windows) for all the quantized scales/aspect-ratios by rescaling images accordingly.} In this way, we represent all possible objects in images using a fixed small window size.

The intuitions behind this image rescaling are as follows. The objects of interest in images are usually well-defined with clear boundary ({\it i.e.} high-contrast edges) between them and background. At low resolution, these high-contrast edges preserve the discrimination between objects and background, while the details inside the object regions become blur or even fade away. This allows us to avoid modeling very complex object variations, making every object instance look similar to each other. {\em Our method indeed tries to localize these boundary information using linear filters.} In our recent work \cite{BingObj2014}, this intuition was shared as well.

\subsubsection{Localization Quality}\label{sssec:lq}
From Proposition \ref{prop:ch3-existence}, we can see that {\em the localization quality of a given quantization scheme in our method is dependent on the parameter $\eta$ (NOT $\eta_0$), chosen to construct the quantization scheme.} For instance, in VOC object detection challenges, the overlap score threshold for correct localization is set to 0.5, {\it i.e.} $\eta_0=0.5$. However, to construct our quantization scheme, we can choose an arbitrary value for the parameter $\eta$ as long as $\eta_0\leq\eta<1$, say $\eta=0.6$. Then our method can generate better object proposals than those using $\eta=0.5$, in general. In order to generate proposals with better localization, we have to create more quantized scales/aspect-ratios (based on Proposition \ref{prop:ch3-min-num}), leading to larger search space and higher computational cost accordingly (based on Proposition \ref{prop:3}).

We have verified this situation in \cite{zhang_cvpr11}. Fig. \ref{fig:cvpr11} is cited from \cite{zhang_cvpr11}, where $K\in\{36,121,196\}$ corresponds to $\eta\in\{\frac{1}{2},\frac{2}{3},\frac{3}{4}\}$, respectively, for constructing the quantization schemes. As we see, with increase of $K$, all the curves are pushing towards the top-right corner, in general. This indicates that increasing $K$ does help localize objects better, with observations of larger area-under-the-curve (AUC) scores. Fig. \ref{fig:ct} is also cited from \cite{zhang_cvpr11}, showing that larger $K$ does result in higher computational time under the same parameter setting.

\section{Two-stage Cascade SVMs}
\label{sec:app}
Cascaded classifiers have a decade history in object detection \cite{viola_jones_2001a,Heitz+al:NIPS08a,10.1109/TPAMI.2011.281}, especially the very successful Viola and Jones's method for face detection \cite{viola_jones_2001a}. Cascaded classifiers are good tools for handling extremely imbalanced data, that is, too many negatives and too few positives. Object detection is one of the applications with extremely imbalanced data, where the objects of interest in an image are very few but the non-object are many, considering the huge structural search space of windows. In the cascade, only ``positives'' are passed on as outputs of each stage, which have higher ranks than those ``negatives''.

For ease of explanation of our cascaded approach, we list the main notation used in the following sections in Table \ref{tab:0}.
\begin{table}[t]\centering
\caption{\footnotesize Some notations used in explanation of our cascade SVMs.} \label{tab:0}
\begin{tabular}{|p{1.2cm}|p{6.3cm}|}
  \hline
  Notation & Definition\\
  \hline
  $\mathcal{T}$ & The set of all possible windows in an image.\\  
  $\mathcal{S}$ & The set of all possible windows in our window quantization scheme.\\
  $\mathcal{S}(w,h)$ & The set of all the windows in an image with width $w$ and height $h$.\\
  $o(t,s)$ & The overlap between window $t\in \mathcal{T}$ and window $s\in \mathcal{S}$ (see Def. \ref{def:ch1-overlap}).\\
  $o_t$ & The maximum overlap for window $t\in \mathcal{T}$ in an image (see Def. \ref{def:ch3-maximum-overlap}).\\
  $\eta\in[0,1]$ & Overlap score threshold for proposal generation.\\
  $k$ & A given scale/aspect-ratio combination in our quantization scheme.\\
  $\mathcal{S}_k$ & The set of all the windows which can be represented to $\eta$-accuracy at quantized scale/aspect-ratio $k$.\\
  $\mathbf{w}_k$ , $\mathbf{z}_k$ & Learned linear classifiers at Stage \uppercase\expandafter{\romannumeral1} and \uppercase\expandafter{\romannumeral2}, respectively, for quantized scale/aspect-ratio $k$.\\
  $\mathbf{v}$ & A channel response feature vector used in Stage \uppercase\expandafter{\romannumeral2} for learning $\mathbf{z}$.\\
  \hline
\end{tabular}
\end{table}

In our training data, each image is annotated with the bounding boxes of the objects of interest. Our goal is to give higher ranks to the correct object proposals, given the overlap score threshold parameter $\eta$, than the wrong ones in a very efficient way, such that the windows at the top of the ranking list can be taken as our final object proposals. 

\subsection{Stage \uppercase\expandafter{\romannumeral1}: Scale/Aspect-ratio Specific Ranking}
\label{ssec:pol}
The first stage of our cascade aims to pass on a number of object proposals based on different sliding windows at each of a set of quantized scales and aspect ratios to the next stage. This is done by learning a linear classifier for each quantized scale/aspect-ratio separately. 

\subsubsection{Individual Classifier Learning}
\label{sssec:icl}

Given $\eta$ and a set of quantized scales/aspect-ratios, for each scale $k$\footnote{In the following sections, we refer to scale $k$ as quantized scale/aspect-ratio $k$ for short.} we wish to learn a linear classifier $f_1(\mathbf{x}_{s};\mathbf{w}_k)=\mathbf{w}_k\cdot\mathbf{x}_s$, as suggested in~\cite{park2010multiresolution}, to rank the window $s\in \mathcal{S}_k$, whose feature vector is denoted as $\mathbf{x}_s$, among all the windows in $\mathcal{S}_k$.

Ideally, we expect that within image $I$ the ranking score for any window $s_i\in \mathcal{S}_k\bigcap \mathcal{T}_I$ with $o_{s_i}\geq\eta$ is always higher than that of any window $s_j\in \mathcal{T}_I$ with $o_{s_j}<\eta$. That is, for $\mathbf{w}_k$ we require that within the image $I$ all the corresponding positive training windows $\mathcal{I}^+_k = \{s_i\in \mathcal{S}_k\bigcap \mathcal{T}_I|o_{s_i}\geq\eta\}$ should be ranked above all the training negatives $\mathcal{I}^- = \{s_j\in \mathcal{T}_I|o_{s_j}<\eta\}$. Naturally this leads us to formulate the problem as a ranking SVM as follows:
\begin{eqnarray}
\label{eqn:3}
\lefteqn{\hspace{2mm}\min_{\substack{\mathbf{w}_k,\boldsymbol{\xi}}} \hspace{2mm} \frac{1}{p}\|\mathbf{w}_k\|_p^p+C\sum_{\substack{i,j,n}}\xi_{ij}^n}\\
&{
\begin{array}{ll}
\hspace{-2mm}\mbox{s.t.} & \forall n,i\in \mathcal{I}^+_{kn},j\in \mathcal{I}^-_{n},\, \mathbf{w}_k\cdot(\mathbf{x}_{i}^n-\mathbf{x}_{j}^n)\geq 1-\xi_{ij}^n, \\
\hspace{-2mm}& \xi_{ij}^n\geq 0, \quad p\in{1,2}.
\end{array}}\nonumber
\end{eqnarray}
Here, $\mathbf{x}_{i}^n$ and $\mathbf{x}_{j}^n$ are the feature vectors associated with positive window $i$ and negative window $j$ in training image $I_n$ respectively, $\boldsymbol{\xi}$ are the slack variables, $C\geq0$ is a predefined regularization parameter, and $\|\cdot\|_p$ denotes the $\ell_p$ norm of vectors. 

Recall that the purpose of learning the individual classifier is to build the proposal pool for further usage, so the constraints in Eq.~\ref{eqn:3} are restricted to {\em one} quantized scale in {\em one} image. Therefore, the (local) ranking scores from each classifier are incomparable across scales/aspect-ratios, necessitating the second stage in the cascade.

\textbf{\em Remarks:} In order to make Eq. \ref{eqn:3} more general, we introduce a dummy feature $\mathbf{0}$ and define that its rank is higher than negatives but lower than positives. Then only comparing positive/negative features with the dummy feature turns Eq. \ref{eqn:3} into a standard SVM without ranking constraints. We denote the solution of Eq. \ref{eqn:3} with ranking constraints as ``$\ell_p$-w/r'', and the solution of Eq. \ref{eqn:3} without ranking constraints as ``$\ell_p$-o/r'', respectively.

\subsubsection{Proposal Selection with Non-Max Suppression}
\label{sssec:ps}
To decide which proposals to forward from the first stage to the second of the cascade, we look for the local maxima in the response image of classifier $\mathbf{w}_k$ as illustrated in Fig. \ref{fig:0}(\lowercase\expandafter{\romannumeral1},c), and set a threshold on the maximum number of windows to be passed on. The first stage thus has two controlling parameters. The first, $\gamma \in [0,2]$ specifies the ratio between the size of the neighborhood over which we search for the local maxima, and the reference window size for each classifier. This is the non-max suppression parameter. The second, $d_1 \in \{1,\cdots,1000\}$ specifies the maximum number of windows, which are the top $d_1$ ranked local maxima, as illustrated in Fig. \ref{fig:0}(\lowercase\expandafter{\romannumeral1},d), that can be passed on from any scale. This non-max suppression step is utilized to deal with the difficulty of multiple correct proposals per object.

\subsection{Stage \uppercase\expandafter{\romannumeral2}: Ranking Score Calibration}
\label{ssec:por}
The first stage of the cascade generates a number of proposal windows at each scale $k$ for image $I$. The second stage then re-ranks these windows globally, so that the best proposals across scales are forwarded. To achieve this, we introduce a new feature vector for each window, $\mathbf{v}$, which consists of the channel responses of the classifier at the first stage. For instance, $\mathbf{v}$ could be a 4-dimensional feature vector if feature $\mathbf{x}$ is divided into 4 segments without overlaps, each of which gives a response to the corresponding classifier. The reason for splitting $\mathbf{x}$ into different segments is that we could make full use of information in different segments to improve the calibration performance. 

Based on $\mathbf{v}$, we can re-rank each window $i$ by the decision function $f(\mathbf{v_i})=\mathbf{z}_{k_i}\cdot\mathbf{v}_{i}+e_{k_i}$, where $k_i$ denotes the quantized scale/aspect-ratio associated with window $i$, $\mathbf{z}_{k_i}$ is a set of coefficients for scale $k_i$ that we would like to learn, and $e_{k_i}$ is the corresponding bias term. Similarly, we formulate this learning problem as a multi-class ranking SVM as shown in Eq.~\ref{eqn:4}:
\begin{eqnarray}
\label{eqn:4}
\lefteqn{\hspace{2mm}\min_{\substack{\mathbf{z},\mathbf{e},\boldsymbol{\xi}}} \hspace{2mm} \frac{1}{p}\|\mathbf{z}\|_p^p+C\sum_{\substack{i,j,n}}\xi_{ij}^n}\\
&{
\begin{array}{ll}
\hspace{-2mm}\mbox{s.t.} & \forall n, i\in \hat{\mathcal{I}}^+_n,j\in \hat{\mathcal{I}}^-_n,\\
\hspace{-2mm}& \mathbf{z}_{k_i}\cdot\mathbf{v}_{i}^n-\mathbf{z}_{k_j}\cdot\mathbf{v}_{j}^n+e_{k_i}-e_{k_j}\geq 1-\xi_{ij}^n, \\
\hspace{-2mm}& \xi_{ij}^n\geq 0, \quad p\in\{1,2\}.
\end{array}}\nonumber
\end{eqnarray}
Here, $\hat{\mathcal{I}}^+_n$ and $\hat{\mathcal{I}}^-_n$ denote the positive and negative windows in image $I_n$ forwarded from the first stage of the cascade across different quantized scales/aspect-ratios. Similar to Eq. \ref{eqn:3} with the dummy feature, we continue to use the same notations for the solutions of Eq. \ref{eqn:4}.

In this way, all the windows can be ranked in an image. The top $d_2$ windows are then considered as the final proposals generated at the second stage of our cascade.

\subsection{Computational Complexity}
\label{ssec:cc}
Our method involves the application of simple linear classifiers to the images, and as such is dominated by the complexity of 2D convolution which must be applied to each image. The complexity can thus be approximated as $O(K\times R\times (W\times H)\times (W_I\times H_I))$, where $K$ denotes the number of individual classifiers learned in Stage \uppercase\expandafter{\romannumeral1}, $R$ denotes the number of segments used in Stage \uppercase\expandafter{\romannumeral2}, $(W,H)$ denotes the filter size, and $(W_I,H_I)$ denotes the resized image size. We note that our complexity is therefore (largely) independent of the number of potential proposals let through at each stage ($d_1, d_2$), unlike methods which include non-linear classifiers~\cite{bb25233,vedaldi09multiple}. Also, our algorithm is quite suitable for parallel computing, which will reduce the running time dramatically.

\section{Implementation}
\label{sec:imp}
We list some details of our implementation\footnote{The code is available at \url{https://sites.google.com/a/brookes.ac.uk/zimingzhang/code}.} of the cascade SVMs as follows.

\textbf{(1) Scale/aspect-ratio quantization scheme:} In our experiments, we test $\eta\in\{0.5, 0.67, 0.75\}$, which lead respectively to the maximum numbers of classifiers learned at the first stage $K\in\{36, 121, 196\}$ by limiting the sizes of windows from 10 to 500 pixels. This enables us to approximate the sizes of the smallest object and the whole image within the hierarchy.

\textbf{(2) Features and data used in Stage \uppercase\expandafter{\romannumeral1}:} We use simple gradient features to learn each classifier $\mathbf{w}_k$ at the first stage. In detail, we first convert all the images into gray scale, and represent all the object ground-truth bounding boxes to $\eta$-accuracy using our scale/aspect-ratio quantization scheme to provide positive windows. After randomly selecting negatives across scales, all windows are resized to a fixed feature window size $(W,H)$, and then for each pixel, the magnitude of its gradient is calculated. At test time, to generate features $\mathbf{x}$, we simply resize the image for each scale $k$ by the ratio of its reference window to $(W,H)$, and then apply the learned classifier $\mathbf{w}_k$ by 2D convolution.

\textbf{(3) Features used in Stage \uppercase\expandafter{\romannumeral2}:} We use the 1D ({\it i.e.} $R=1$) classifier responses ({\it i.e.} margins) from Stage \uppercase\expandafter{\romannumeral1} as features to train the ranking SVM, because from \cite{zhang_cvpr11}, we can see that the performance gained by increasing the dimension of features in Stage \uppercase\expandafter{\romannumeral2} is marginal, but computational time is boosted significantly, especially for large window size $(W,H)$.

\textbf{(4) Parameters $\gamma$, $d_1$, $W$ and $H$:} Here we follow our work in \cite{zhang_cvpr11} and keep using the same parameters as before. Precisely, $\gamma=0.6$, $d_1=50$ \footnote{When $K=36$, we set $d_1=150$ so that our method can select more than $10^3$ proposals from Stage \uppercase\expandafter{\romannumeral1}. For other $K$, we still use $d_1=50$.}, $W=H=16$ pixels. Please refer to \cite{zhang_cvpr11} for the parameter selection details.

\textbf{(5) SVM solver:} We employ LIBLINEAR \cite{REF08a} as our solver. To train ranking SVMs, we take $10^5$ samples randomly as the training set, each of which is created by a positive minus a negative. Without tuning, in all the cases, we set the regularization parameter $C=10$.

\section{Experiments}
\label{sec:exp}

In \cite{zhang_cvpr11} we have demonstrated the capability of our method for class-specific object proposal generation, and partial experimental results are shown in Fig. \ref{fig:cvpr11} and Fig. \ref{fig:ct}. For more details, please refer to \cite{zhang_cvpr11}.

In this paper, our method is extended for {\em generic object proposal generation}, and outputs bounding boxes as object proposals. Therefore, it is fair to compare our method with the two very closely related work \cite{Rahtu_iccv11}\footnote{We downloaded their public code and precomputed windows for VOC2007 from \url{http://www.cse.oulu.fi/CMV/Downloads/ObjectDetection}.} and \cite{Alexe2012pami}\footnote{We downloaded their public code and precomputed windows for VOC2007 from \url{http://groups.inf.ed.ac.uk/calvin/objectness/}.}. We did not compare ours with \cite{Endres:2010:CIO:1888150.1888195} because their outputs are pixels in the proposals rather than bounding boxes, which makes their method better for segmentation measure.


We test our method on PASCAL VOC2007 \cite{pascal-voc-2007}. VOC2007 contains 20 object categories, and consists of 9963 natural images with object labels and their corresponding ground-truth bounding boxes released for training, validation and test sets.

We learn only one object model per quantized scale/aspect-ratio by using all the object instances in the training data as positives to train a single binary object/non-object filter and output object proposals per image during testing, no matter what classes the object instances belong to. This is our default learn and testing procedure without specific mention.

We measure our performance in terms of {\em object recall vs. overlap score threshold} (recall-overlap for short) curves \cite{zhang_cvpr11,Rahtu_iccv11,vedaldi09multiple,bb25233}, {\em object recall vs. number of proposals} (recall-proposal for short) curves at $\eta=0.5$, and running speed. We follow the PASCAL VOC challenge and use $\eta=0.5$ for correct detection. 

\subsection{VOC2007}
We first test different cascade settings on this dataset using different $K$'s and $\ell$'s, and then compare our method with \cite{Rahtu_iccv11,Alexe2012pami}. We use the training/validation dataset, consisting of 5011 images, to train our model, and test it on the test dataset, comprising 4952 images. 

\subsubsection{Cascade Setting Comparison}
\label{sssec:roe}

Fig. \ref{fig:cascade-comparison} summarizes the comparison results, where our program runs for three times and we report the mean and standard deviation of our results. From the top 3 settings in each sub-figure, we can see that (1) In general, the performances using different settings are close to each other; (2) In Stage \uppercase\expandafter{\romannumeral1}, the method $\ell_1-o/r$ seems to work best, which trains $\ell_1$-norm SVMs, rather than ranking SVMs; (3) In Stage \uppercase\expandafter{\romannumeral2}, both methods $\ell_1-o/r$ and $\ell_2-w/r$ seem to work better than others, the first training $\ell_1$-norm SVMs and the second training $\ell_2$-norm ranking SVMs; (4) The method proposed in \cite{zhang_cvpr11} is slightly worse than the best setting; (5) With a larger $K$, the AUC score under 1000 proposals becomes larger, while differences of the AUC score under a fewer proposals ({\it i.e.} 1, 10, 100) are marginal. This also verifies that with a larger $K$, the localization quality of our object proposals will become better, in general, as stated in Section \ref{sssec:lq}.

It surprises us that the $\ell_1-o/r$ SVMs work so well in our cascade, because usually $\ell_2$-norm SVMs work better than $\ell_1$-norm SVMs \cite{REF08a}. We believe that $\ell_1$-norm SVMs actually select the discriminant features and suppress non-discriminant ones between objects and non-objects.

\subsubsection{Recall-Overlap Evaluation}
The object recall {\it vs.} overlap score threshold (recall-overlap for short) curves measure the quality of proposals within a fixed number of proposals by varying the overlap score threshold. 

Fig. \ref{fig:recall-overlap} shows our comparison results on VOC2007. We can see here the movement of the curves towards the top-right both as we allow more output proposals ($d_2\in\{1, 10, 100, 1000\}$) and as we increase $K=\{36, 121, 196\}$ in our quantization scheme. Recall that our quantized scales/aspect-ratios are designed to cover bounding boxes to a particular overlap score threshold of $\eta$, so $K\in\{36, 121, 196\}$ corresponds to $\eta\in\{0.5, 0.67, 0.75\}$ respectively. This affects the performance observed, and on the $K=36$ graph for instance, we see that the curves are high for $\eta \leq 0.5$, but then drop quickly. However, the curves for the $K=121$ and $K=196$ drop at the corresponding later points, around $\eta=0.6$, implying our quantization is capturing the desired information.

From the curves, we also can see that our method has a similar behavior to \cite{Alexe2012pami}, and their AUC values are close to each other, and at $\eta=0.5$, in most cases our method and \cite{Alexe2012pami} achieve higher object recall than \cite{Rahtu_iccv11}. However, in terms of proposal localization quality, \cite{Rahtu_iccv11} is the best among these methods, because its curves drop quickly when $\eta$ is larger than around 0.75, while the curves of ours and \cite{Rahtu_iccv11} drop when $\eta$ is larger than around 0.55. This observation indicates that compared to our method and \cite{Alexe2012pami}, the correct detection proposals outputted by \cite{Rahtu_iccv11} are closer to the ground-truth bounding boxes of objects, which may be caused by the structured learning used in \cite{Rahtu_iccv11}. 

Fig.~\ref{fig:recall-overlap-class-2007} breaks down the VOC2007 results in Fig. \ref{fig:recall-overlap} by classes using 1000 proposals, and displays the recall-overlap curves. Similar observations to Fig. \ref{fig:recall-overlap} can be made. Table \ref{tab:1} summarizes the AUC score comparison on VOC2007.
\begin{table}[t]\centering
\caption{{{{\footnotesize AUC score comparison on VOC2007 using 1000 proposals in Fig. \ref{fig:recall-overlap} and Fig. \ref{fig:recall-overlap-class-2007}.}}}} \label{tab:1}
\begin{tabular}{c|cc}
  Methods & AUC (objects) & AUC (classes)\\
  \hline\hline
  Ours ($\ell_1-o/r+\ell_1-o/r$) & 64.5\% & (65.1$\pm$2.2)\% \\
  \hline\hline
  \cite{Alexe2012pami} & 64.9\% & (66.8$\pm$4.2)\% \\
  \cite{Rahtu_iccv11} & \textbf{\textit{67.4\%}} & \textbf{\textit{(70.8$\pm$8.3)\%}}
\end{tabular}
\end{table}

\subsubsection{Recall-Proposal Evaluation}
As a pre-process step in a system, the object recall with a certain $\eta$ using a fixed number of proposals is more important, because this recall determines the best performance that objects can be detected. Therefore, we propose another measure using the object recall vs. number of proposals (recall-proposal for short) curves at $\eta=0.5$.

In Fig.~\ref{fig:recall-proposal} we show how the recalls of different methods are effected as we increase the number of output proposals $d_2$ from 1 to 1000 on VOC2007. We can see that when $d_2$ is beyond 200, the curves become flatter and flatter. We believe that this property of our approach is useful for detection tasks, because it narrows down significantly the total number of windows that classifiers need to check while losing few correct detections. From the comparison of the 4 cascade settings, $\ell_1-o/r+\ell_1-o/r$ performs best. Therefore, {\em in the following experiments we use the setting ``$\ell_1-o/r+\ell_1-o/r$'' as our default cascade setting}. Compared with \cite{Alexe2012pami,Rahtu_iccv11}, our method has a similar behavior to \cite{Alexe2012pami}, and both are better than \cite{Rahtu_iccv11} significantly.

Similarly, Fig.~\ref{fig:recall-proposal-class-2007} breaks down the VOC2007 results in Fig. \ref{fig:recall-proposal} by classes and displays the recall-overlap curves. As we see, some categories need far fewer proposals to achieve good performance. For instance, for the dog category, 100 output proposals saturate performance. Table \ref{tab:2} lists the object recall comparison on VOC2007.
\begin{table}[t]\centering
\caption{{{{\footnotesize Object recall comparison on VOC2007 using 1000 proposals as shown in Fig. \ref{fig:recall-proposal} and Fig. \ref{fig:recall-proposal-class-2007}.}}}} \label{tab:2}
\begin{tabular}{c|cc}
  Methods & Recall (objects) & Recall (classes)\\
  \hline\hline
  Ours ($\ell_1-o/r+\ell_1-o/r$) & \textbf{\textit{93.8}}\% & \textbf{\textit{(95.1$\pm$3.5)\%}} \\
  \hline\hline
  \cite{Alexe2012pami} & 88.6\% & (92.0$\pm$6.7)\% \\
  \cite{Rahtu_iccv11} & 77.7\% & (82.8$\pm$12.8)\%
\end{tabular}
\end{table}

Particularly, here we also perform a same experiment used in objectness \cite{Alexe2012pami}. We divide the 20 object categories into two sets. Same as \cite{Alexe2012pami}, we use the 14 categories ({\it i.e.} aeroplane, bicycle, boat, bottle, bus, chair, diningtable, horse, motorbike, person, pottedplant, sofa, train, tvmonitor) as testing categories, and the rest as training categories, which means that these 14 categories are unseen during training. The images containing objects within the training categories in the training/validation dataset are utilized as the training data for learning our models, and the images containing objects within the testing categories in the test dataset are utilized as the test data for evaluating our method. This experiment is designed for exploring the generality of the proposal methods. Table \ref{tab:ch3-generic} lists our comparison results between ours and objectness \cite{Alexe2012pami}. Still our method outperforms \cite{Alexe2012pami} in terms of object recall given the number of proposals.
\begin{table}[t]\centering
\caption{\footnotesize Object recall comparison on VOC2007 using different numbers of proposals and the same experimental setting in \cite{Alexe2012pami}.} \label{tab:ch3-generic}
\begin{tabular}{c|ccc}
  Method & 10 Prop. & 100 Prop. & 1000 Prop.\\
  \hline\hline
  Ours ($\ell_1-o/r+\ell_1-o/r$) & \textbf{46.4}\% & \textbf{78.7}\% & \textbf{93.1}\%  \\
  \cite{Alexe2012pami} & 41.0\% & 71.0\% & 91.0\%
\end{tabular}
\end{table}

\subsubsection{Computational Time}
\begin{table}[t]\centering
\caption{{\footnotesize Computational time comparison on VOC2007 in second per image with 1000 proposals.}} \label{tab:3}
\begin{tabular}{c|c}
  Methods & Computational time\\
  \hline\hline
  Ours ($\ell_1-o/r+\ell_1-o/r$) & \textbf{\textit{0.20$\pm$0.02}} \\
  \hline\hline
  \cite{Alexe2012pami} & 3.58$\pm$0.25 \\
  \cite{Rahtu_iccv11} & 2.22$\pm$0.42
\end{tabular}
\end{table}
The computational time comparison of the three methods is listed in Table~\ref{tab:3}. Our implementation is a mixture of Matlab and C++, just like \cite{Alexe2012pami,Rahtu_iccv11}, and all the programs are run on a single core of Intel Xeon W3680 CPU with 3.33GHz. The computational time shown here includes all the steps at the test stage starting from loading images. As we see, our method is more than 10 times faster than \cite{Alexe2012pami,Rahtu_iccv11}, because our method only utilizes the simple gradients in gray images as features, and 2D convolution for classification, which are very efficient.

\section{Conclusion and Discussion}
\label{sec:con}
We propose a very efficient two-stage cascade SVM method for both class-specific and generic object proposal generation. To achieve better computational efficiency, we propose a scale/aspect-ratio quantization scheme to reduce the bounding box search space into log-space. To represent each object instance, we utilize the simple gradients within small fixed-size windows ({\it i.e.} $8\times 8$ pixels). We learn linear filters in each stage based on SVM formulations, resulting in applying fast 2D convolution to localizing object proposals during testing. Non-max suppression is used to select proper proposals in the first stage.

We envisage that the cascaded model can be used as the initial stage in complex systems. Our framework naturally incorporates scale and aspect ratio information about objects, which are treated separately in the first stage of the cascade, and we emphasize the flexibility of the framework, where different types of features could easily be incorporated at this stage. Our method is both fast and efficient, and we have shown a substantial improvement in speed and recall over two recent related work \cite{Alexe2012pami,Rahtu_iccv11}. Besides object detection, we believe that our work will contribute to many other research areas, such like segmentation \cite{cpmc_pami12} and stereo matching \cite{Bleyer:2012:ESO:2403138.2403174}.

Our recent proposal generation method in \cite{BingObj2014} achieves the fastest running time among all popular object proposal generation methods \cite{hosang2014good,zitnickedge}, and the most repeatable under different imaging conditions ({\it e.g.} illumination, rotation, scaling, blurring, {\it etc.}). However, the main issue of our method seems that the localization quality of our proposals are worse quantitatively compared to other methods. This is mainly because of our scale/aspect-ratio quantization scheme. Unfortunately, as we stated above, for our method better localization quality can be achieved at the cost of higher computational cost. Thus, how to reduce such computational burden as well as improving the localization quality will be our future work.

\begin{figure*}[h]
\begin{minipage}[b]{0.33\linewidth}
 \begin{center}
 \centerline{\includegraphics[width=1.1\columnwidth]{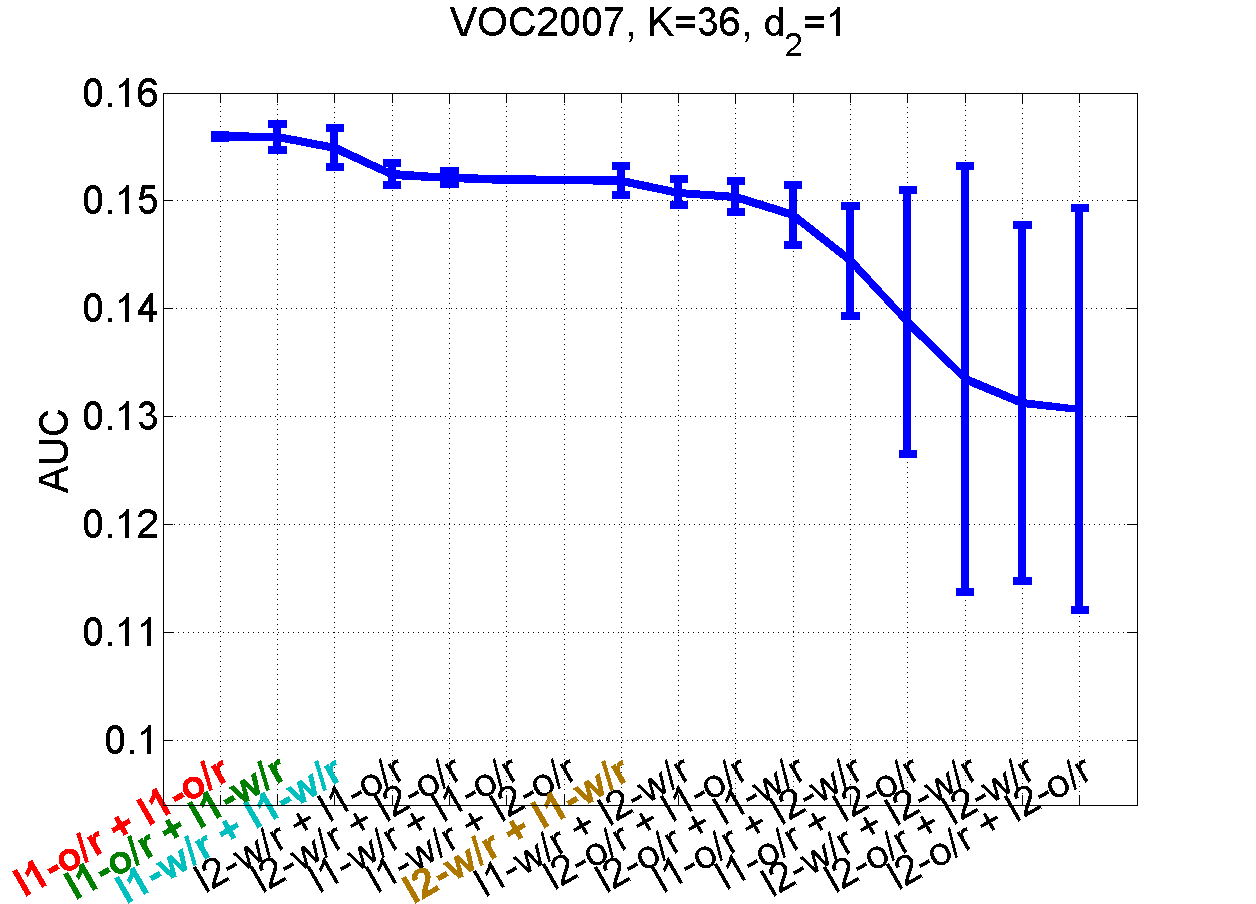}}
 \end{center}
\end{minipage}
\begin{minipage}[b]{0.33\linewidth}
 \begin{center}
 \centerline{\includegraphics[width=1.1\columnwidth]{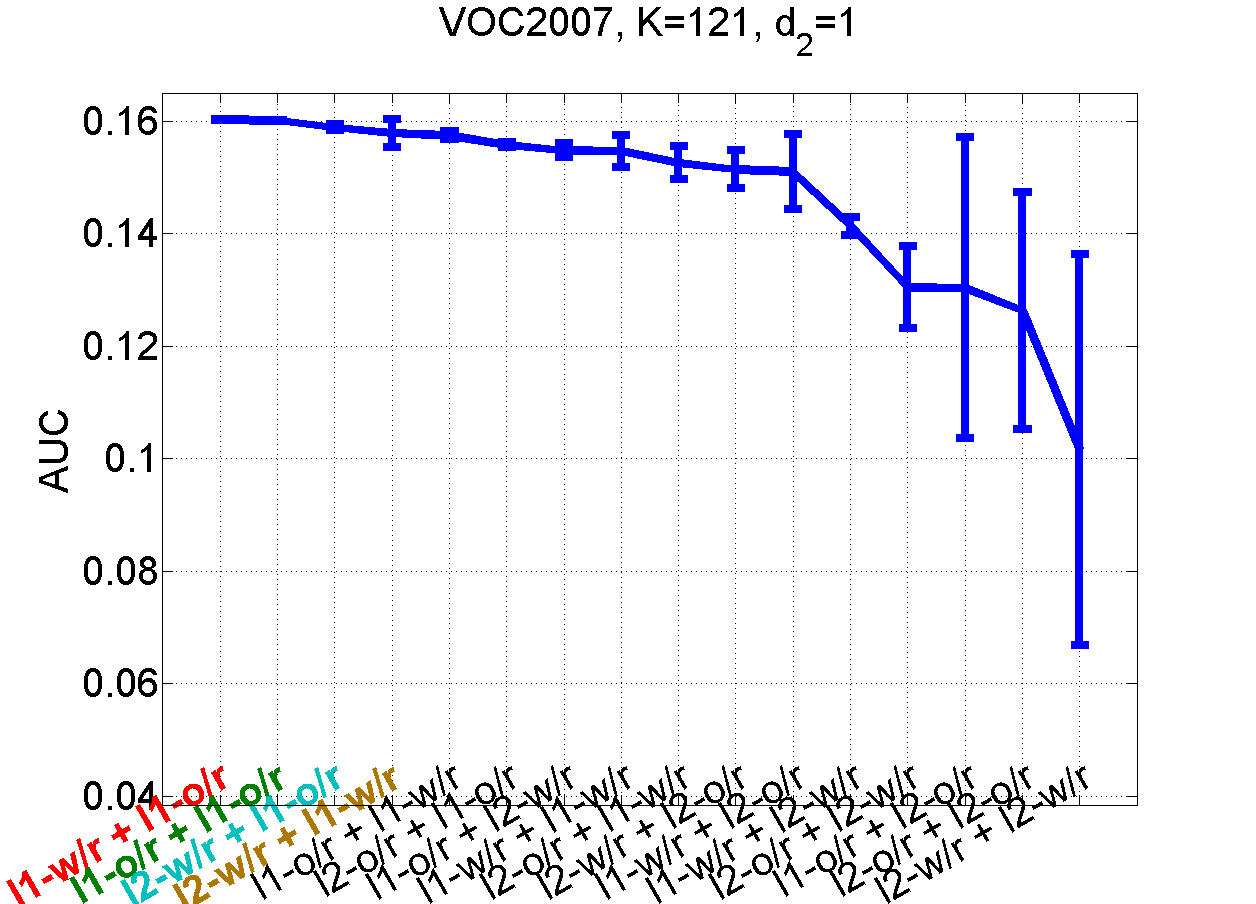}}
 \end{center}
\end{minipage}
\begin{minipage}[b]{0.33\linewidth}
 \begin{center}
 \centerline{\includegraphics[width=1.1\columnwidth]{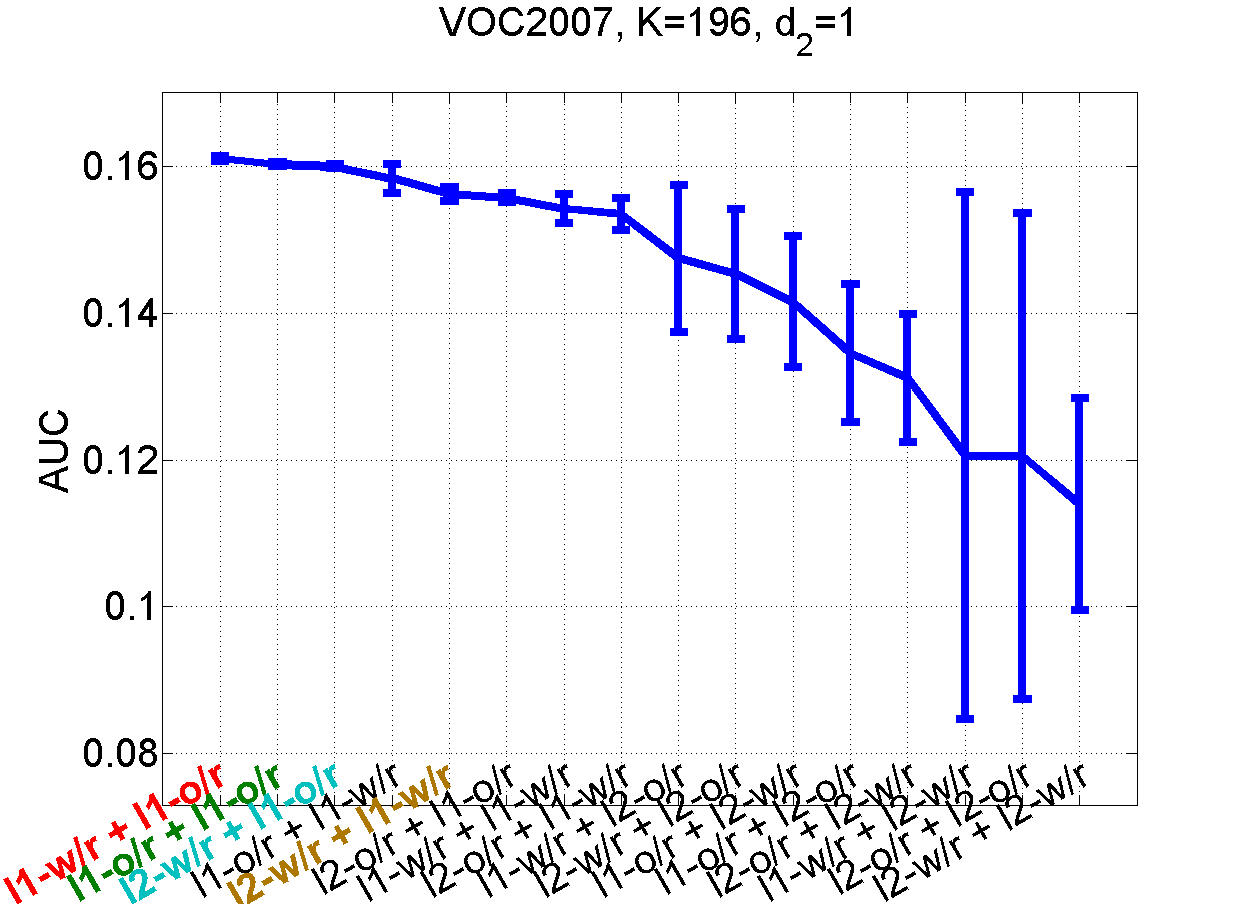}}
 \end{center}
\end{minipage}
\begin{minipage}[b]{0.33\linewidth}
 \begin{center}
 \centerline{\includegraphics[width=1.1\columnwidth]{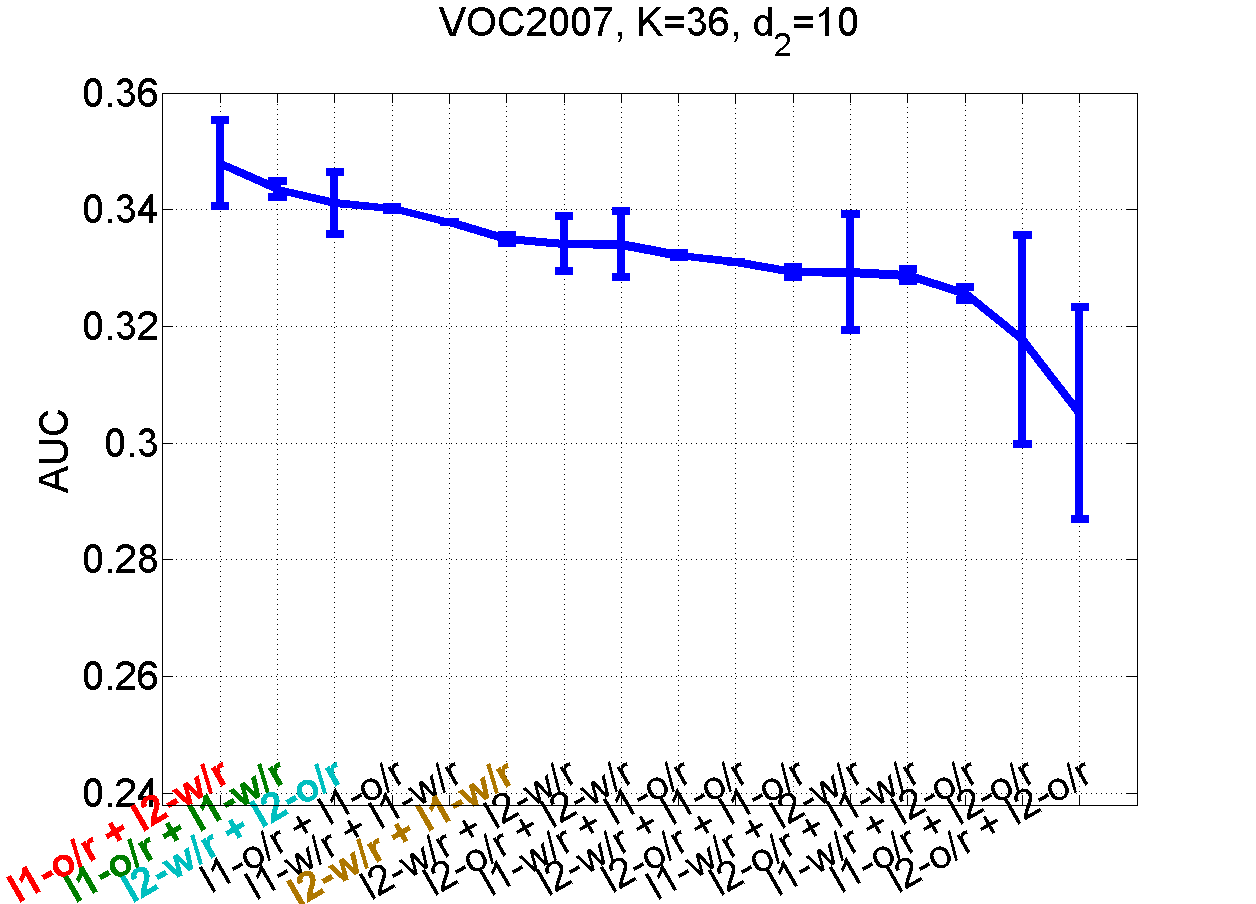}}
 \end{center}
\end{minipage}
\begin{minipage}[b]{0.33\linewidth}
 \begin{center}
 \centerline{\includegraphics[width=1.1\columnwidth]{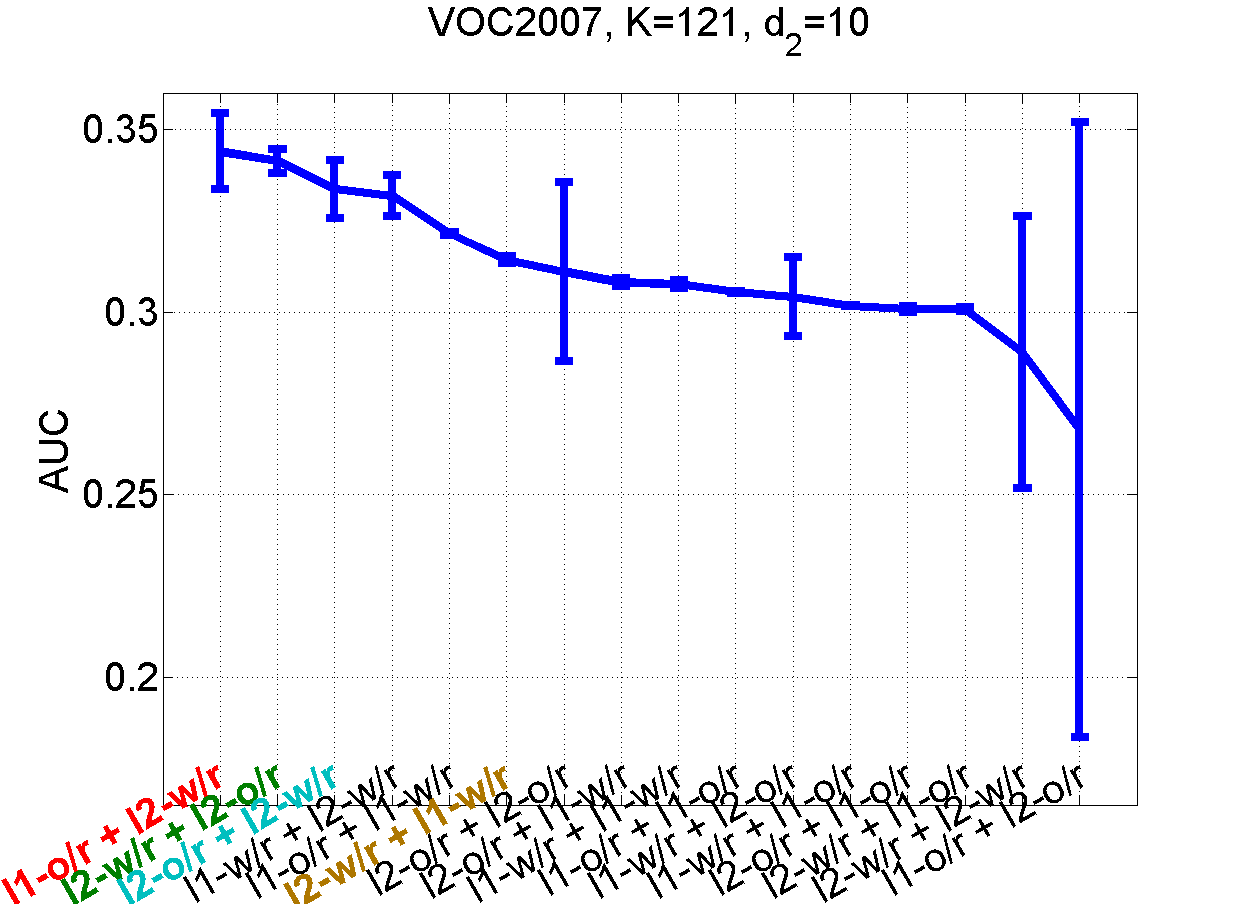}}
 \end{center}
\end{minipage}
\begin{minipage}[b]{0.33\linewidth}
 \begin{center}
 \centerline{\includegraphics[width=1.1\columnwidth]{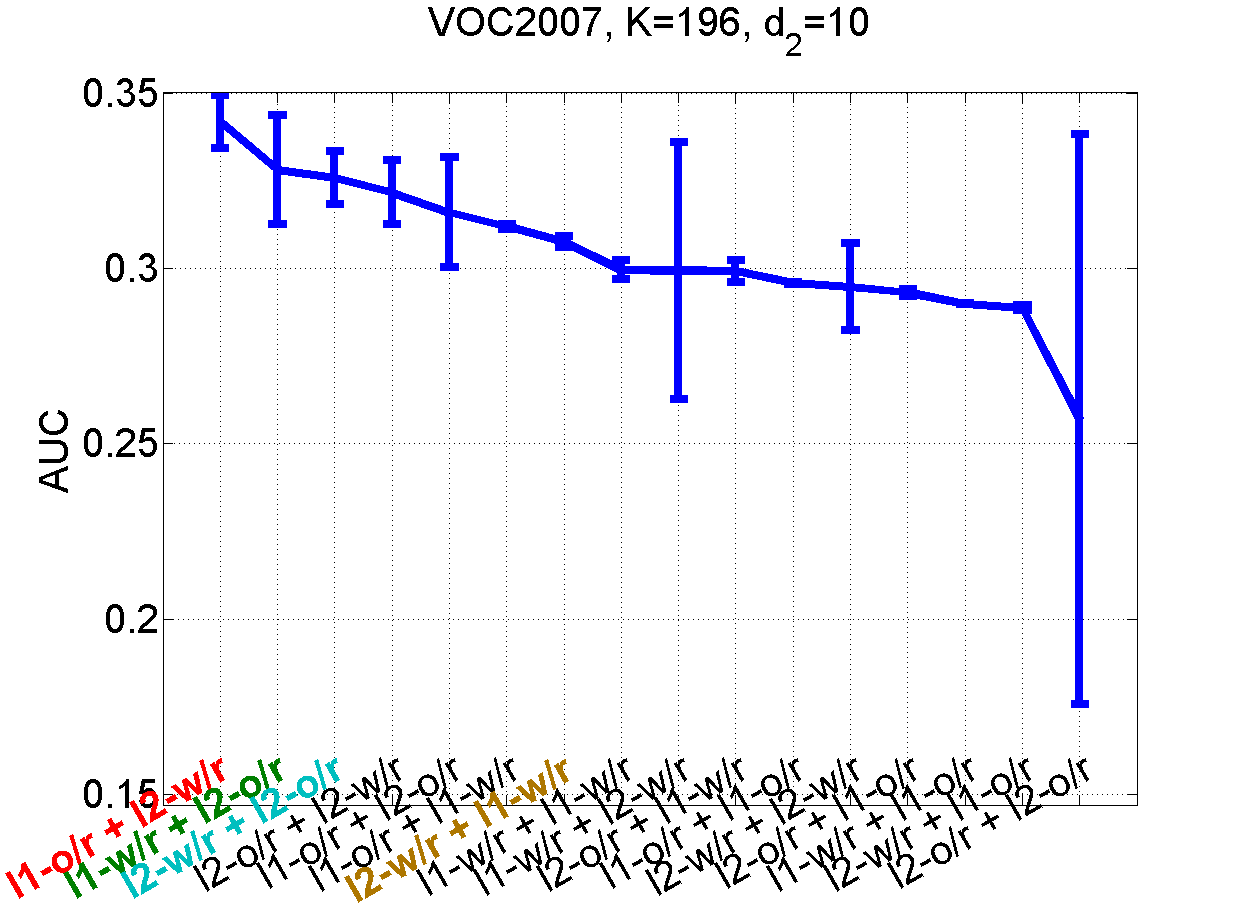}}
 \end{center}
\end{minipage}
\begin{minipage}[b]{0.33\linewidth}
 \begin{center}
 \centerline{\includegraphics[width=1.1\columnwidth]{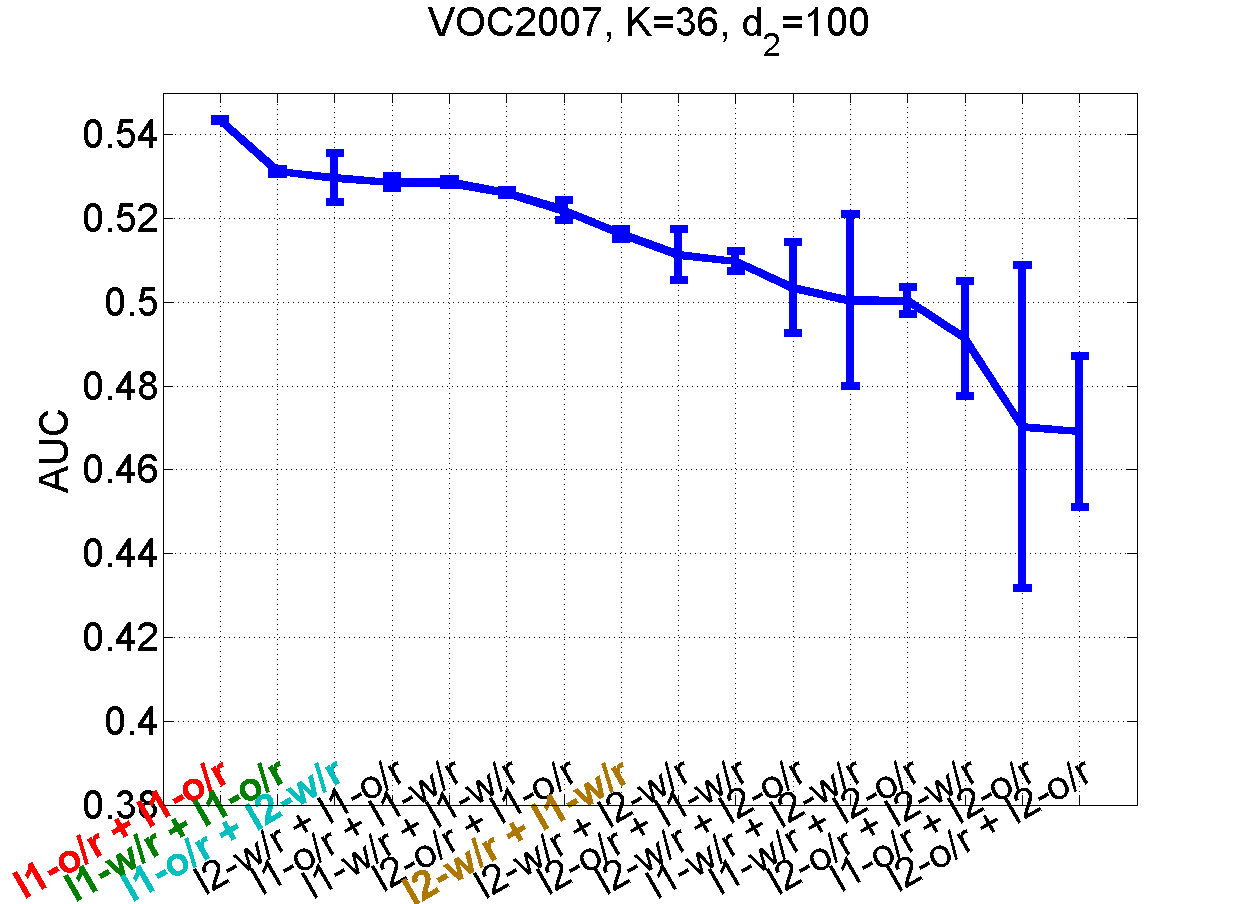}}
 \end{center}
\end{minipage}
\begin{minipage}[b]{0.33\linewidth}
 \begin{center}
 \centerline{\includegraphics[width=1.1\columnwidth]{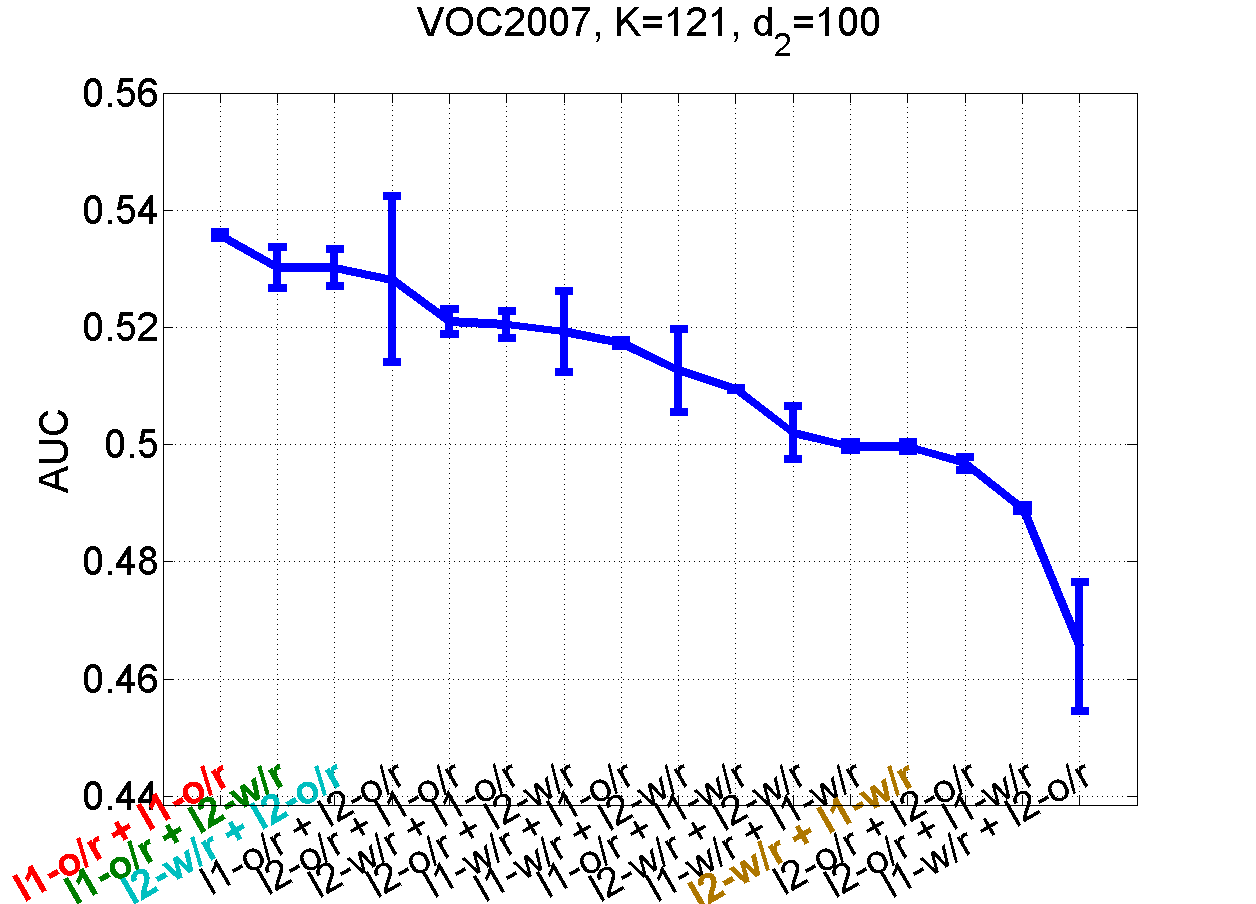}}
 \end{center}
\end{minipage}
\begin{minipage}[b]{0.33\linewidth}
 \begin{center}
 \centerline{\includegraphics[width=1.1\columnwidth]{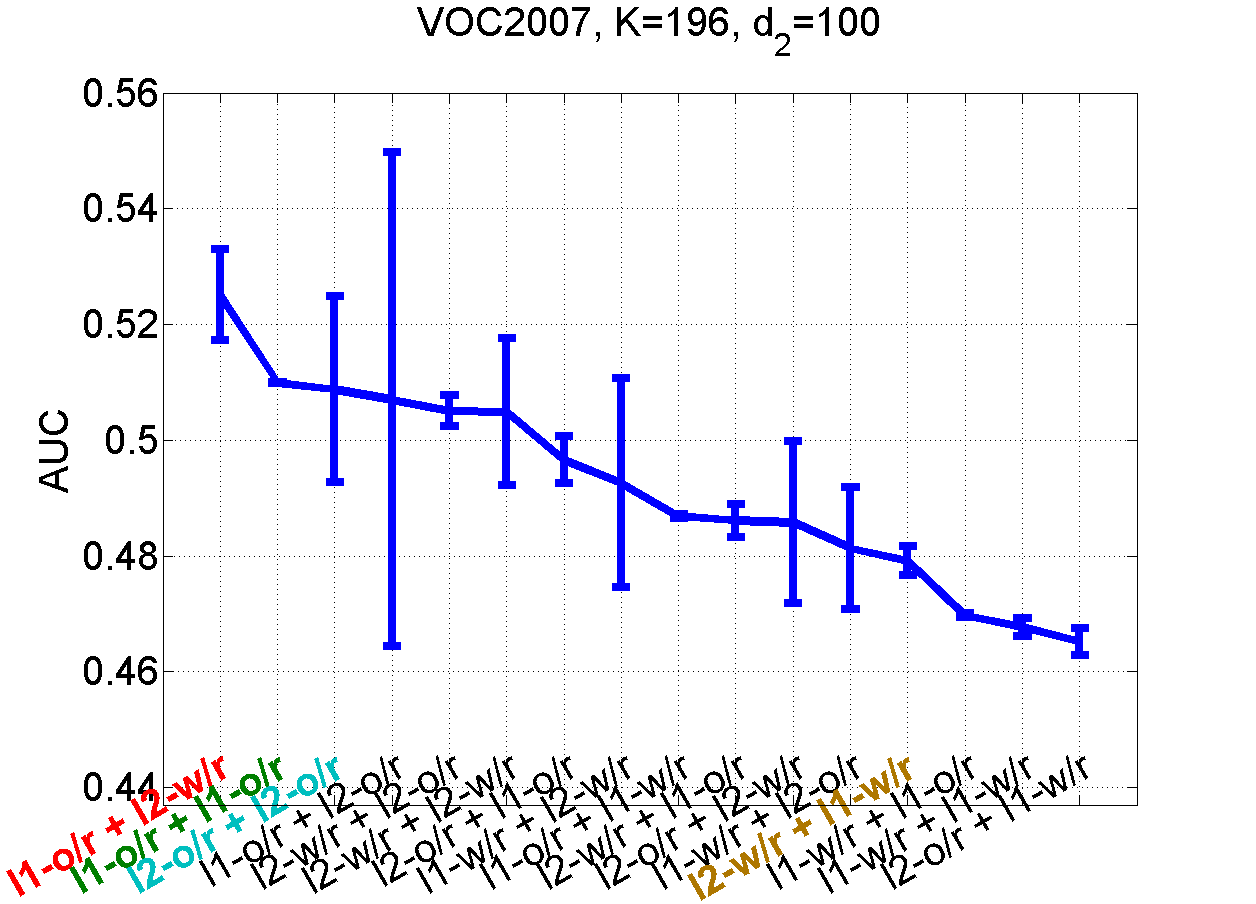}}
 \end{center}
\end{minipage}
\begin{minipage}[b]{0.33\linewidth}
 \begin{center}
 \centerline{\includegraphics[width=1.1\columnwidth]{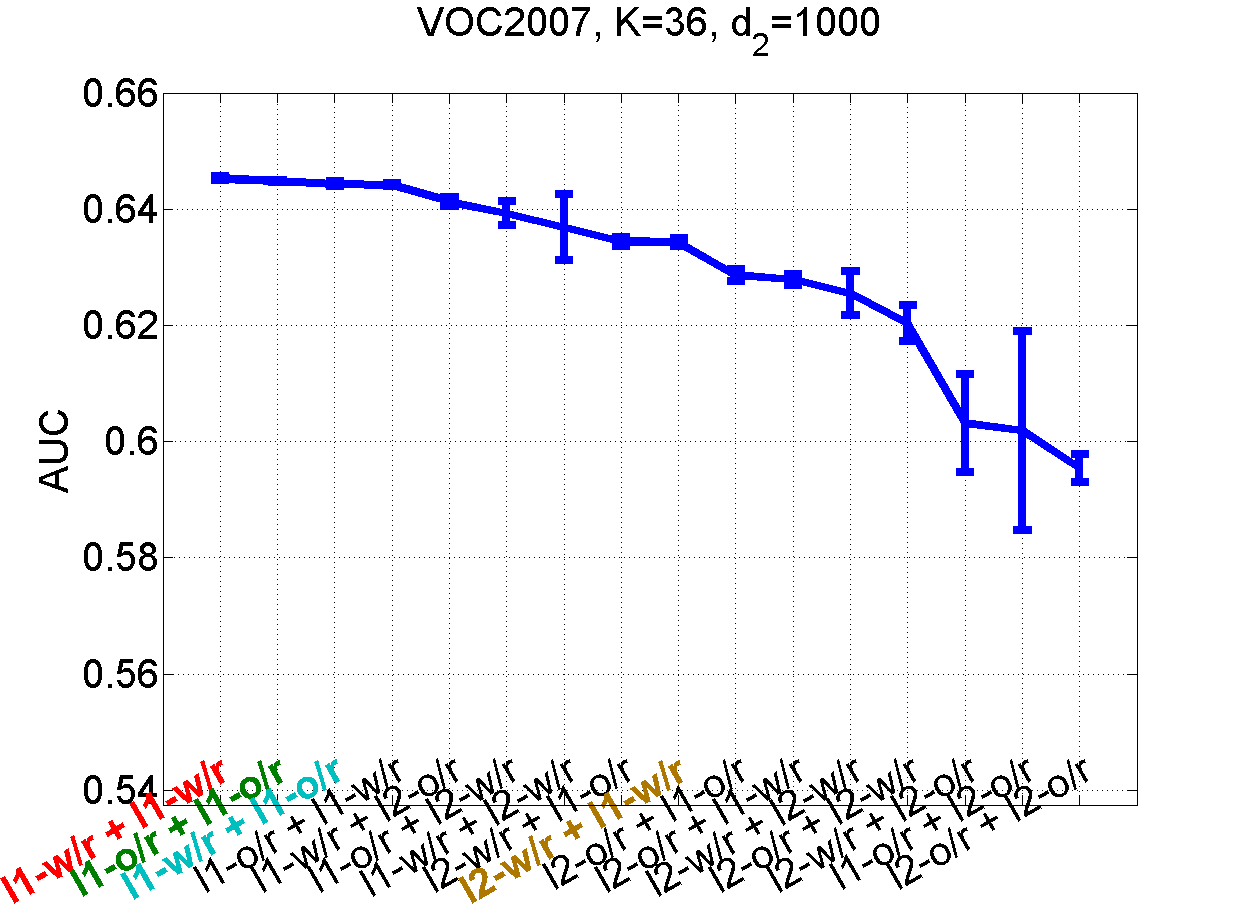}}
 \centerline{\footnotesize{(a) $K=36$}}
 \end{center}
\end{minipage}
\begin{minipage}[b]{0.33\linewidth}
 \begin{center}
 \centerline{\includegraphics[width=1.1\columnwidth]{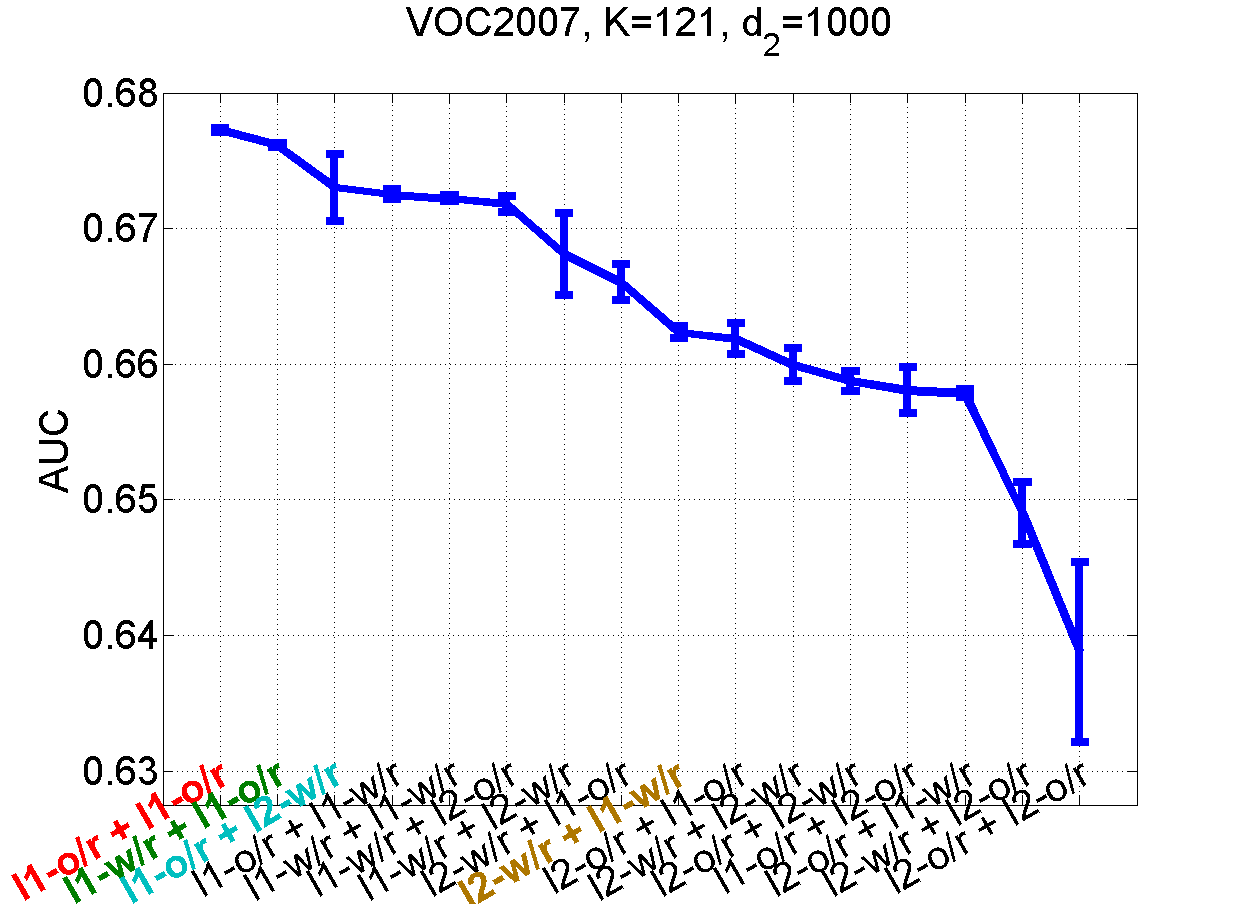}}
 \centerline{\footnotesize{(b) $K=121$}}
 \end{center}
\end{minipage}
\begin{minipage}[b]{0.33\linewidth}
 \begin{center}
 \centerline{\includegraphics[width=1.1\columnwidth]{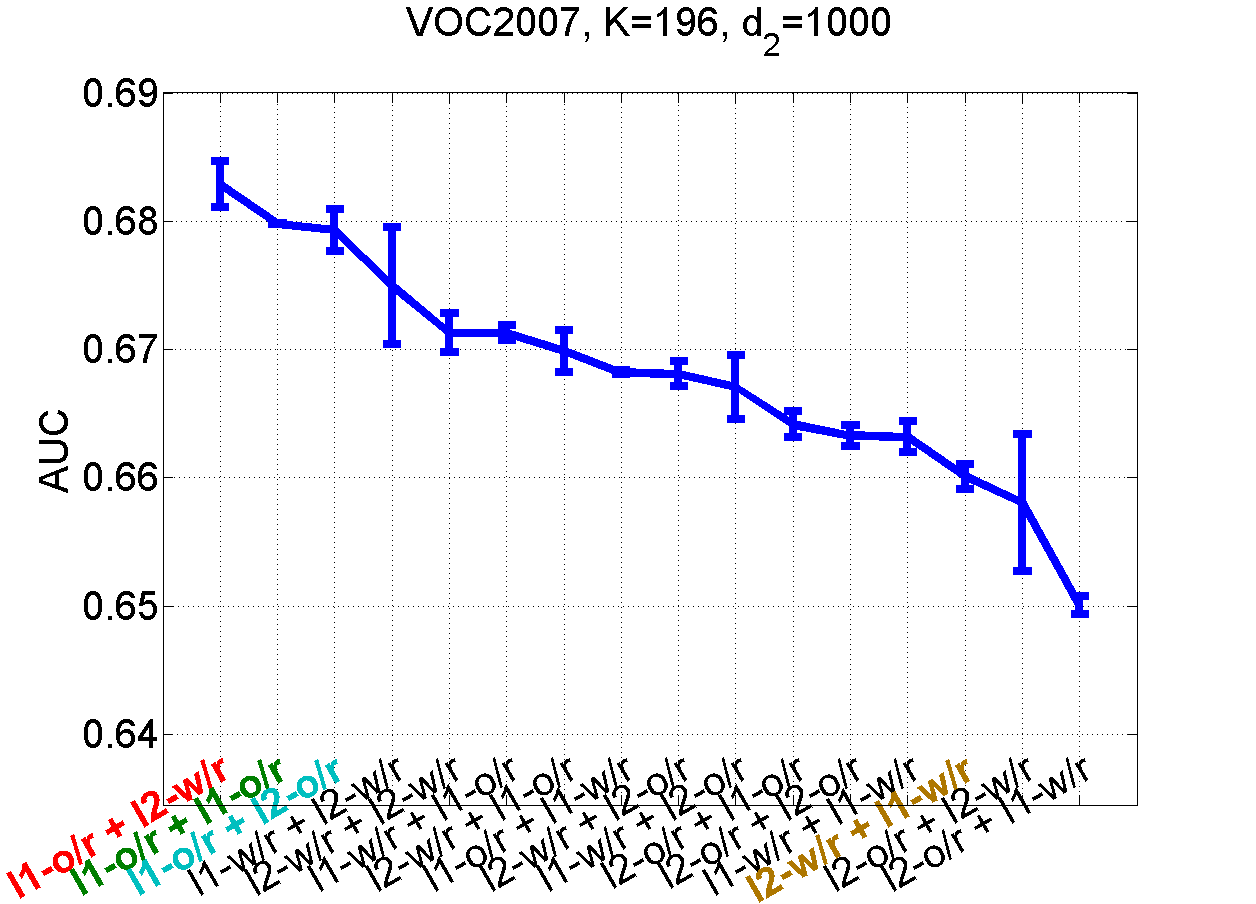}}
 \centerline{\footnotesize{(c) $K=196$}}
 \end{center}
\end{minipage}
\caption{{\footnotesize Comparison of different cascade settings (Stage \uppercase\expandafter{\romannumeral1} + Stage \uppercase\expandafter{\romannumeral2}) on VOC2007 using $K\in\{36,121,196\}$ and $d_2\in\{1,10,100,1000\}$, respectively. In each sub-figure, the cascade settings are sorted in descending order based on the means of different {\em area under object recall-overlap curves} (AUC) scores, the top 3 settings are colored by red, green, and cyan, respectively. Note that the setting ``$\ell_2-w/r + \ell_1-w/r$'' is the method proposed in \cite{zhang_cvpr11}, colored by yellow.}}\label{fig:cascade-comparison}
\vspace{-0mm}
\end{figure*}

\begin{figure*}[h]
\begin{minipage}[b]{0.33\linewidth}
 \begin{center}
 \centerline{\includegraphics[width=1.1\columnwidth]{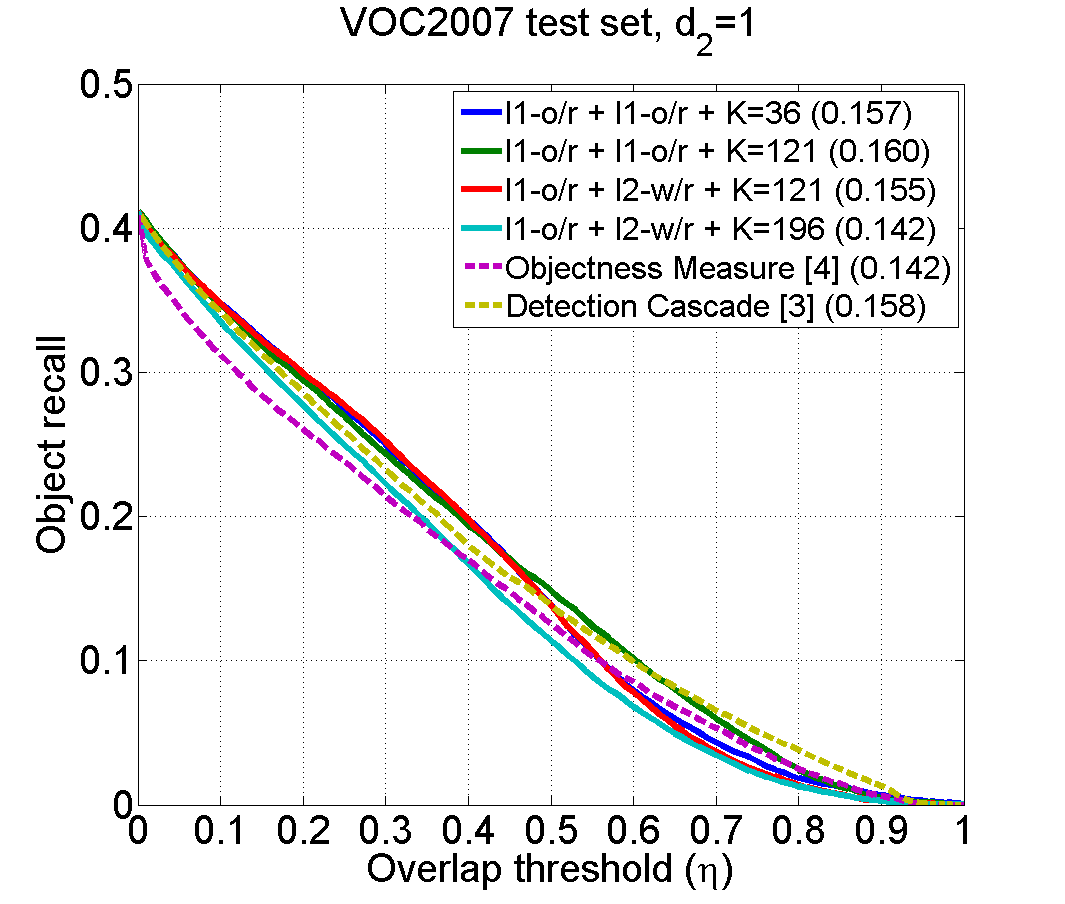}}
 \end{center}
\end{minipage}
\begin{minipage}[b]{0.33\linewidth}
 \begin{center}
 \centerline{\includegraphics[width=1.1\columnwidth]{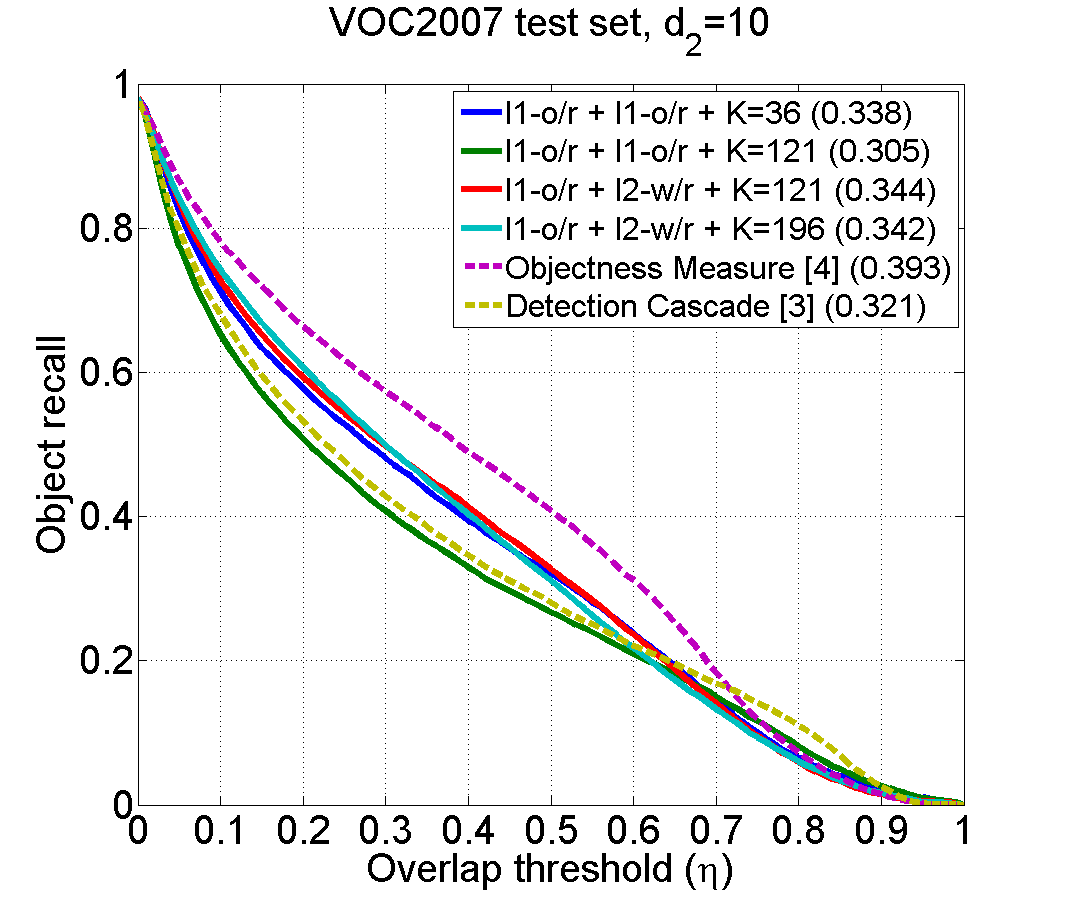}}
 \end{center}
\end{minipage}
\begin{minipage}[b]{0.33\linewidth}
 \begin{center}
 \centerline{\includegraphics[width=1.1\columnwidth]{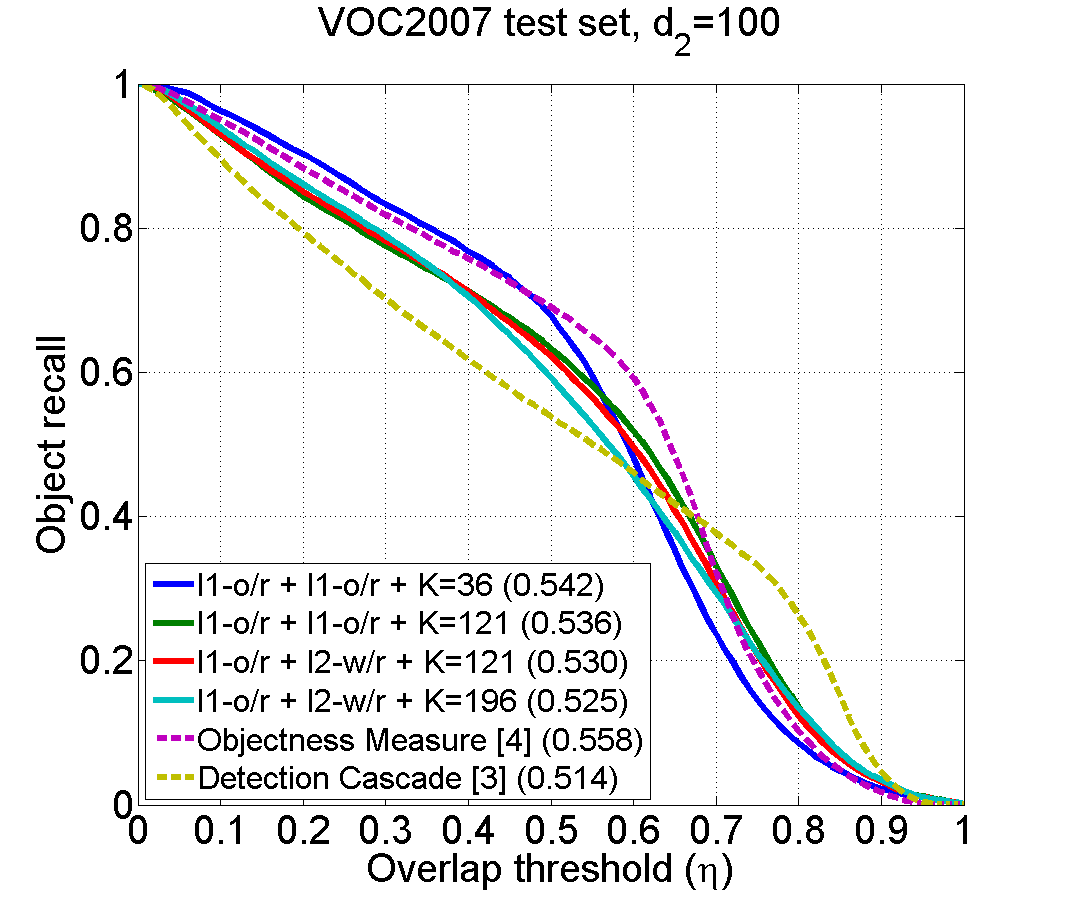}}
 \end{center}
\end{minipage}
\begin{minipage}[b]{0.33\linewidth}
 \begin{center}
 \centerline{\includegraphics[width=1.1\columnwidth]{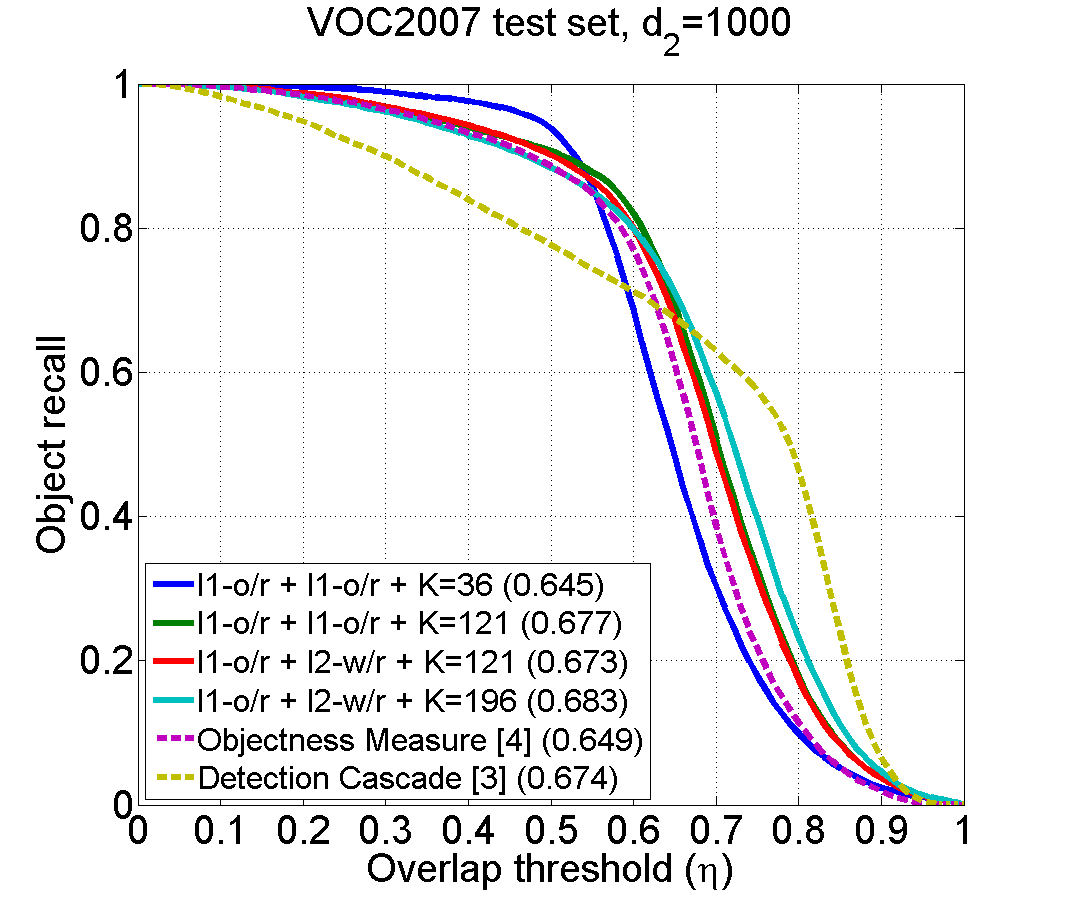}}
 \end{center}
\end{minipage}
\caption{{\footnotesize Comparison of recall-overlap curves using different methods and $d_2$'s on VOC2007. The numbers in the brackets are AUC scores for the methods. Among the 4 cascade settings, our method with larger $K$ achieves better AUC scores using more proposals. Overall, with a same number of proposals, each individual method has similar behaviors. In terms of object recall at $\eta=0.5$, ours and \cite{Alexe2012pami} perform very similarly, and both outperform \cite{Rahtu_iccv11} in most cases. But in terms of localization quality of proposals, \cite{Rahtu_iccv11} performs best.}}\label{fig:recall-overlap}
\vspace{-0mm}
\end{figure*}

\begin{figure*}[h]

 \centerline{\includegraphics[width=\columnwidth]{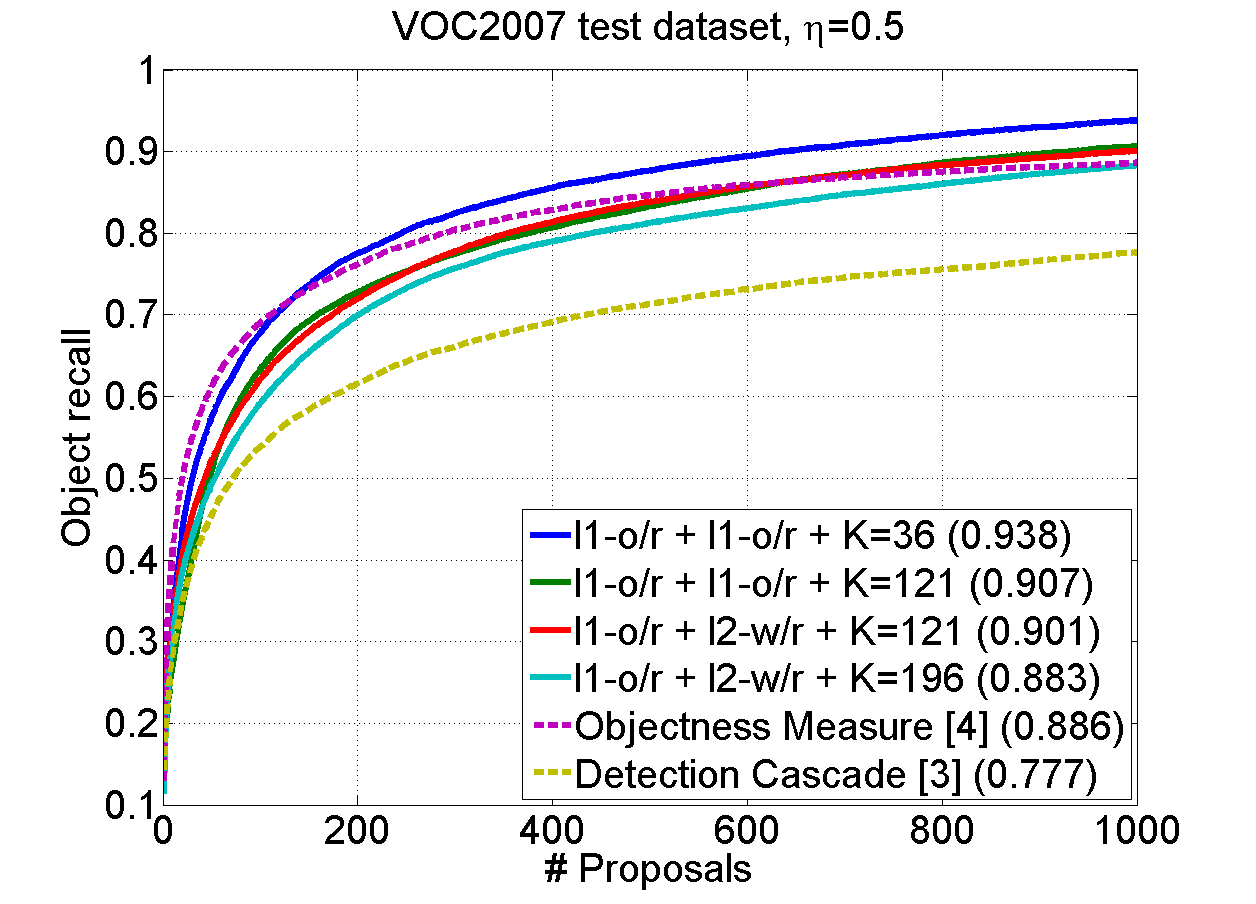}}
\caption{{\footnotesize Comparison of recall-proposal curves using different methods on VOC2007. The numbers in the brackets are the object recalls using 1000 proposals. Among the 4 cascade settings, $\ell_1-o/r+\ell_1-o/r$ performs best. Still our method and \cite{Alexe2012pami} have similar behaviors on both datasets, and both outperform \cite{Rahtu_iccv11} significantly. Using 1000 proposals, our method outperforms \cite{Alexe2012pami,Rahtu_iccv11} by 5.2\% and 16.1\% on VOC2007, respectively.}}\label{fig:recall-proposal}
\vspace{-0mm}
\end{figure*}

\begin{figure*}[t]
\begin{minipage}[b]{0.245\linewidth}
 \begin{center}
 \centerline{\includegraphics[width=1.05\columnwidth]{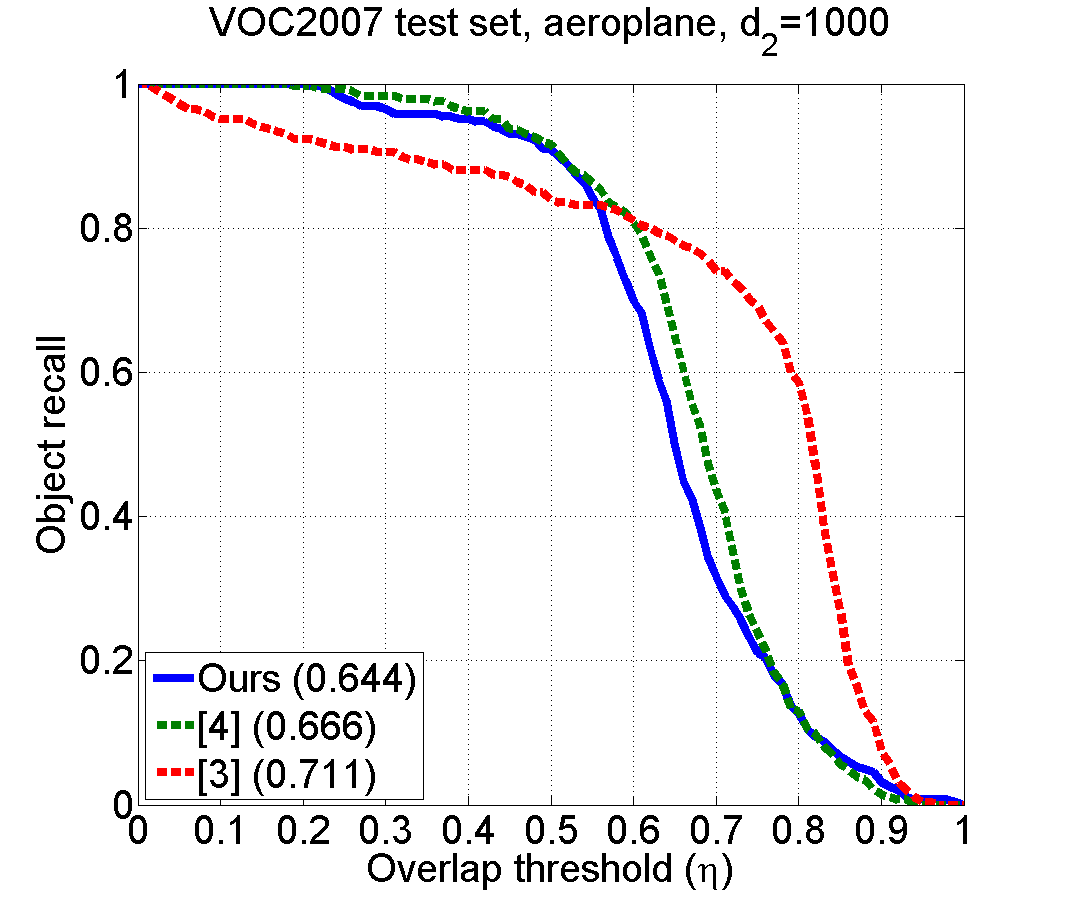}}
 \end{center}
\end{minipage}
\begin{minipage}[b]{0.245\linewidth}
 \begin{center}
 \centerline{\includegraphics[width=1.05\columnwidth]{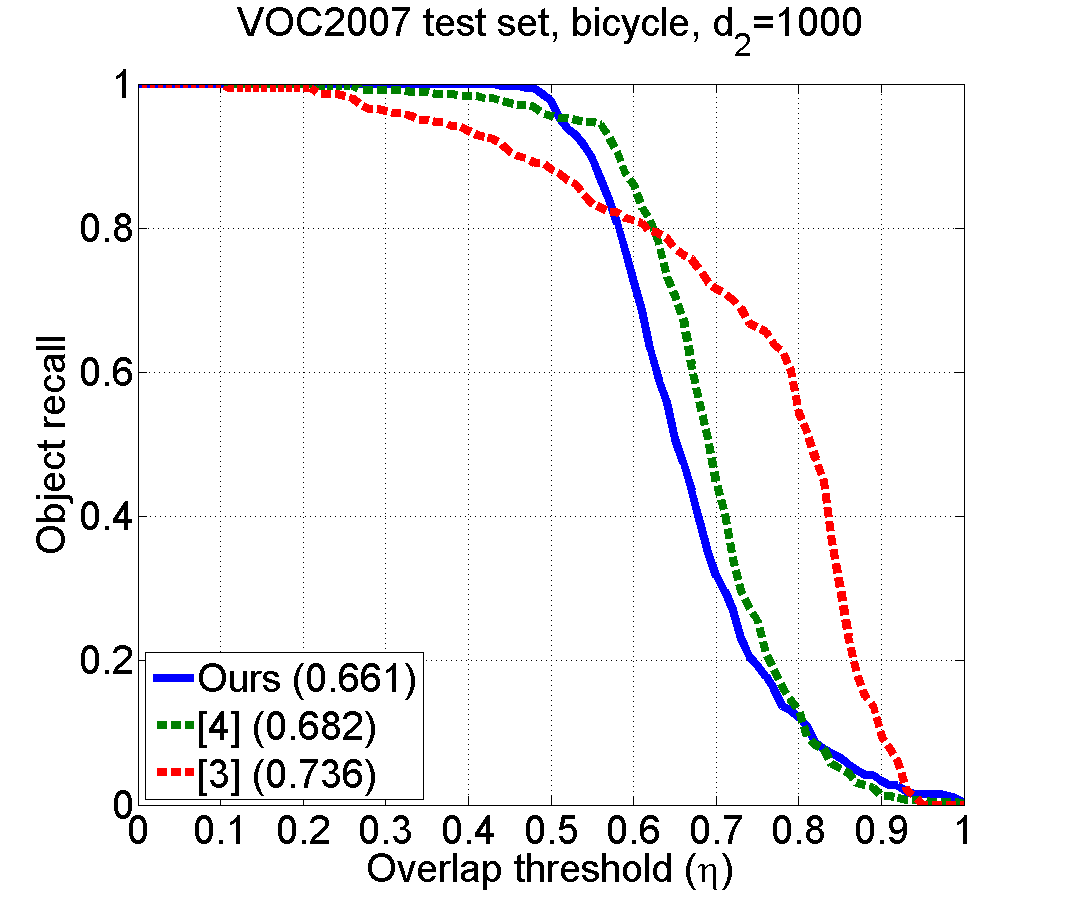}}
 \end{center}
\end{minipage}
\begin{minipage}[b]{0.245\linewidth}
 \begin{center}
 \centerline{\includegraphics[width=1.05\columnwidth]{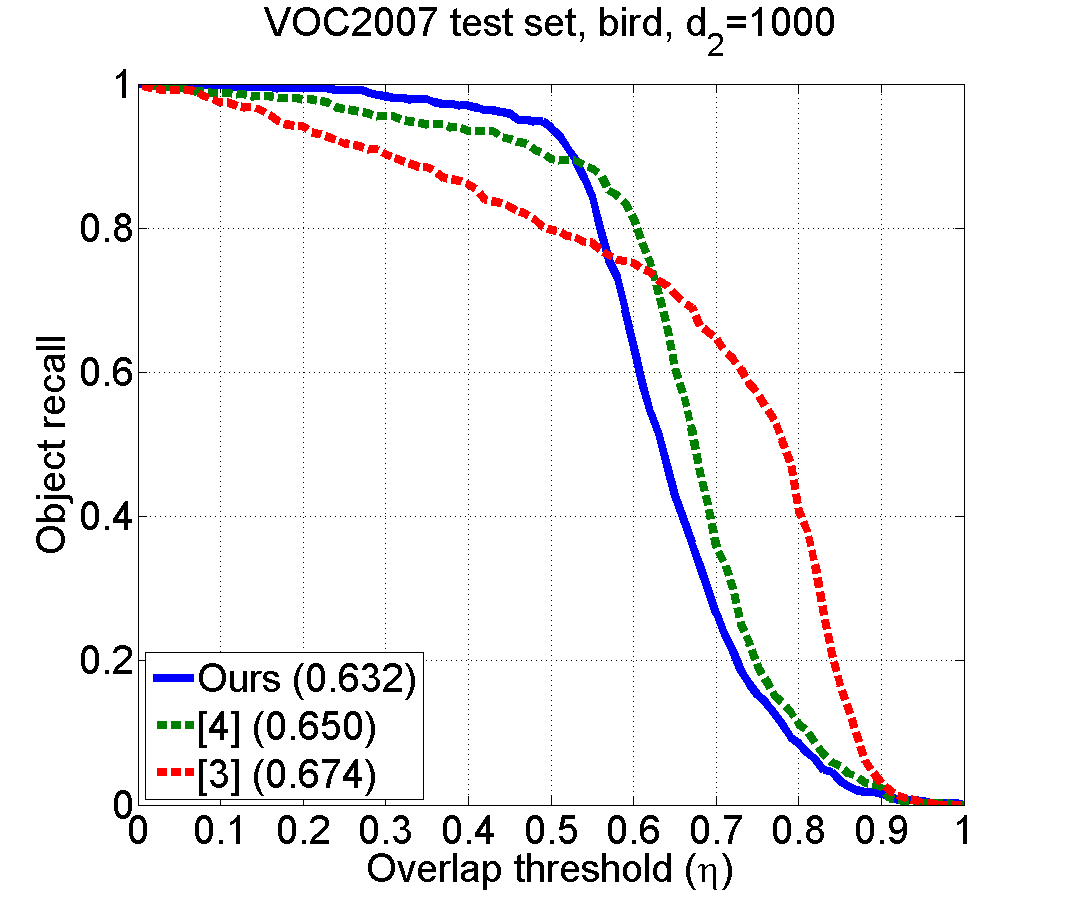}}
 \end{center}
\end{minipage}
\begin{minipage}[b]{0.245\linewidth}
 \begin{center}
 \centerline{\includegraphics[width=1.05\columnwidth]{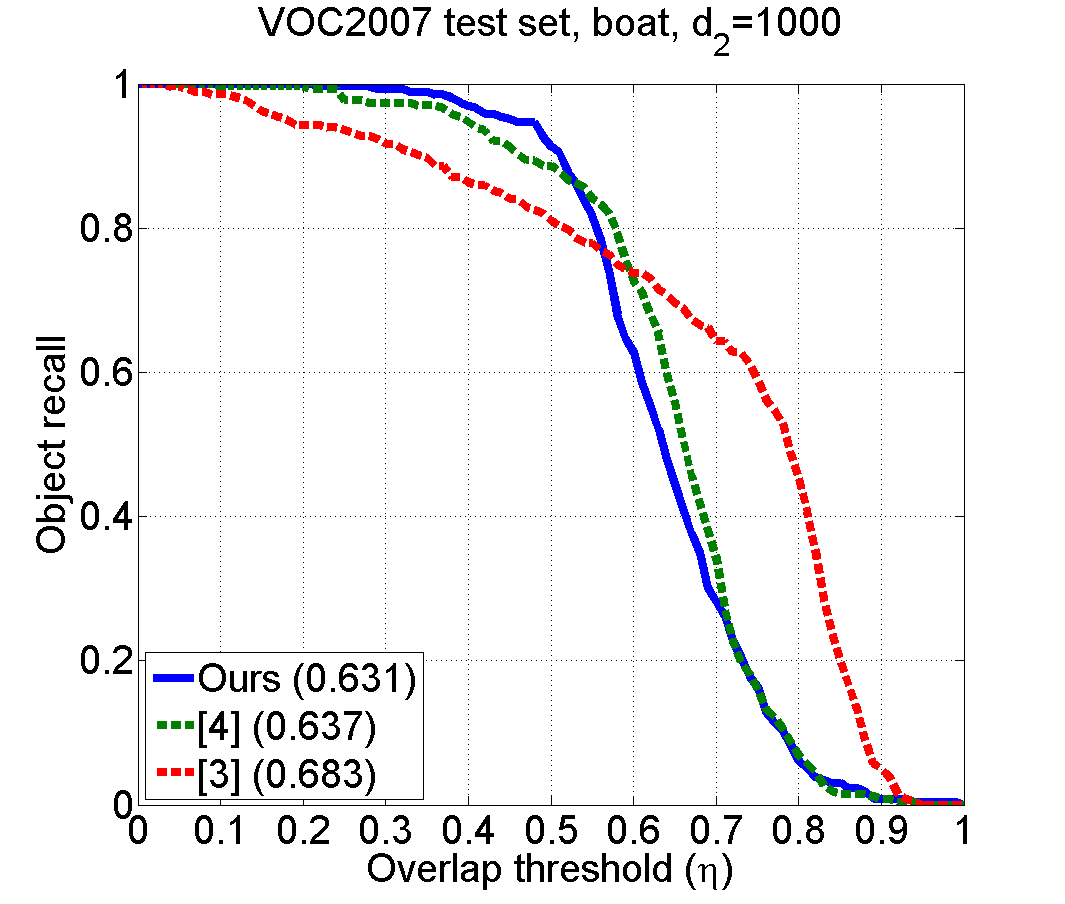}}
 \end{center}
\end{minipage}
\begin{minipage}[b]{0.245\linewidth}
 \begin{center}
 \centerline{\includegraphics[width=1.05\columnwidth]{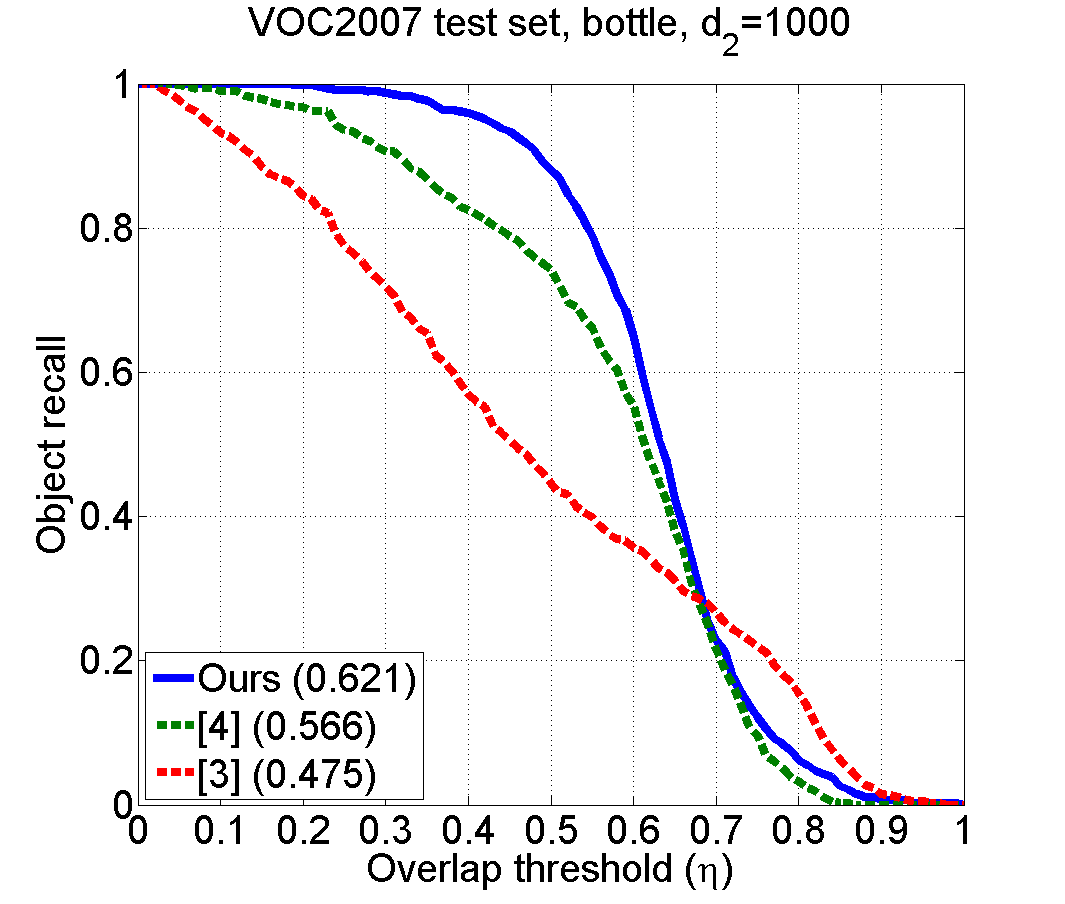}}
 \end{center}
\end{minipage}
\begin{minipage}[b]{0.245\linewidth}
 \begin{center}
 \centerline{\includegraphics[width=1.05\columnwidth]{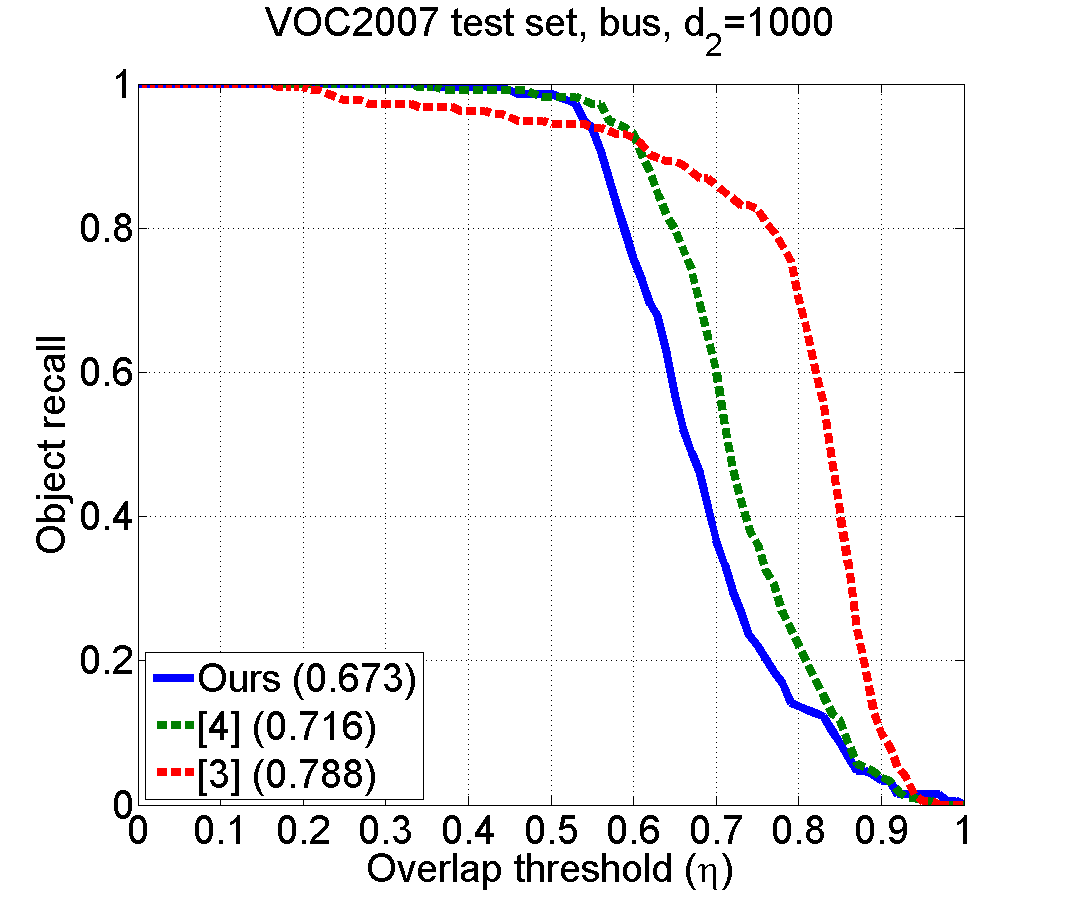}}
 \end{center}
\end{minipage}
\begin{minipage}[b]{0.245\linewidth}
 \begin{center}
 \centerline{\includegraphics[width=1.05\columnwidth]{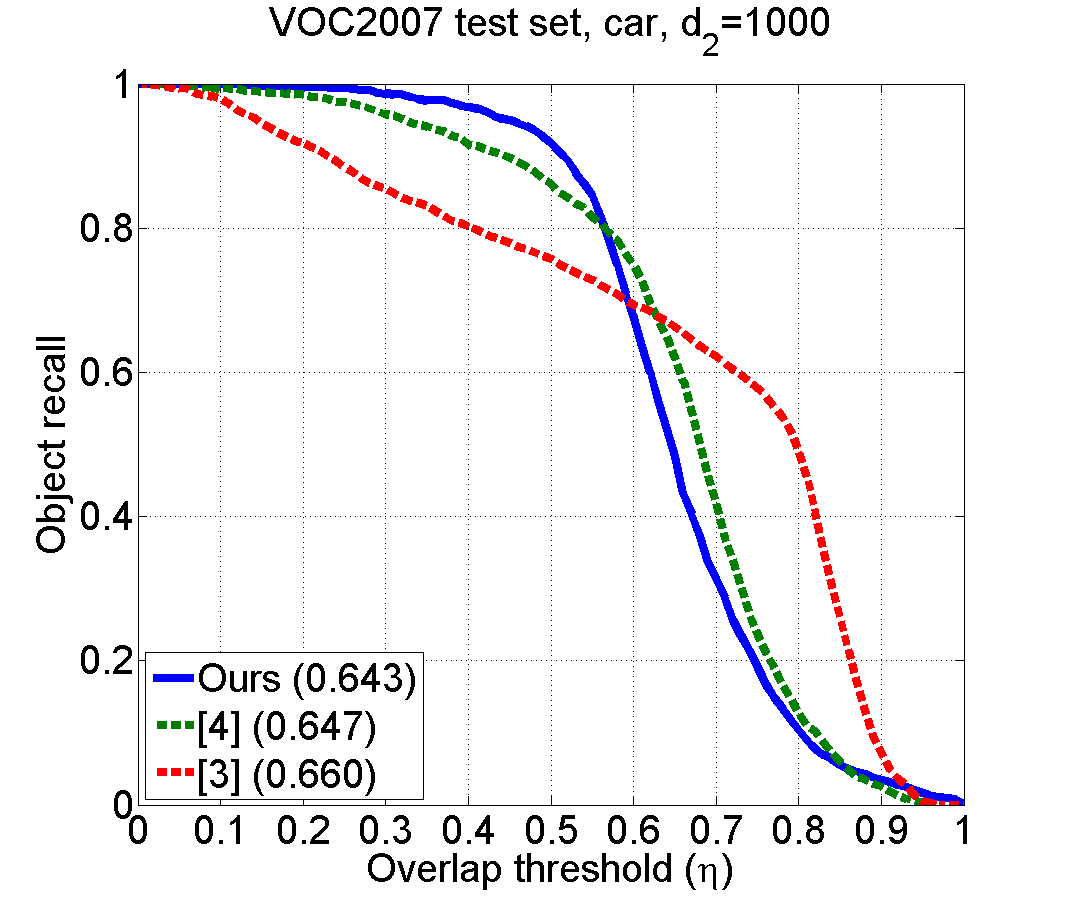}}
 \end{center}
\end{minipage}
\begin{minipage}[b]{0.245\linewidth}
 \begin{center}
 \centerline{\includegraphics[width=1.05\columnwidth]{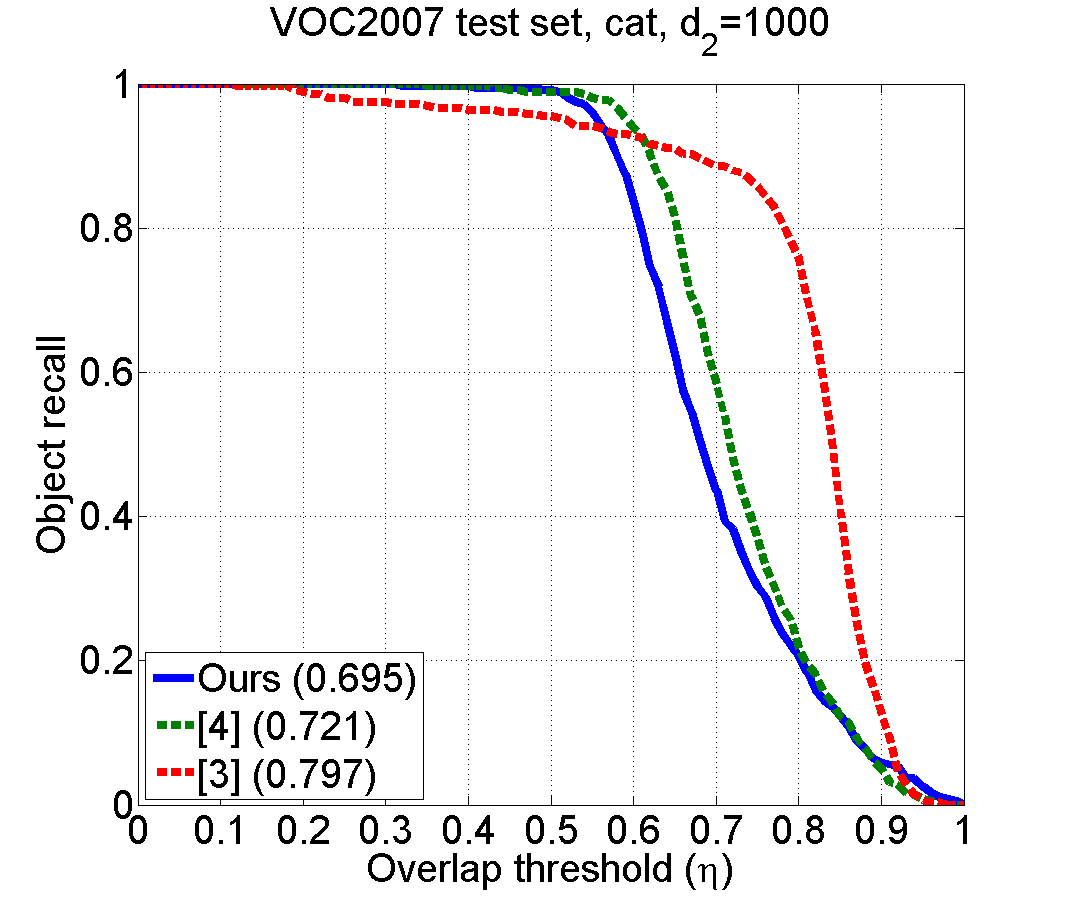}}
 \end{center}
\end{minipage}
\begin{minipage}[b]{0.245\linewidth}
 \begin{center}
 \centerline{\includegraphics[width=1.05\columnwidth]{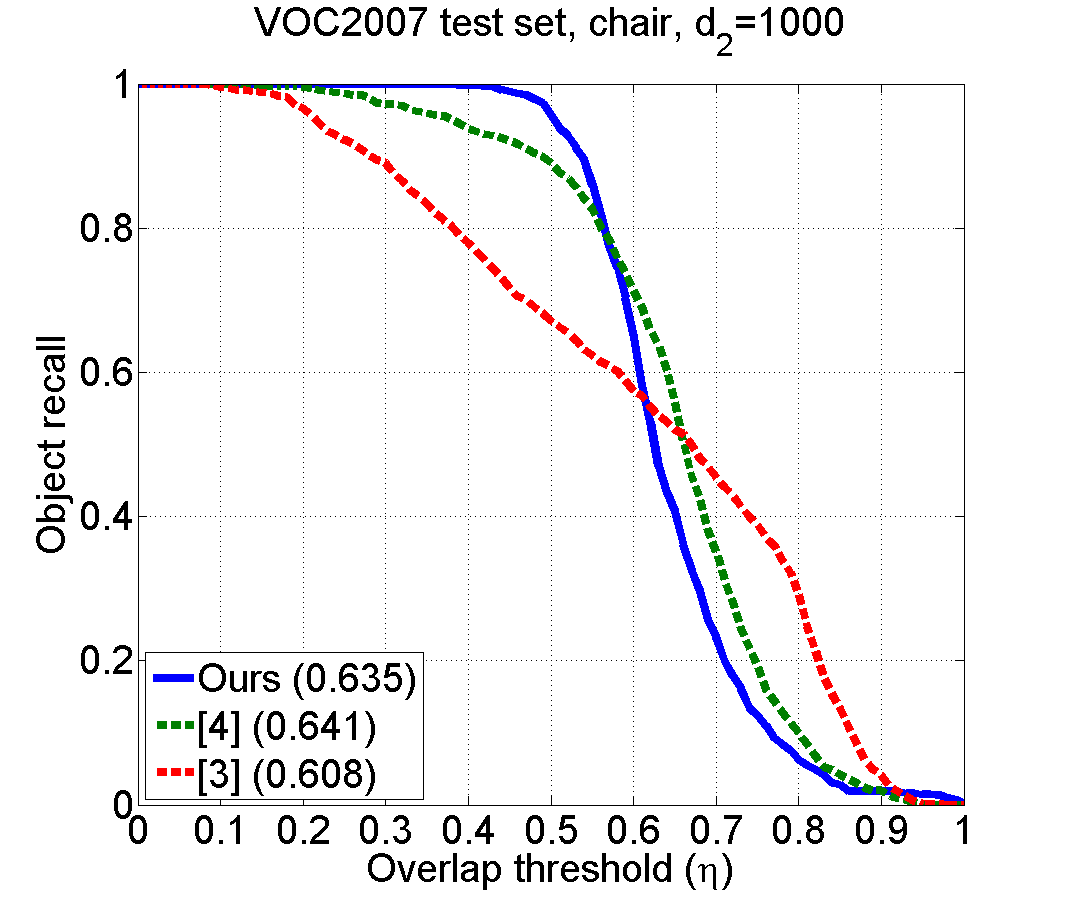}}
 \end{center}
\end{minipage}
\begin{minipage}[b]{0.245\linewidth}
 \begin{center}
 \centerline{\includegraphics[width=1.05\columnwidth]{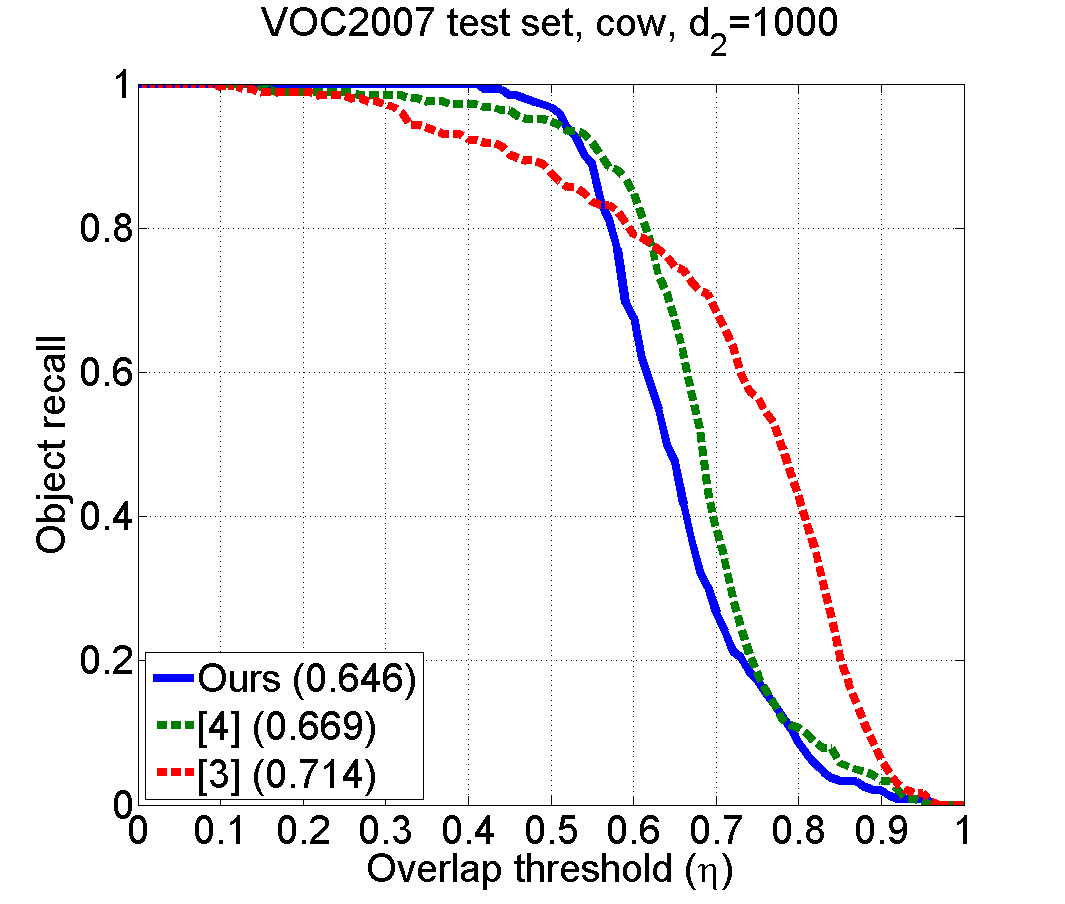}}
 \end{center}
\end{minipage}
\begin{minipage}[b]{0.245\linewidth}
 \begin{center}
 \centerline{\includegraphics[width=1.05\columnwidth]{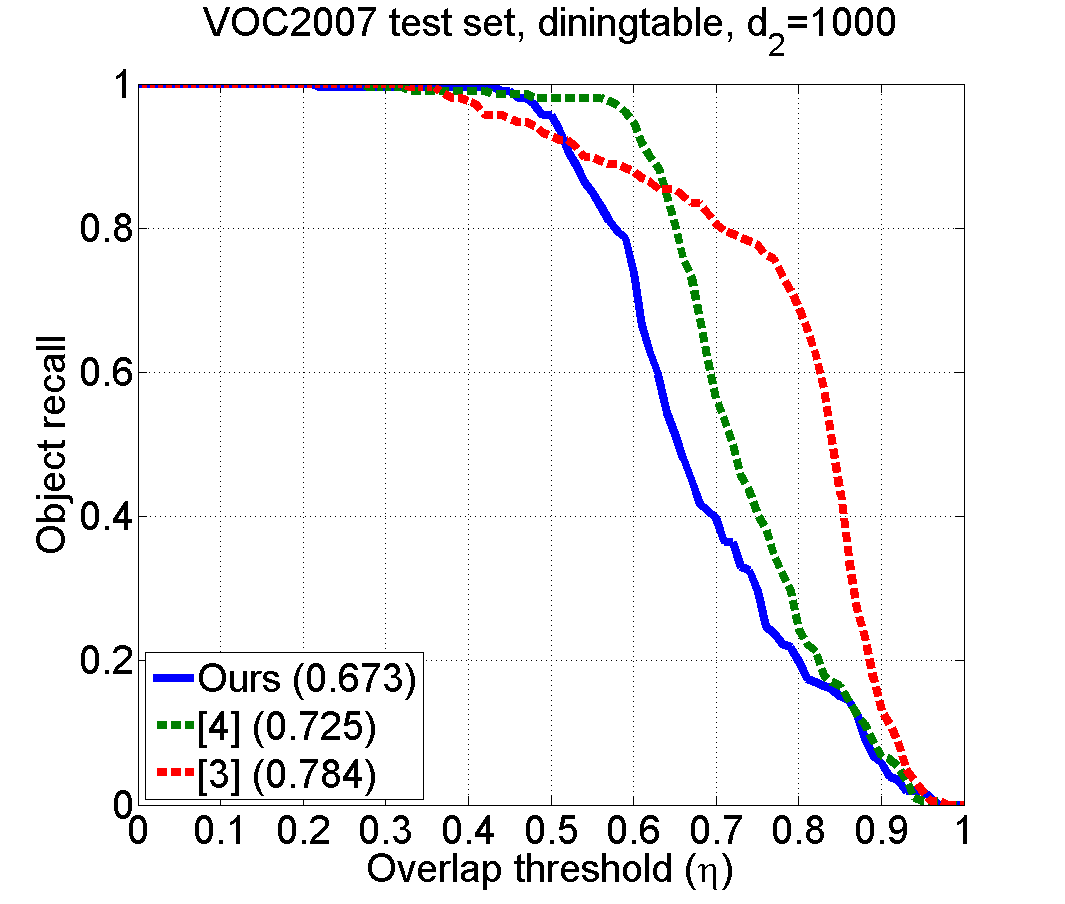}}
 \end{center}
\end{minipage}
\begin{minipage}[b]{0.245\linewidth}
 \begin{center}
 \centerline{\includegraphics[width=1.05\columnwidth]{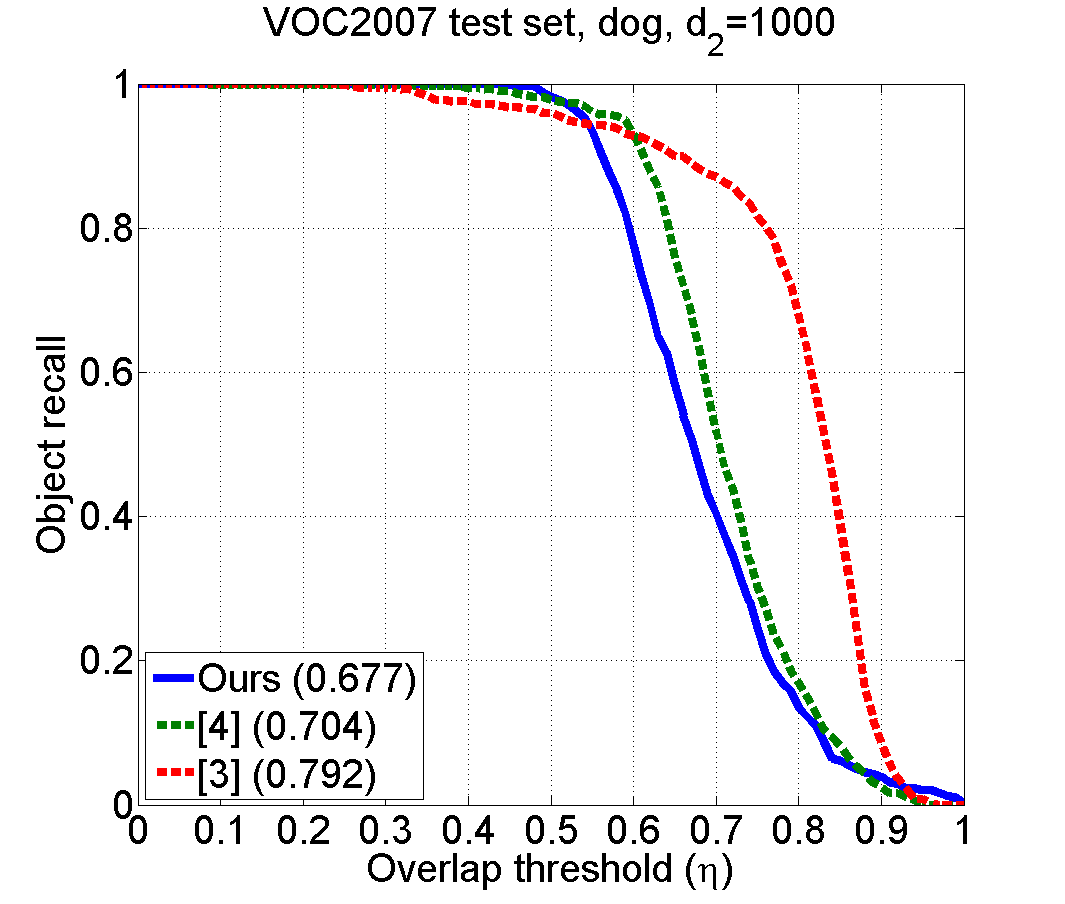}}
 \end{center}
\end{minipage}
\begin{minipage}[b]{0.245\linewidth}
 \begin{center}
 \centerline{\includegraphics[width=1.05\columnwidth]{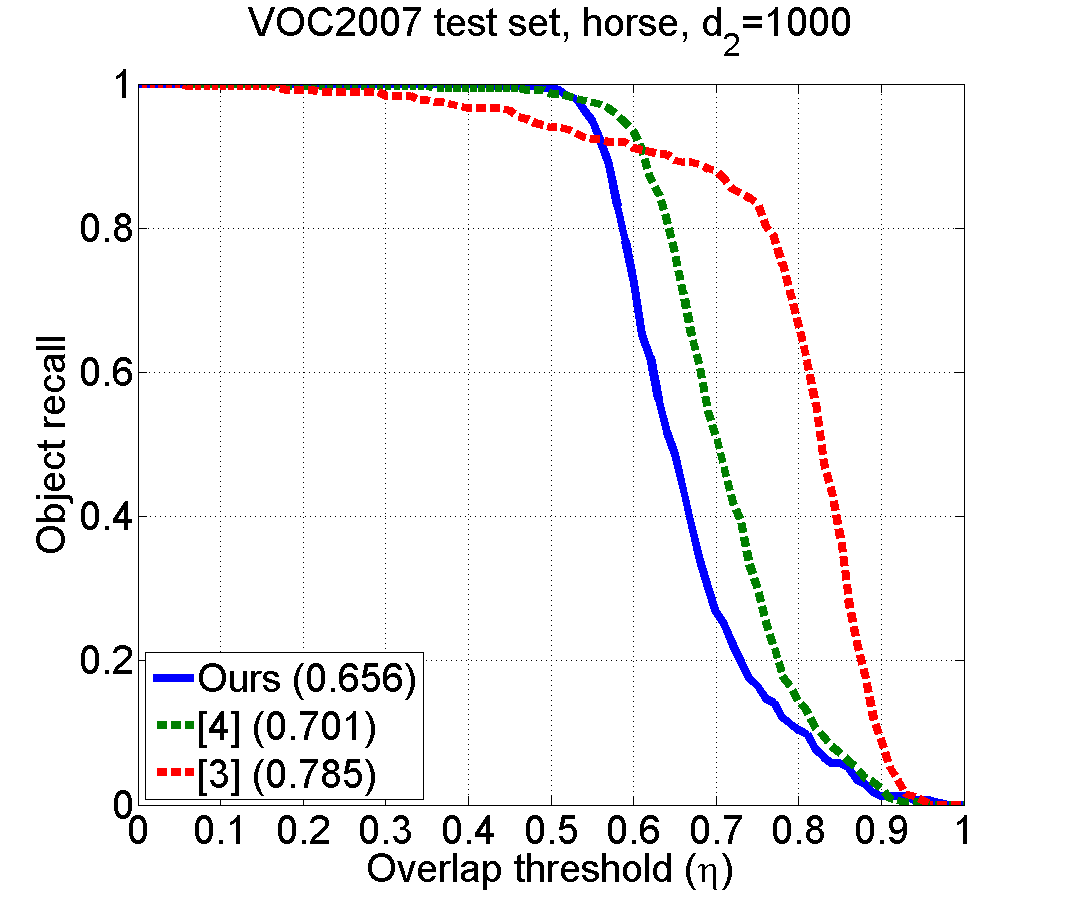}}
 \end{center}
\end{minipage}
\begin{minipage}[b]{0.245\linewidth}
 \begin{center}
 \centerline{\includegraphics[width=1.05\columnwidth]{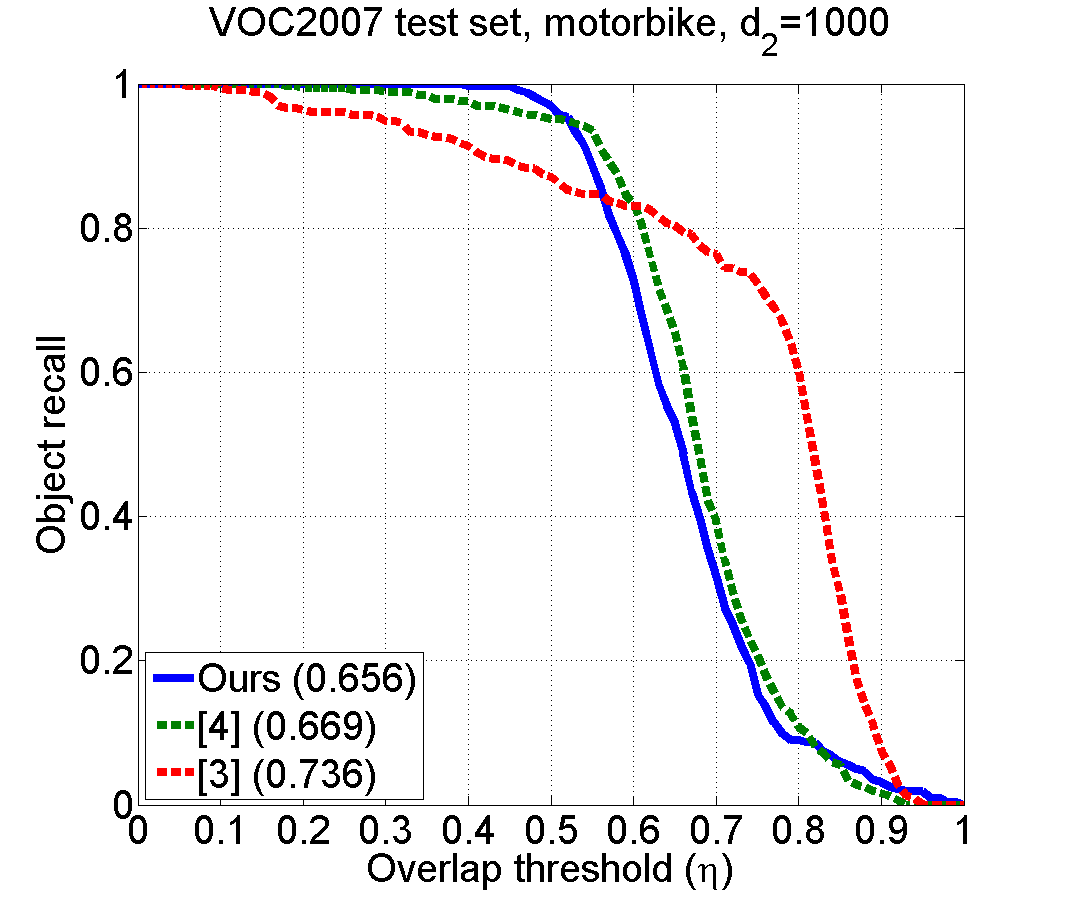}}
 \end{center}
\end{minipage}
\begin{minipage}[b]{0.245\linewidth}
 \begin{center}
 \centerline{\includegraphics[width=1.05\columnwidth]{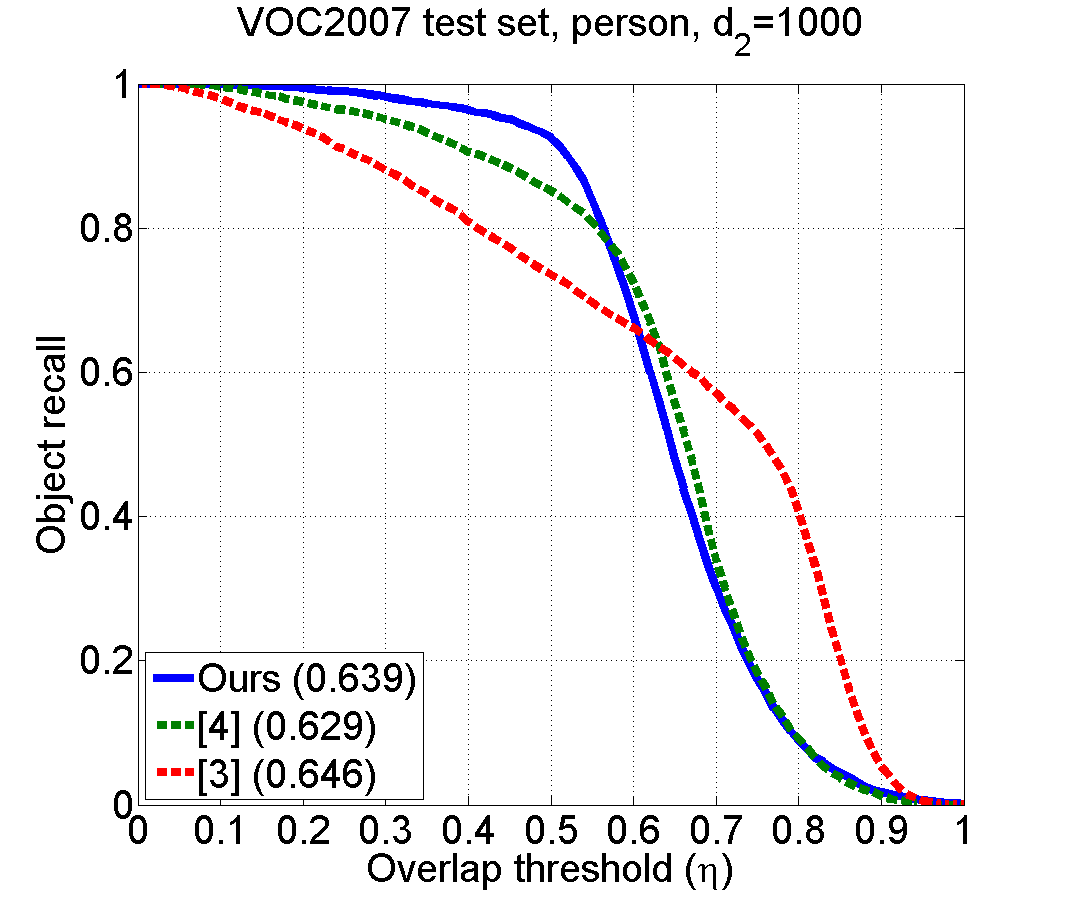}}
 \end{center}
\end{minipage}
\begin{minipage}[b]{0.245\linewidth}
 \begin{center}
 \centerline{\includegraphics[width=1.05\columnwidth]{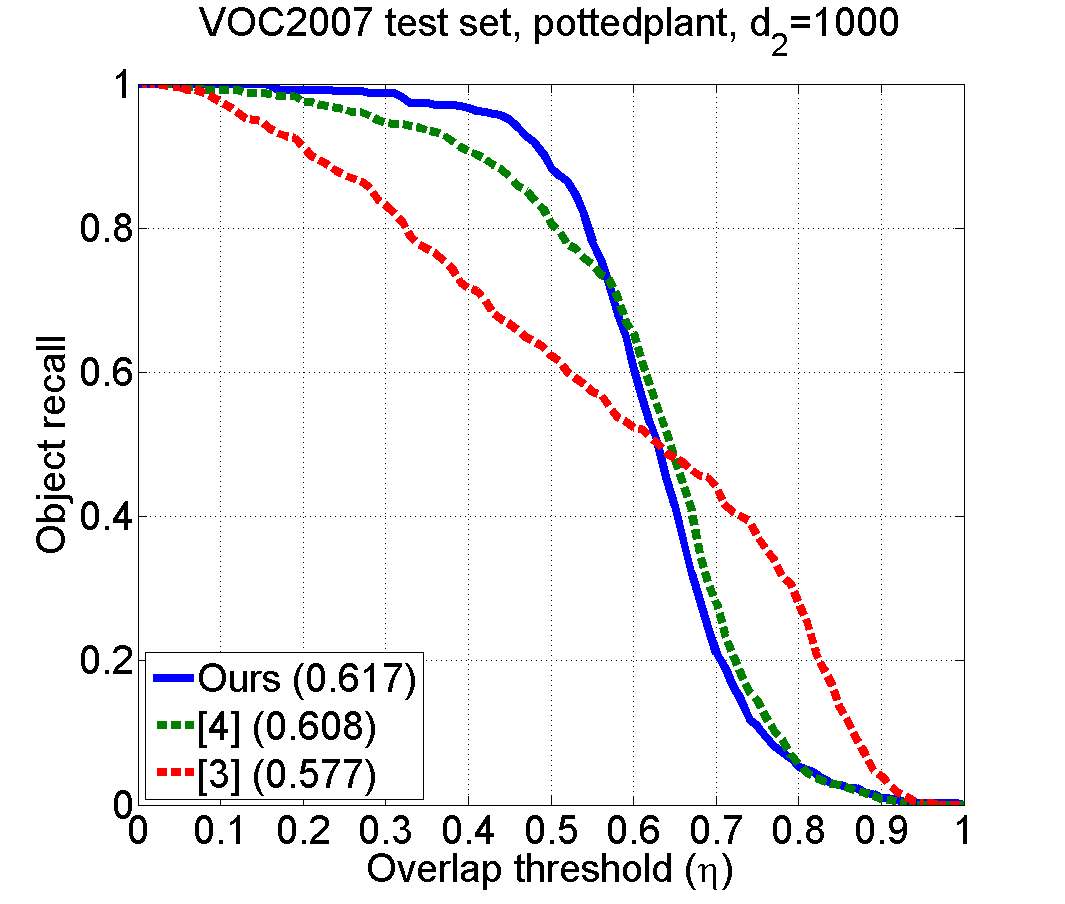}}
 \end{center}
\end{minipage}
\begin{minipage}[b]{0.245\linewidth}
 \begin{center}
 \centerline{\includegraphics[width=1.05\columnwidth]{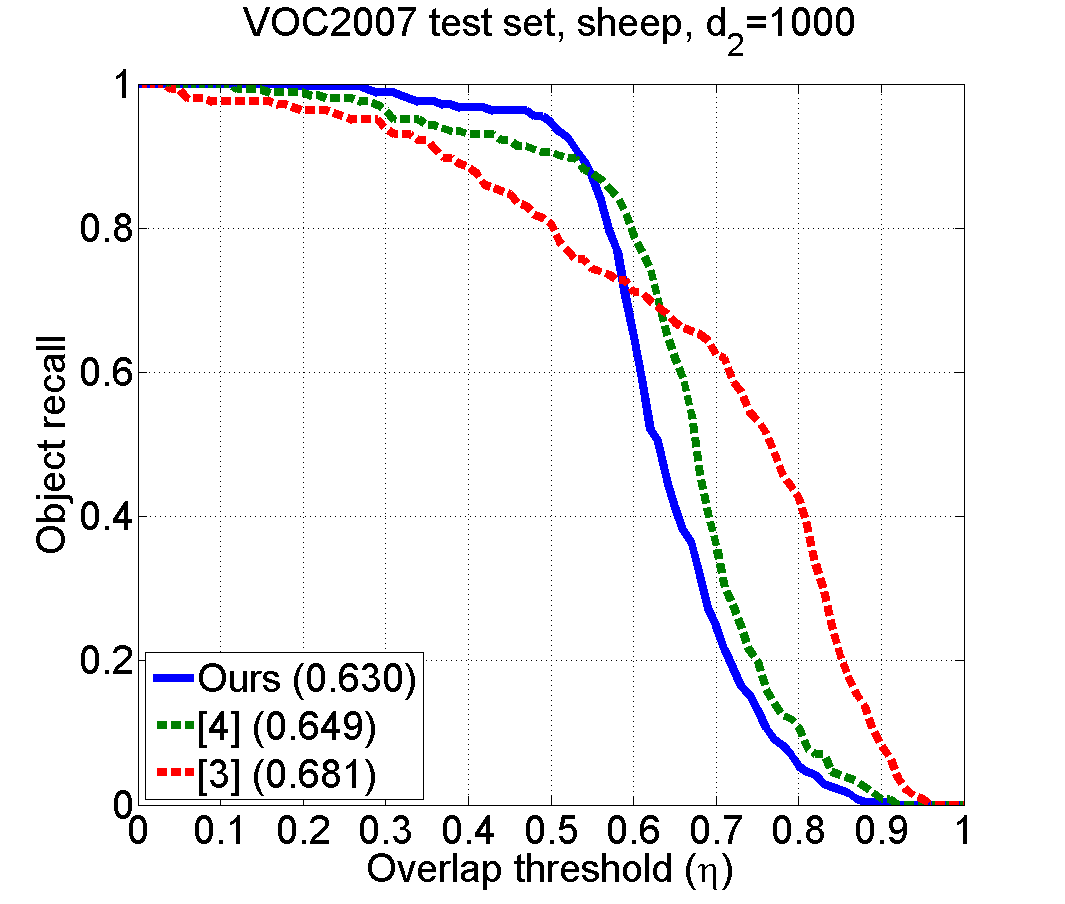}}
 \end{center}
\end{minipage}
\begin{minipage}[b]{0.245\linewidth}
 \begin{center}
 \centerline{\includegraphics[width=1.05\columnwidth]{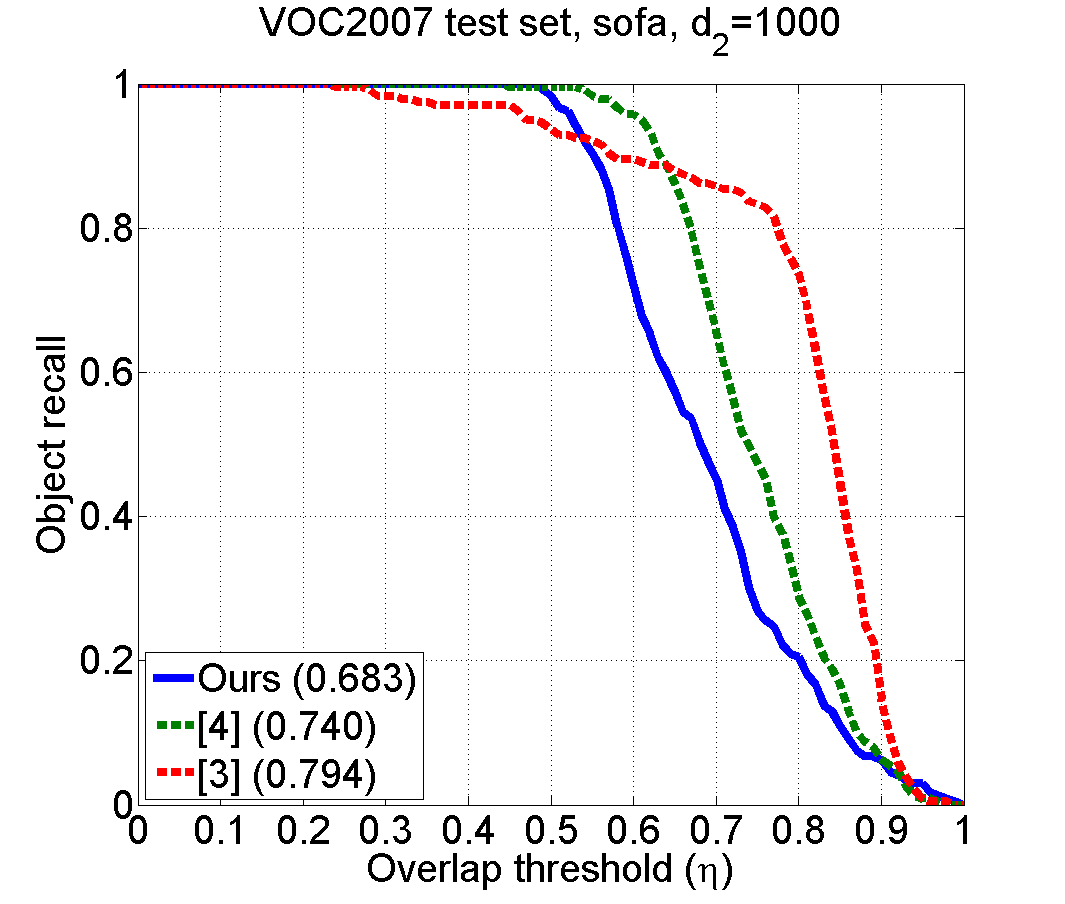}}
 \end{center}
\end{minipage}
\begin{minipage}[b]{0.245\linewidth}
 \begin{center}
 \centerline{\includegraphics[width=1.05\columnwidth]{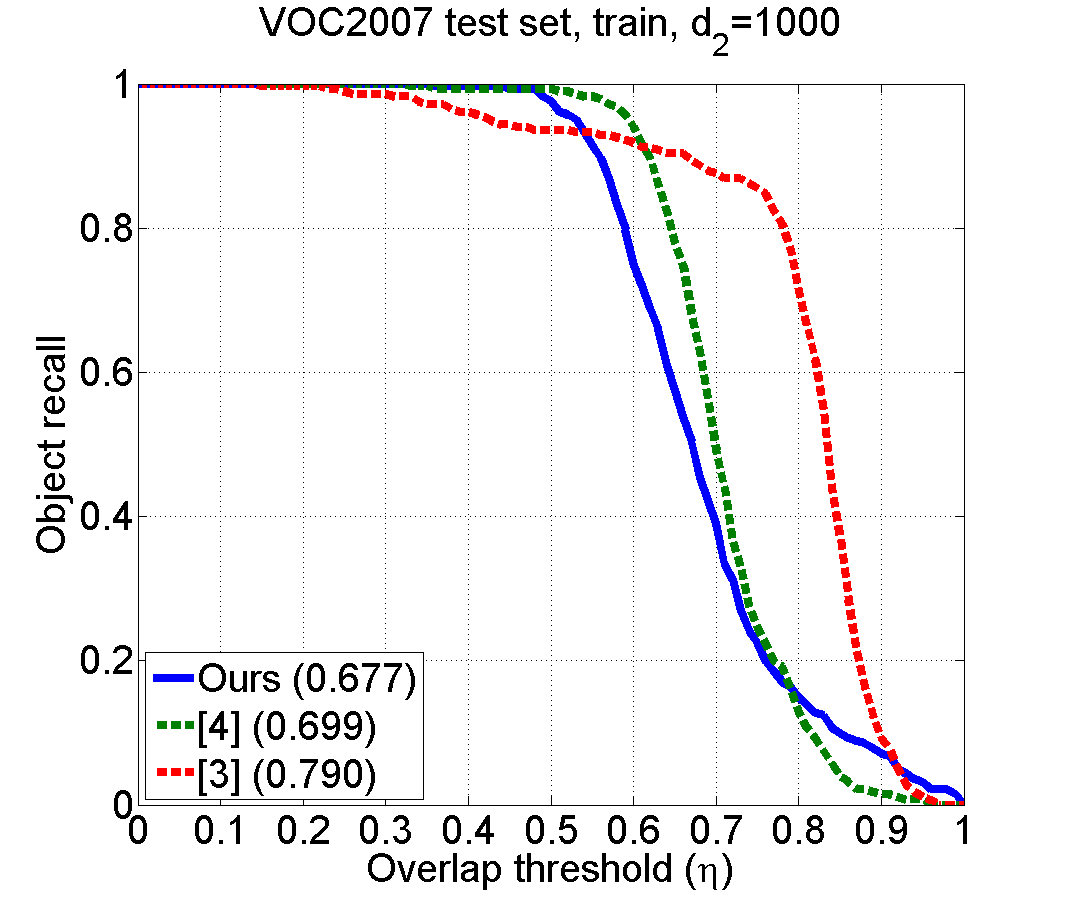}}
 \end{center}
\end{minipage}
\begin{minipage}[b]{0.245\linewidth}
 \begin{center}
 \centerline{\includegraphics[width=1.05\columnwidth]{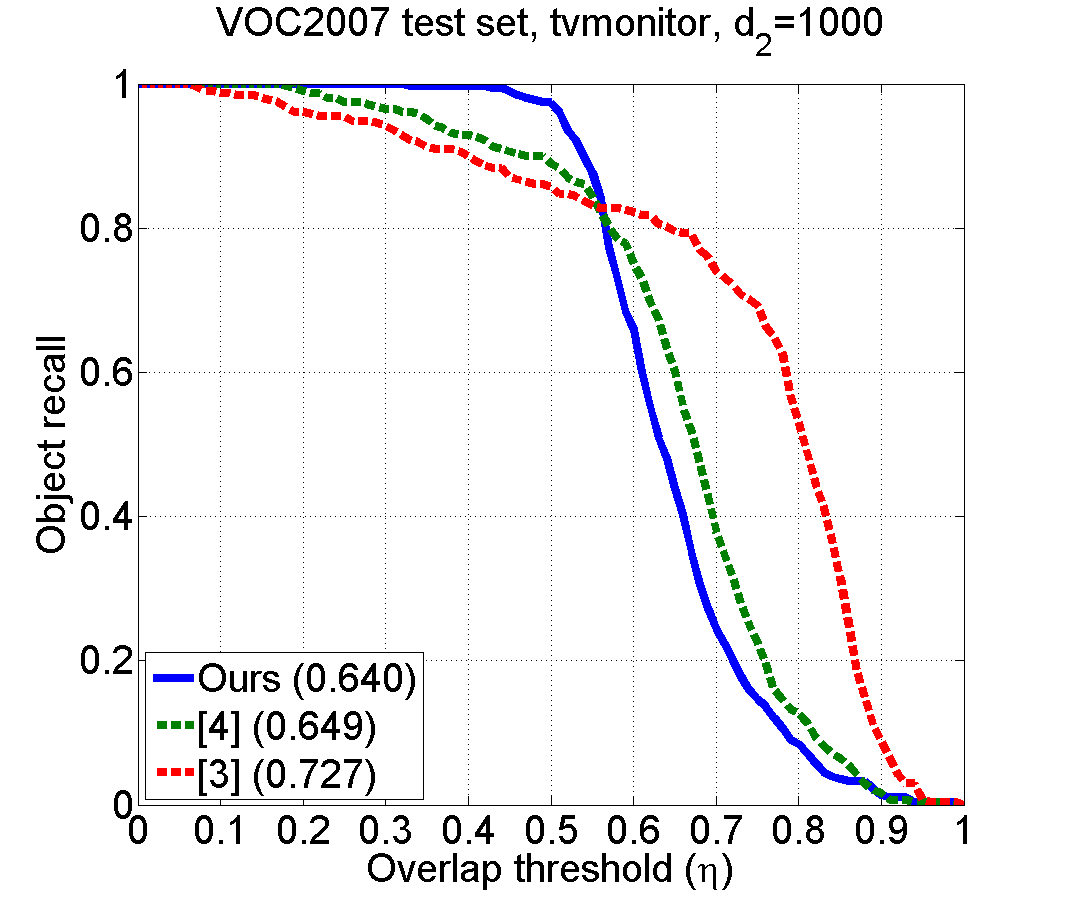}}
 \end{center}
\end{minipage}
\caption{{\footnotesize Comparison of recall-overlap curves using different methods on each class in the test dataset of VOC2007. The numbers in brackets are the AUC scores for each method. In general, our method ({\it i.e.} $\ell_1-o/r + \ell_1-o/r$) performs similarly to \cite{Alexe2012pami}, and when $\eta>0.5$ \cite{Rahtu_iccv11} seems better than ours and \cite{Alexe2012pami} in terms of localization quality of proposals. The mean and standard deviation of AUC scores for our method, \cite{Alexe2012pami,Rahtu_iccv11} are $(65.1\pm 2.2)\%$, $(66.8\pm 4.2)\%$, and $(70.8\pm 8.3)\%$, respectively.}}\label{fig:recall-overlap-class-2007}
\vspace{-0mm}
\end{figure*}

\begin{figure*}[t]
\begin{minipage}[b]{0.245\linewidth}
 \begin{center}
 \centerline{\includegraphics[width=1.05\columnwidth]{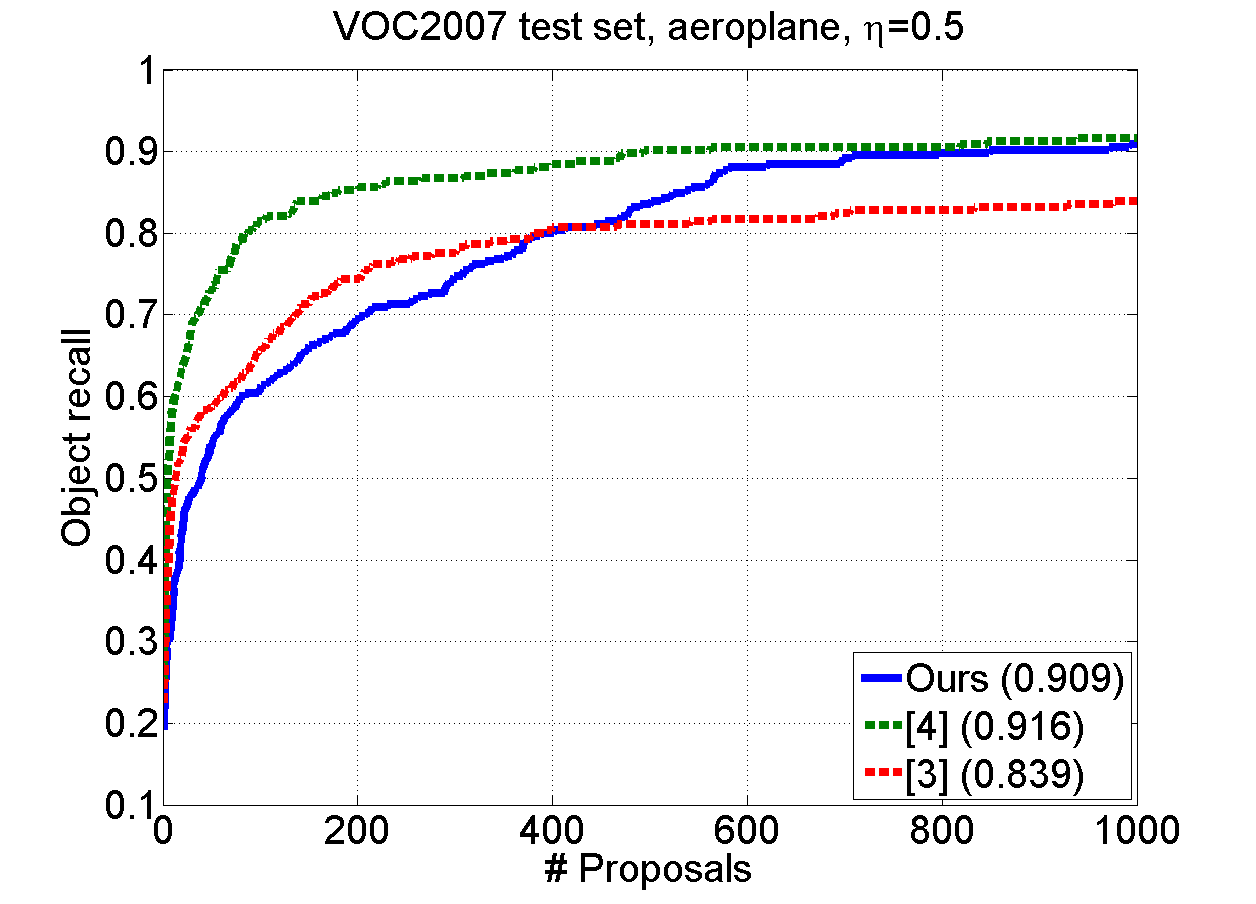}}
 \end{center}
\end{minipage}
\begin{minipage}[b]{0.245\linewidth}
 \begin{center}
 \centerline{\includegraphics[width=1.05\columnwidth]{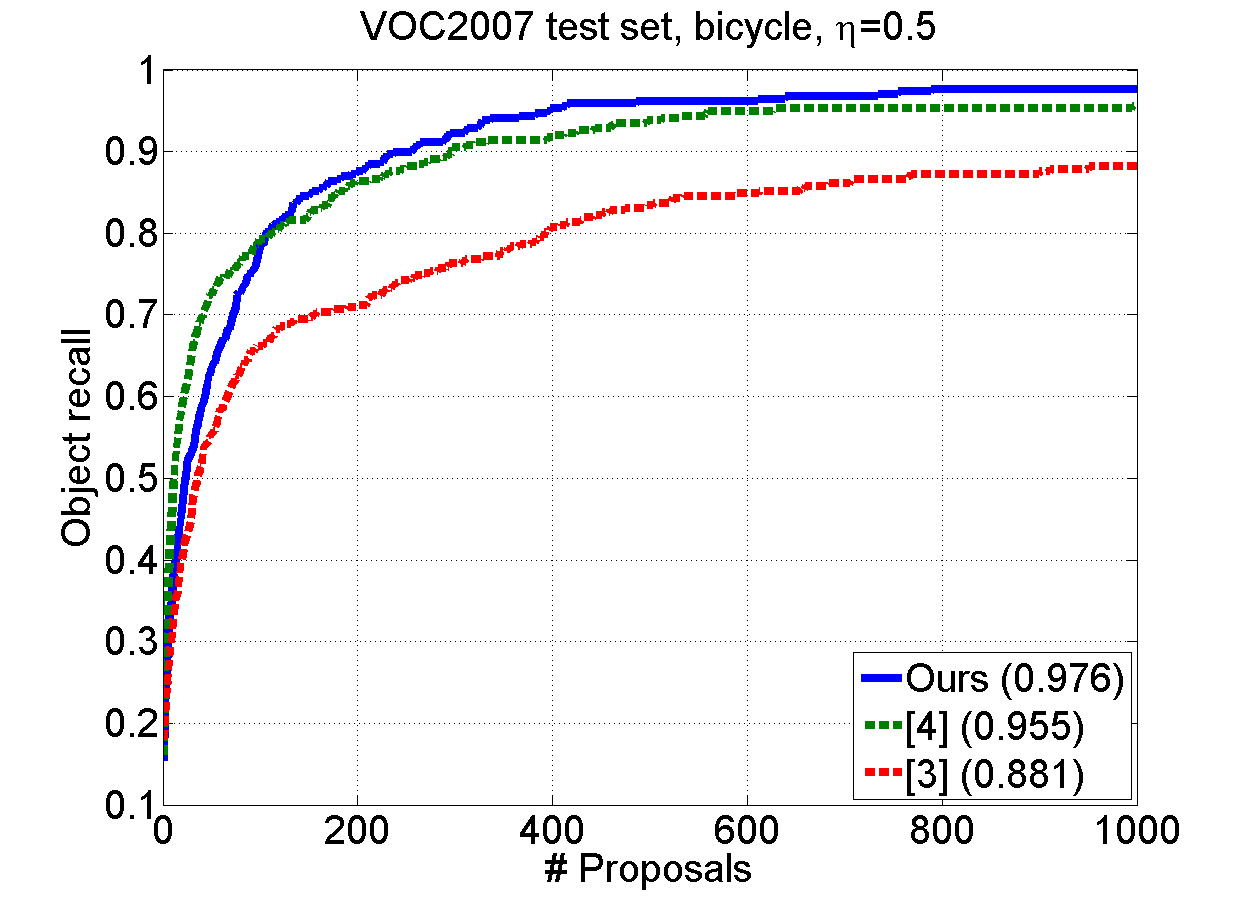}}
 \end{center}
\end{minipage}
\begin{minipage}[b]{0.245\linewidth}
 \begin{center}
 \centerline{\includegraphics[width=1.05\columnwidth]{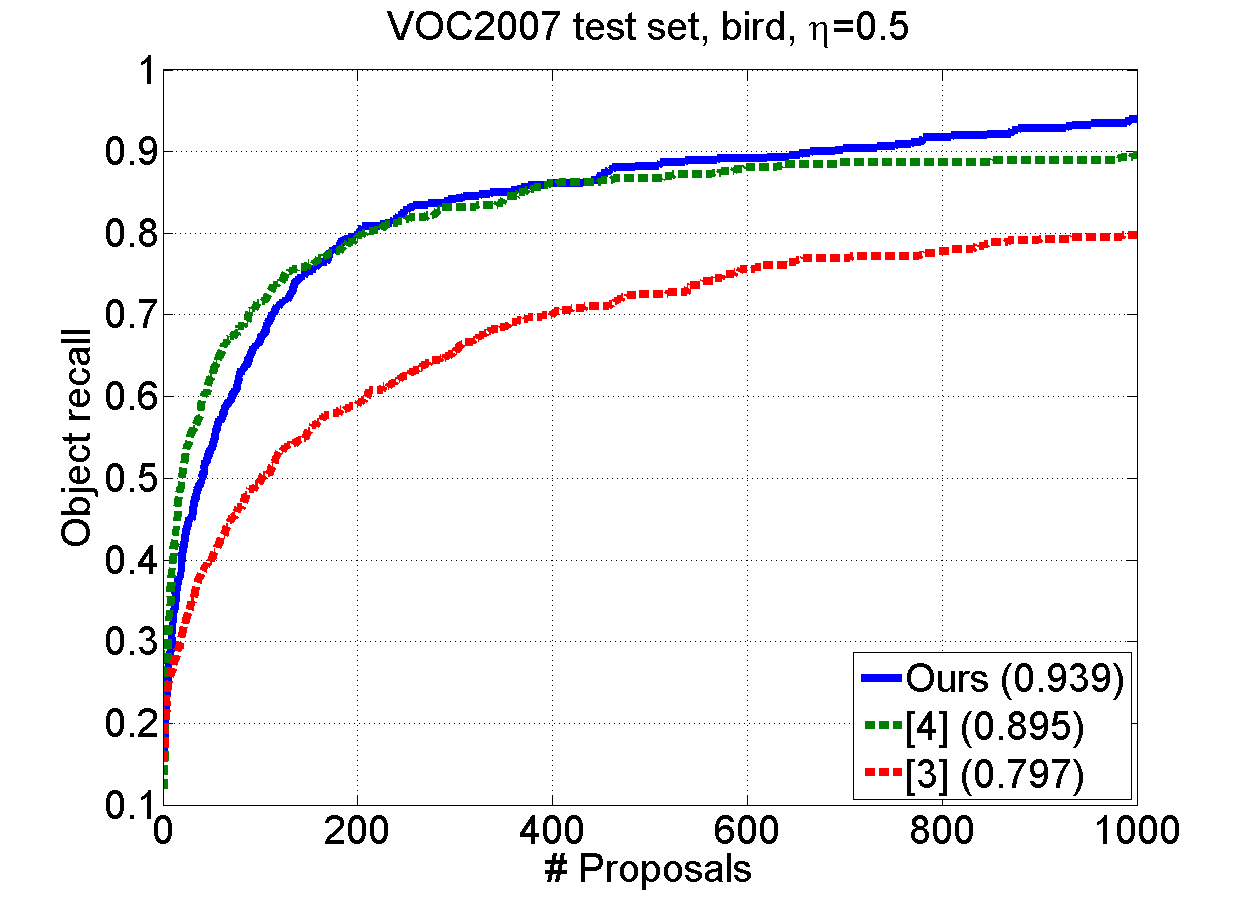}}
 \end{center}
\end{minipage}
\begin{minipage}[b]{0.245\linewidth}
 \begin{center}
 \centerline{\includegraphics[width=1.05\columnwidth]{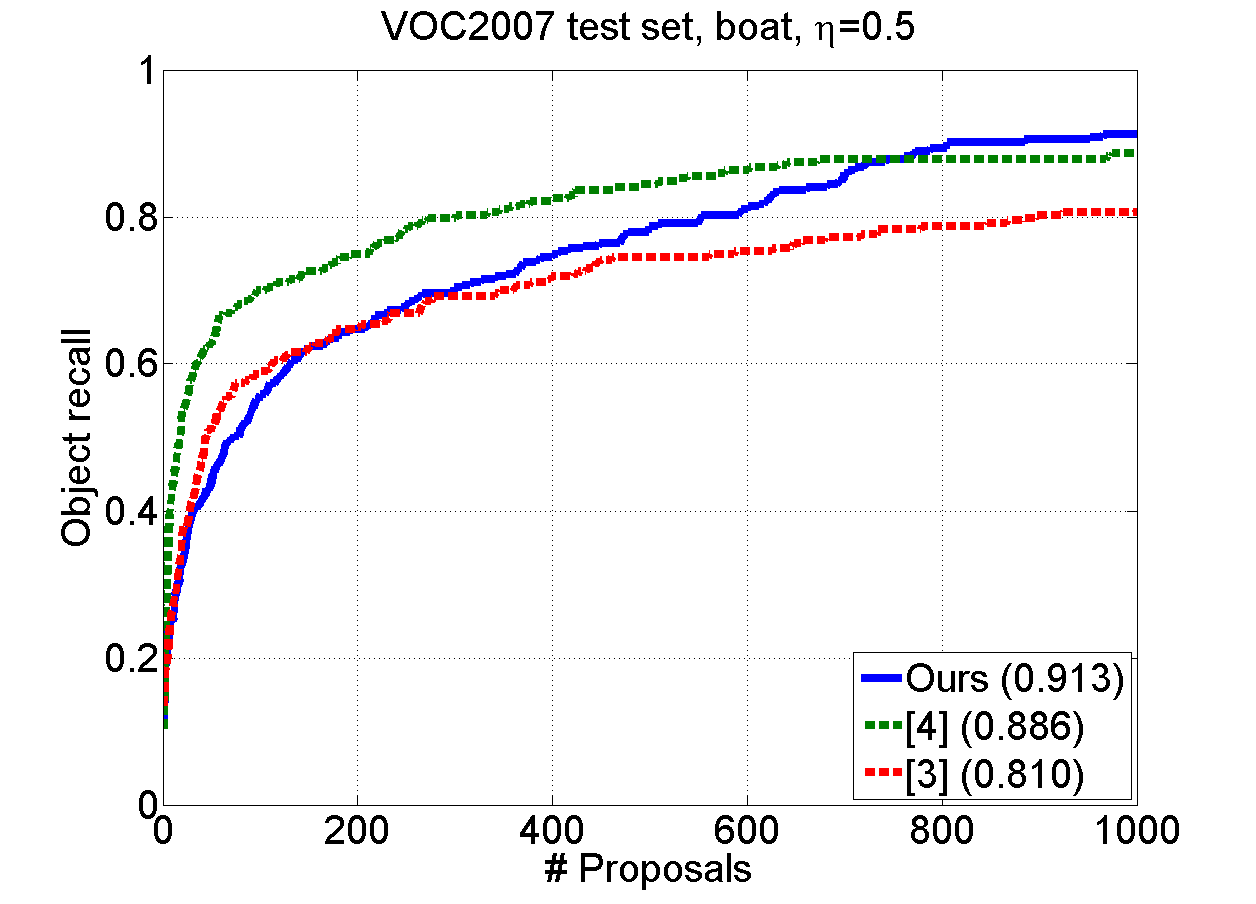}}
 \end{center}
\end{minipage}
\begin{minipage}[b]{0.245\linewidth}
 \begin{center}
 \centerline{\includegraphics[width=1.05\columnwidth]{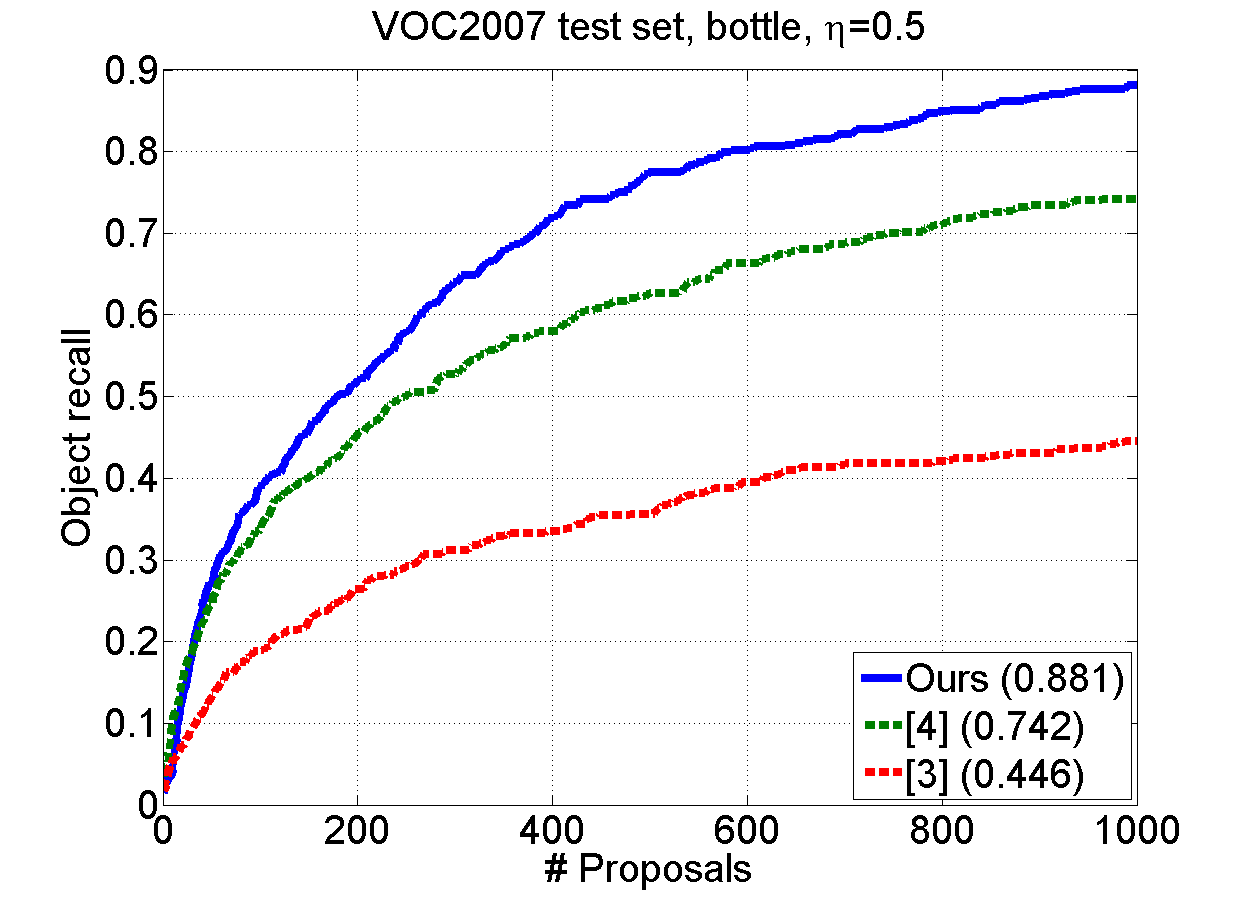}}
 \end{center}
\end{minipage}
\begin{minipage}[b]{0.245\linewidth}
 \begin{center}
 \centerline{\includegraphics[width=1.05\columnwidth]{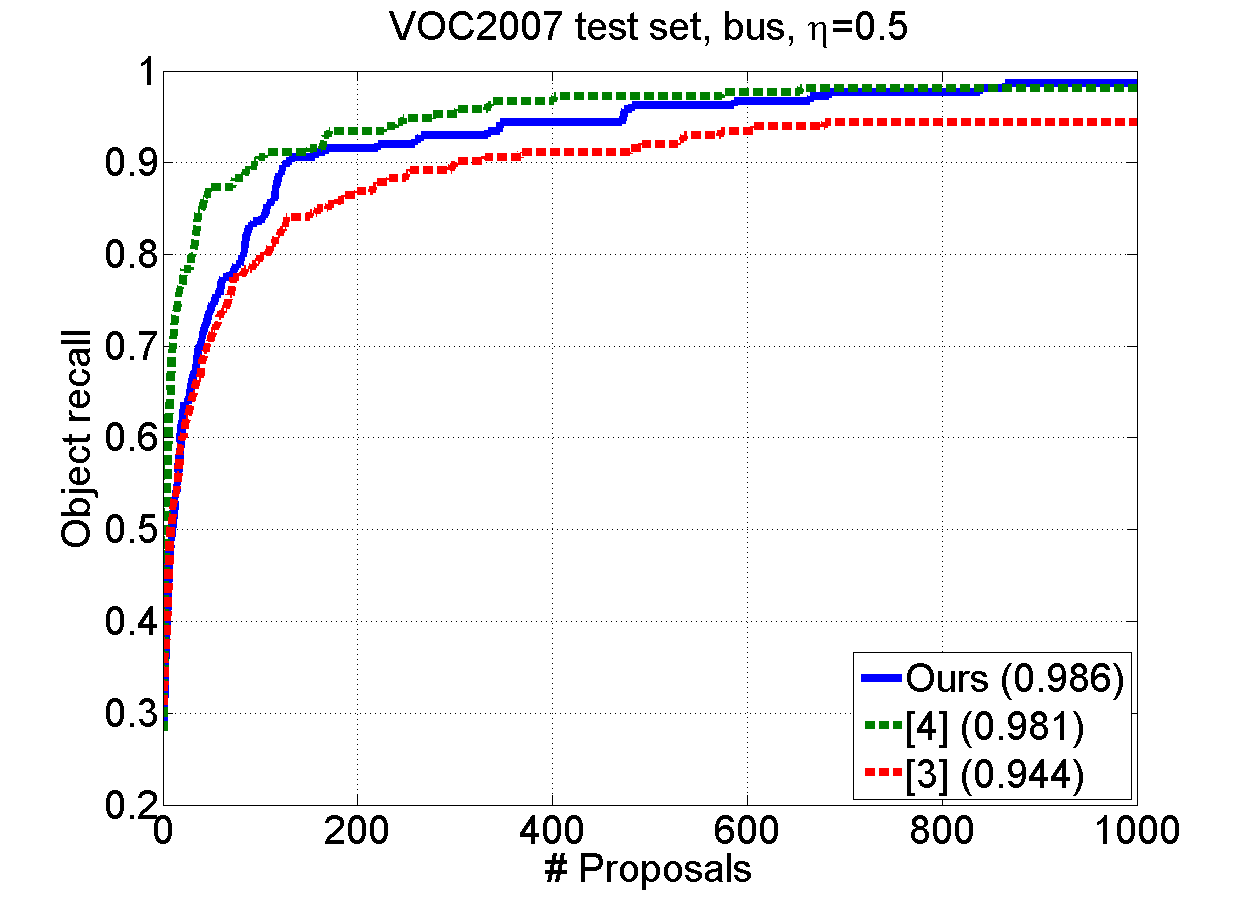}}
 \end{center}
\end{minipage}
\begin{minipage}[b]{0.245\linewidth}
 \begin{center}
 \centerline{\includegraphics[width=1.05\columnwidth]{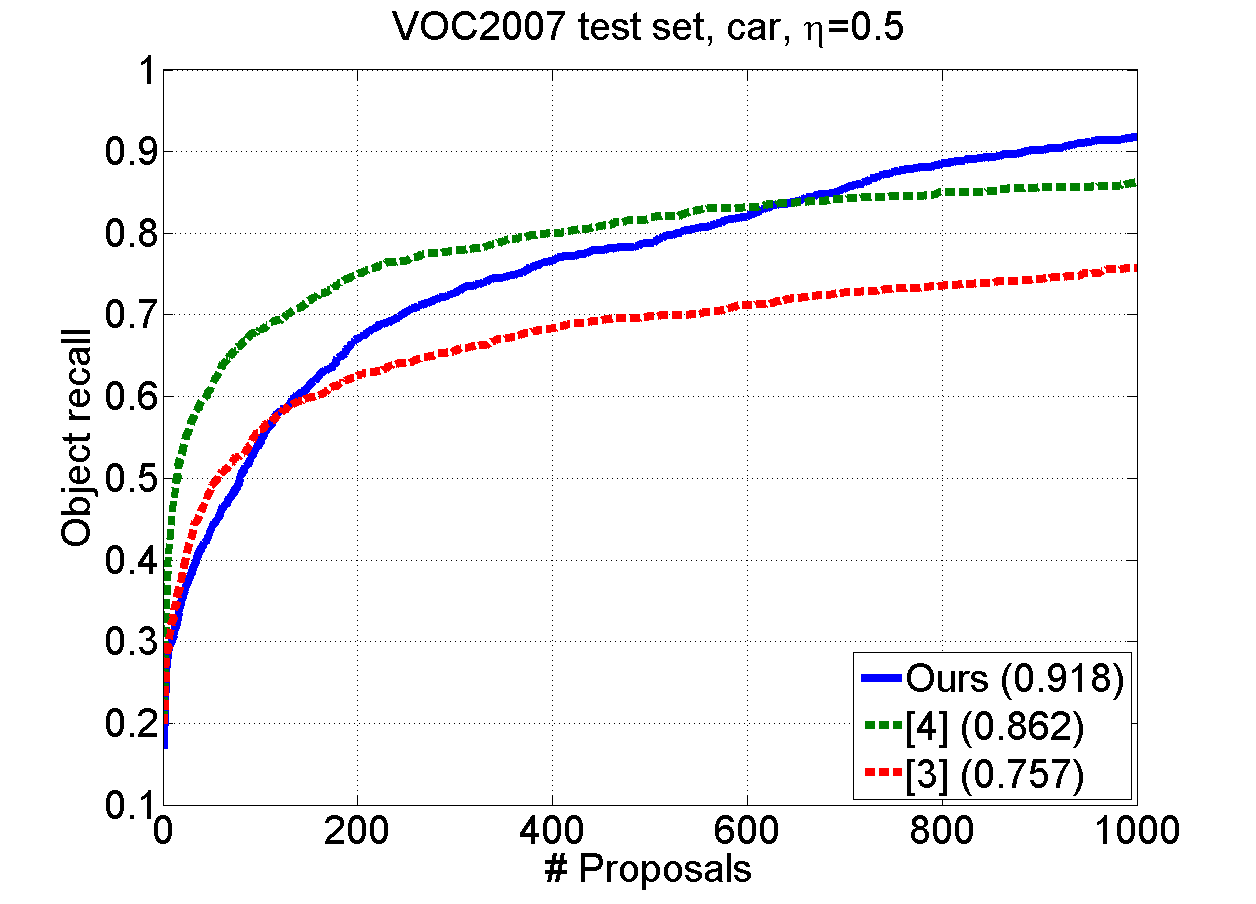}}
 \end{center}
\end{minipage}
\begin{minipage}[b]{0.245\linewidth}
 \begin{center}
 \centerline{\includegraphics[width=1.05\columnwidth]{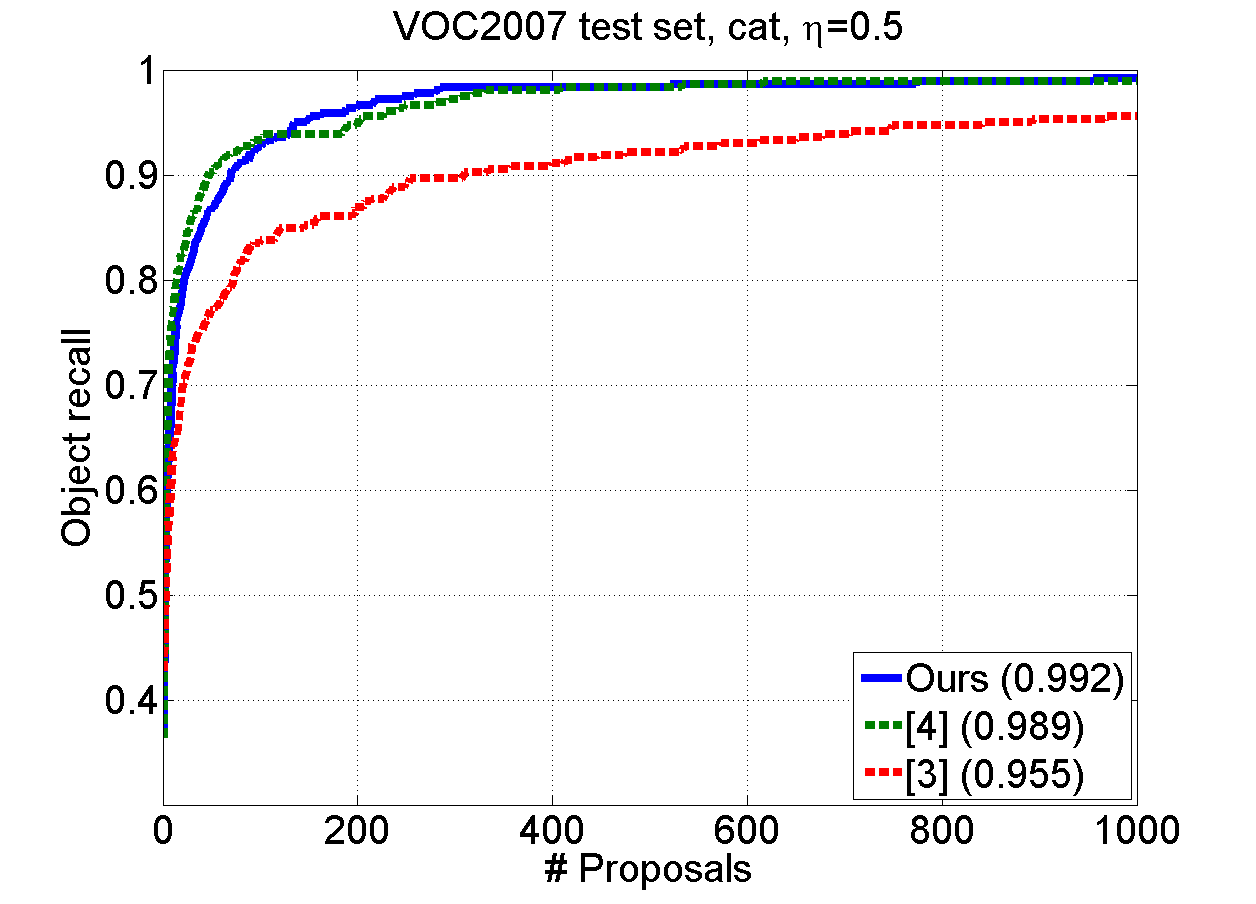}}
 \end{center}
\end{minipage}
\begin{minipage}[b]{0.245\linewidth}
 \begin{center}
 \centerline{\includegraphics[width=1.05\columnwidth]{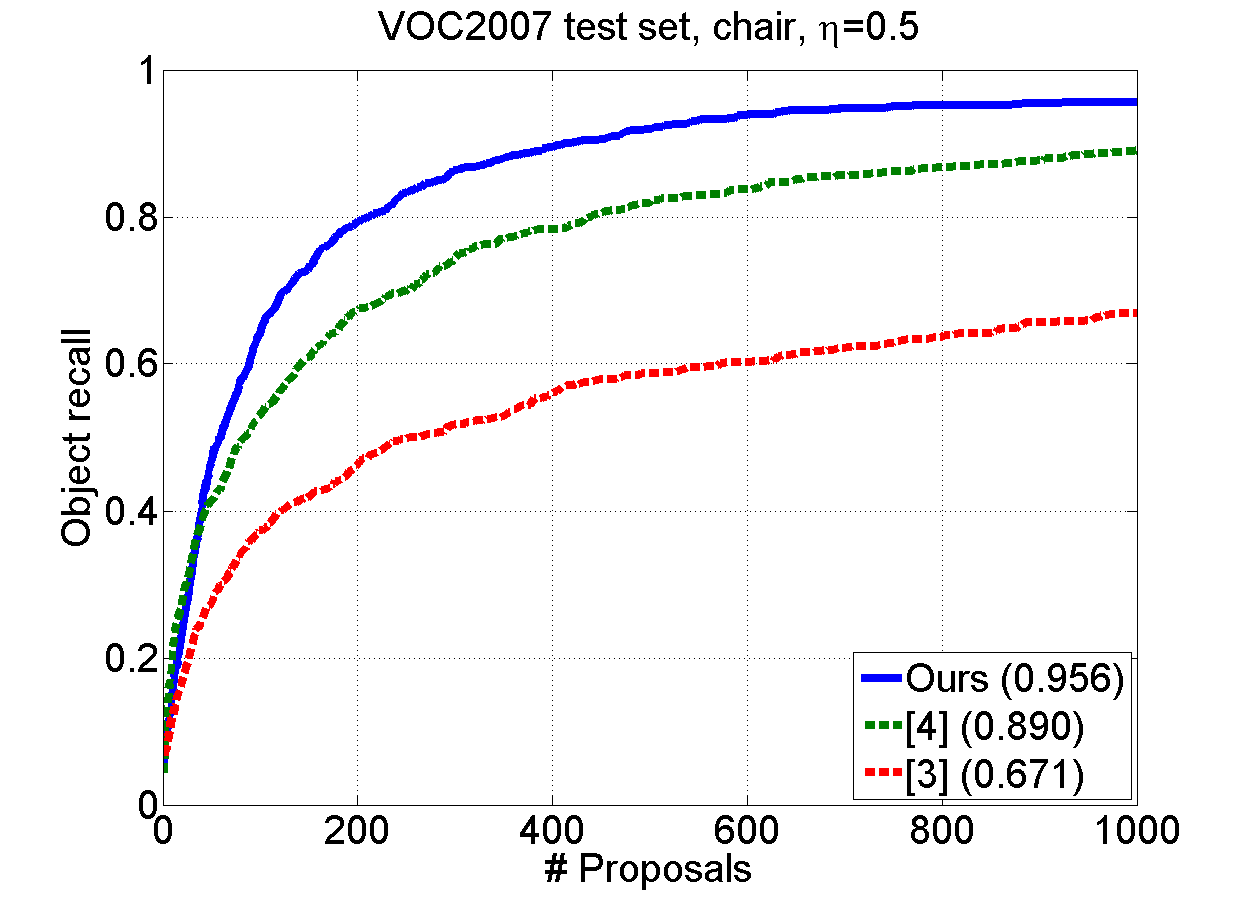}}
 \end{center}
\end{minipage}
\begin{minipage}[b]{0.245\linewidth}
 \begin{center}
 \centerline{\includegraphics[width=1.05\columnwidth]{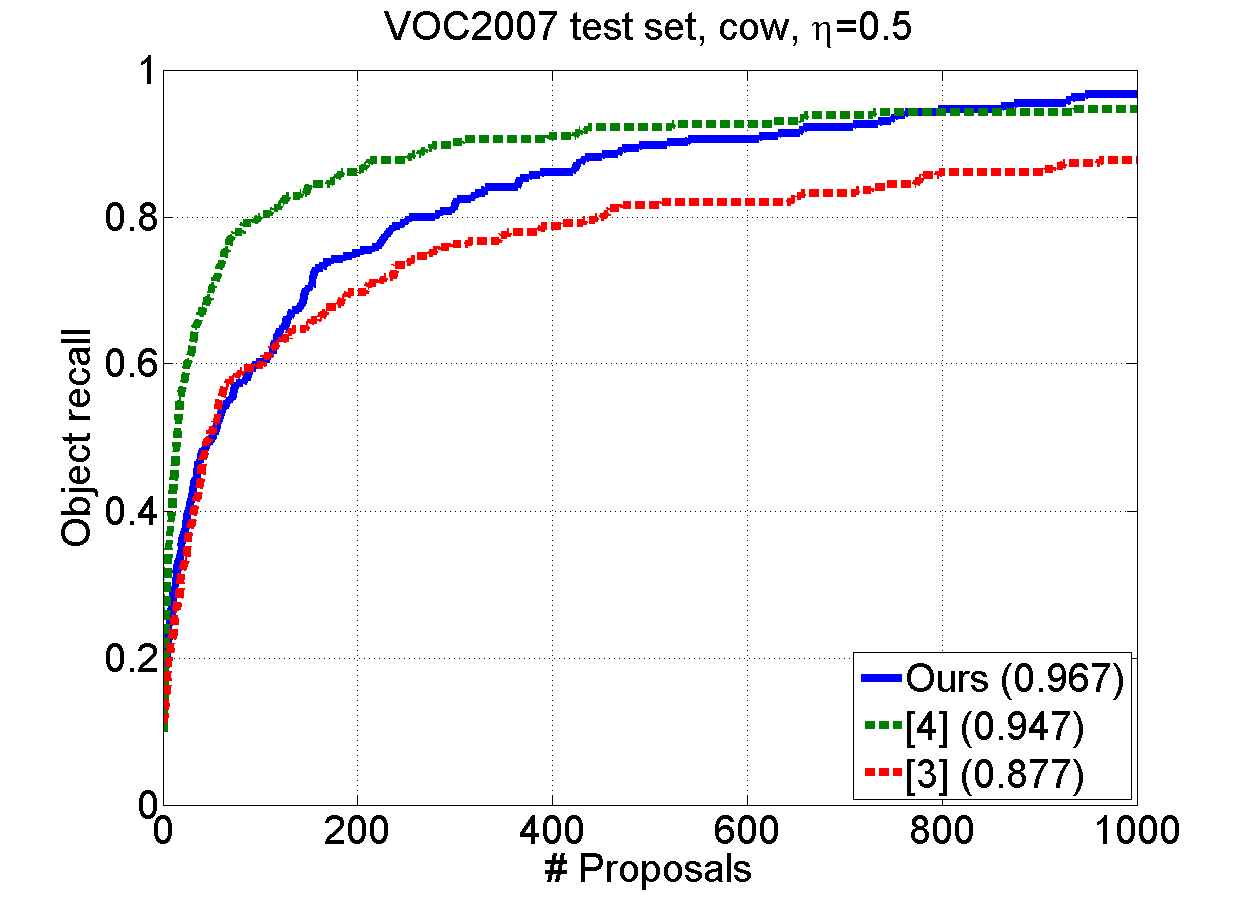}}
 \end{center}
\end{minipage}
\begin{minipage}[b]{0.245\linewidth}
 \begin{center}
 \centerline{\includegraphics[width=1.05\columnwidth]{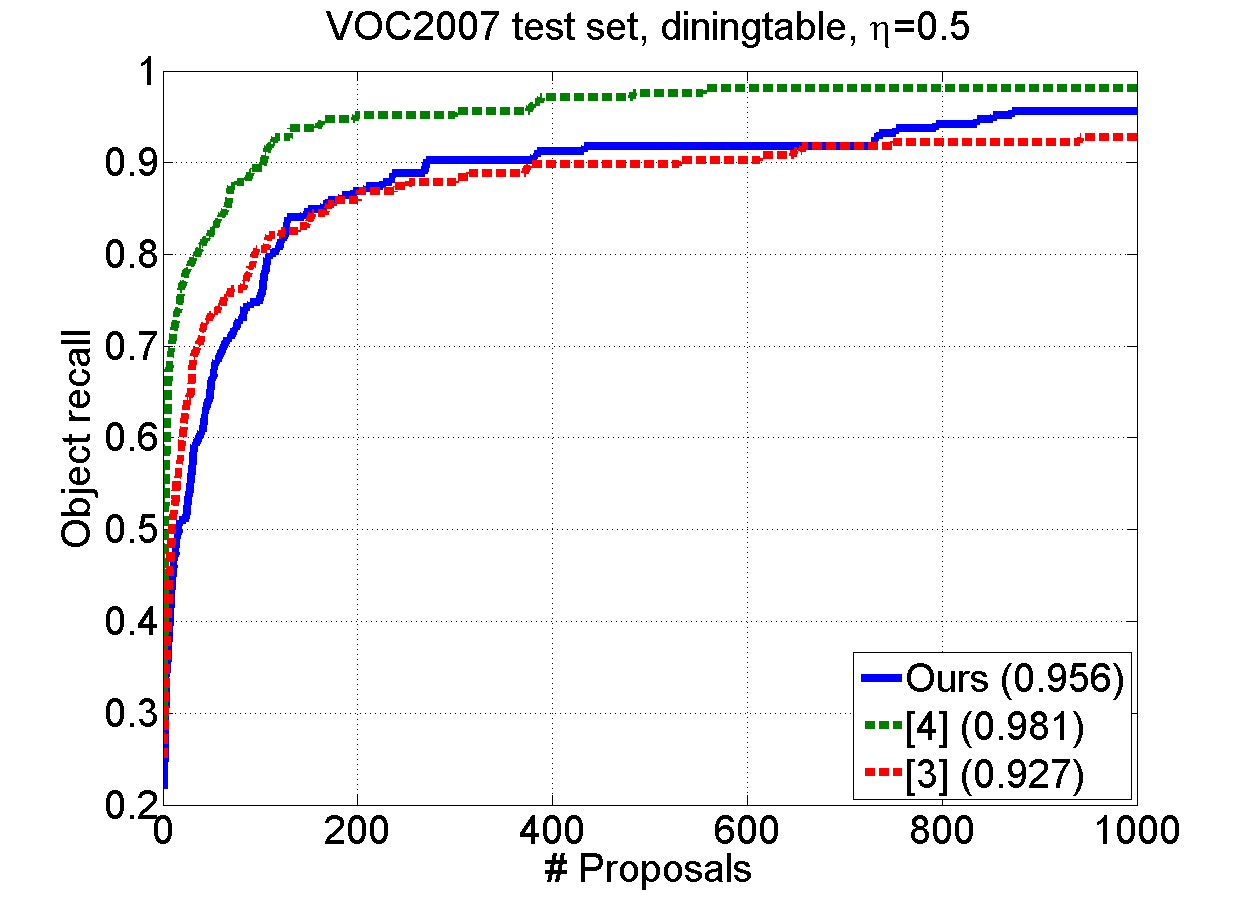}}
 \end{center}
\end{minipage}
\begin{minipage}[b]{0.245\linewidth}
 \begin{center}
 \centerline{\includegraphics[width=1.05\columnwidth]{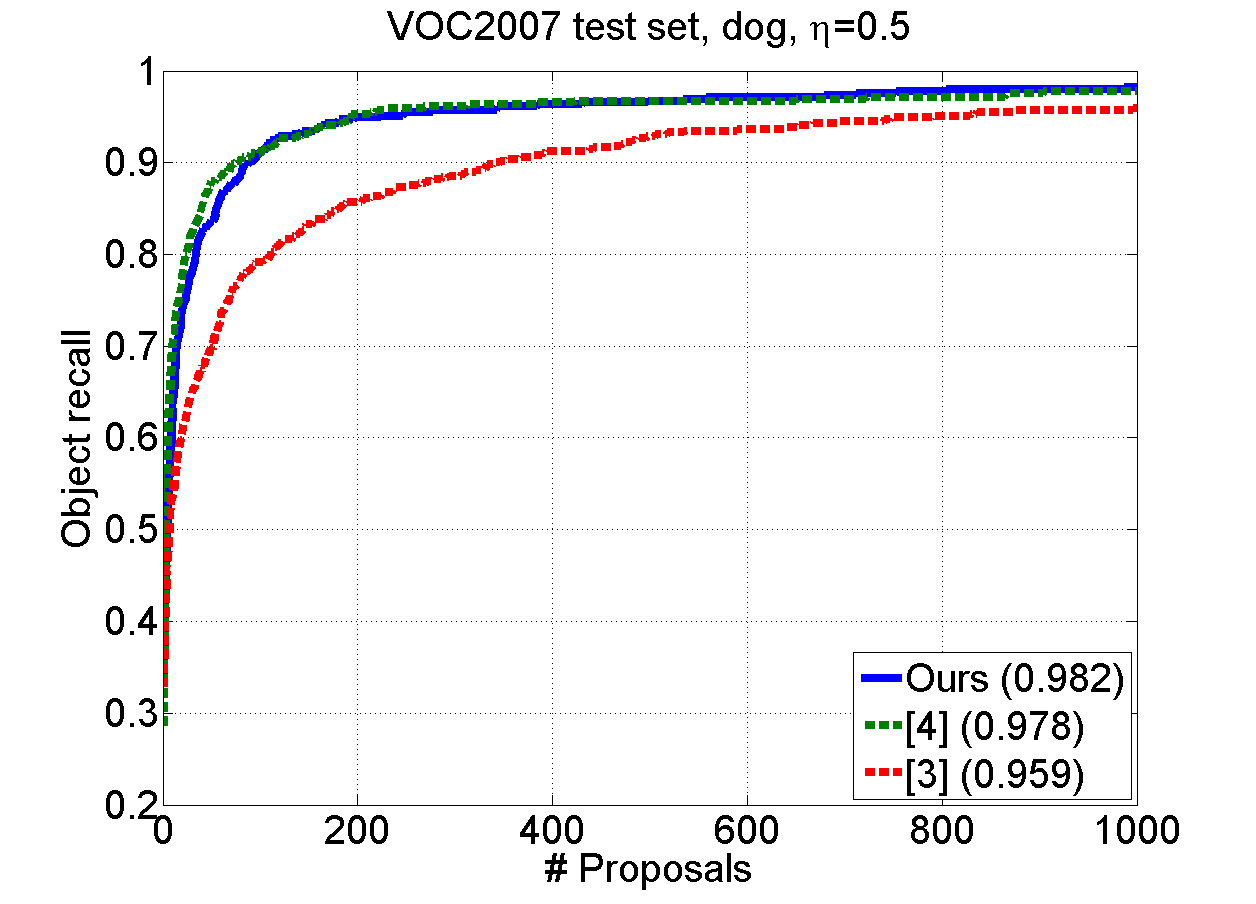}}
 \end{center}
\end{minipage}
\begin{minipage}[b]{0.245\linewidth}
 \begin{center}
 \centerline{\includegraphics[width=1.05\columnwidth]{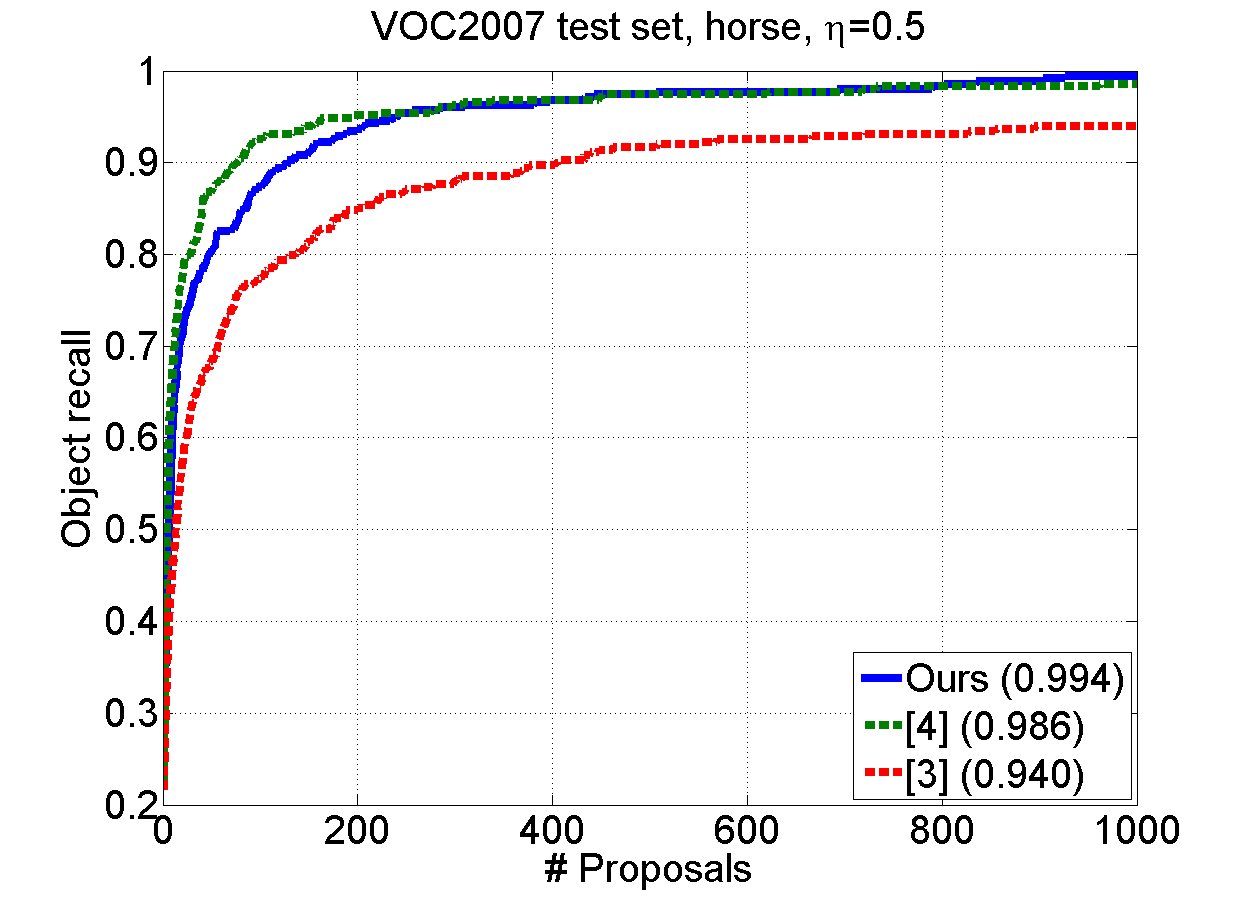}}
 \end{center}
\end{minipage}
\begin{minipage}[b]{0.245\linewidth}
 \begin{center}
 \centerline{\includegraphics[width=1.05\columnwidth]{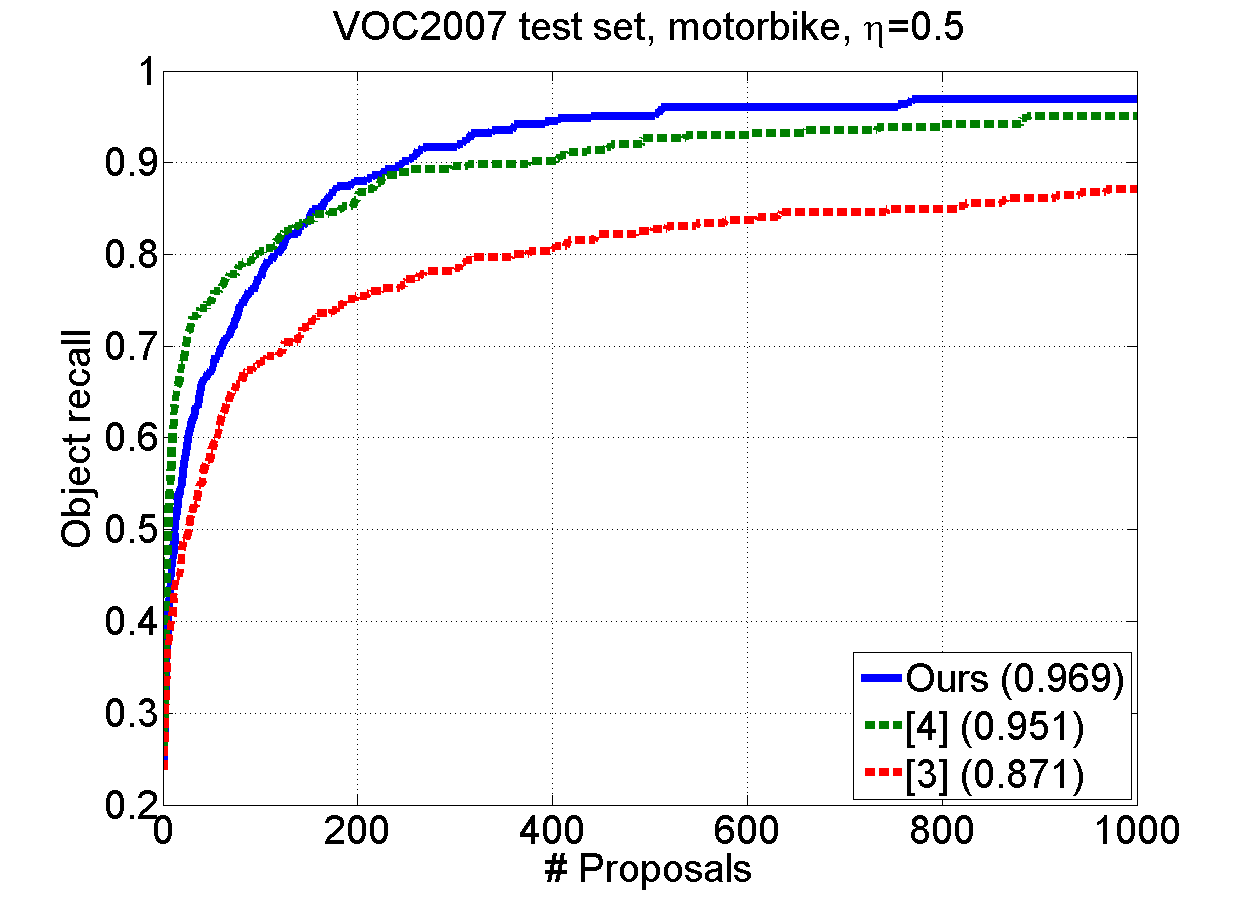}}
 \end{center}
\end{minipage}
\begin{minipage}[b]{0.245\linewidth}
 \begin{center}
 \centerline{\includegraphics[width=1.05\columnwidth]{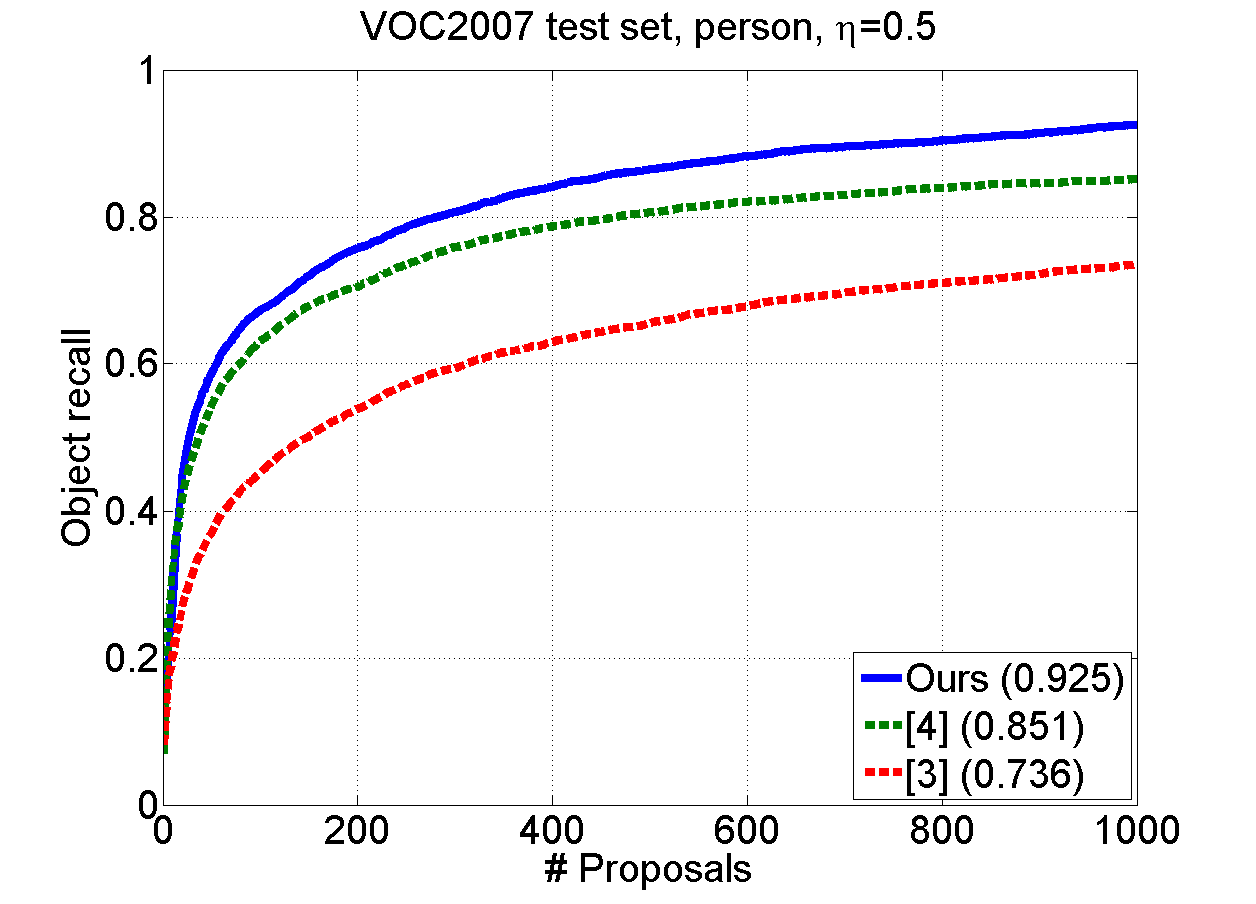}}
 \end{center}
\end{minipage}
\begin{minipage}[b]{0.245\linewidth}
 \begin{center}
 \centerline{\includegraphics[width=1.05\columnwidth]{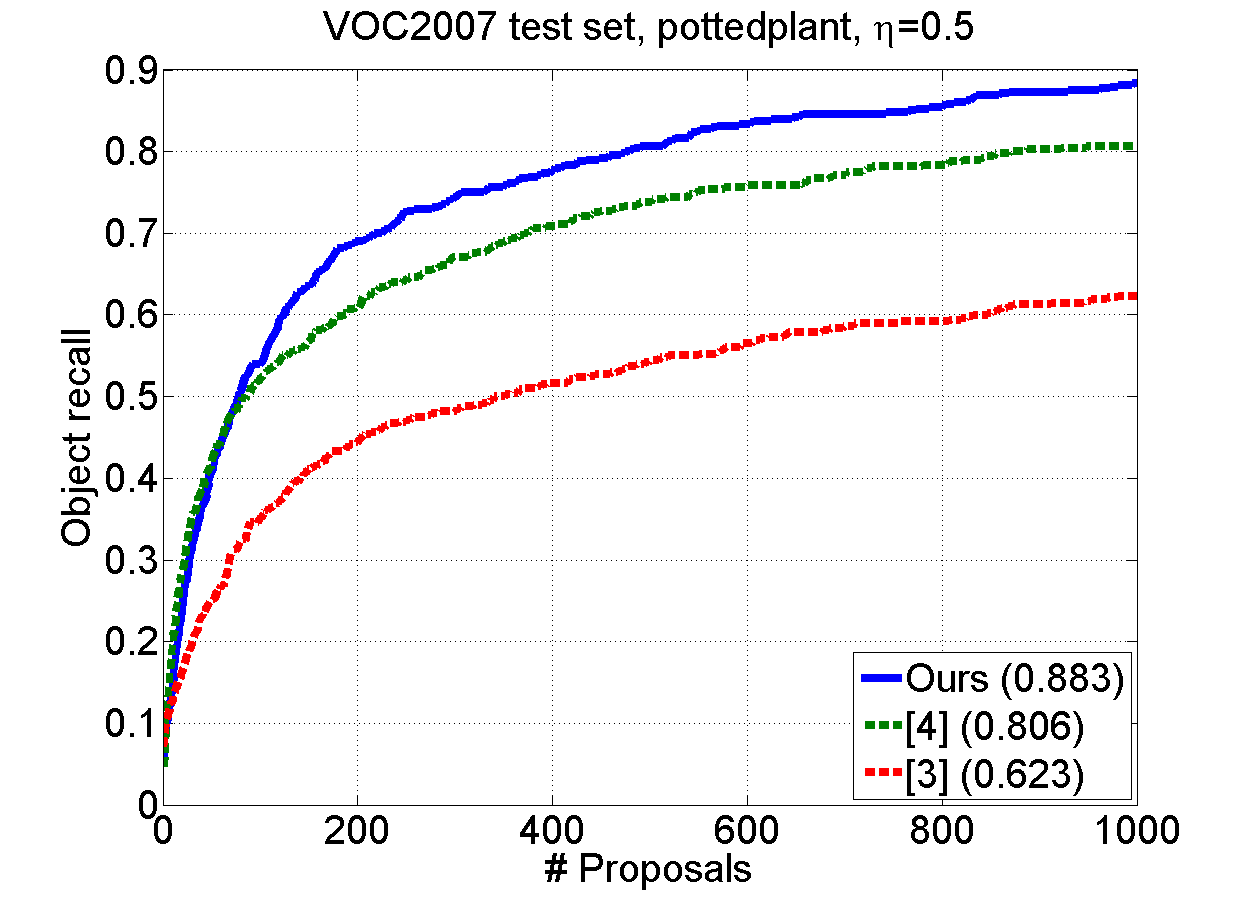}}
 \end{center}
\end{minipage}
\begin{minipage}[b]{0.245\linewidth}
 \begin{center}
 \centerline{\includegraphics[width=1.05\columnwidth]{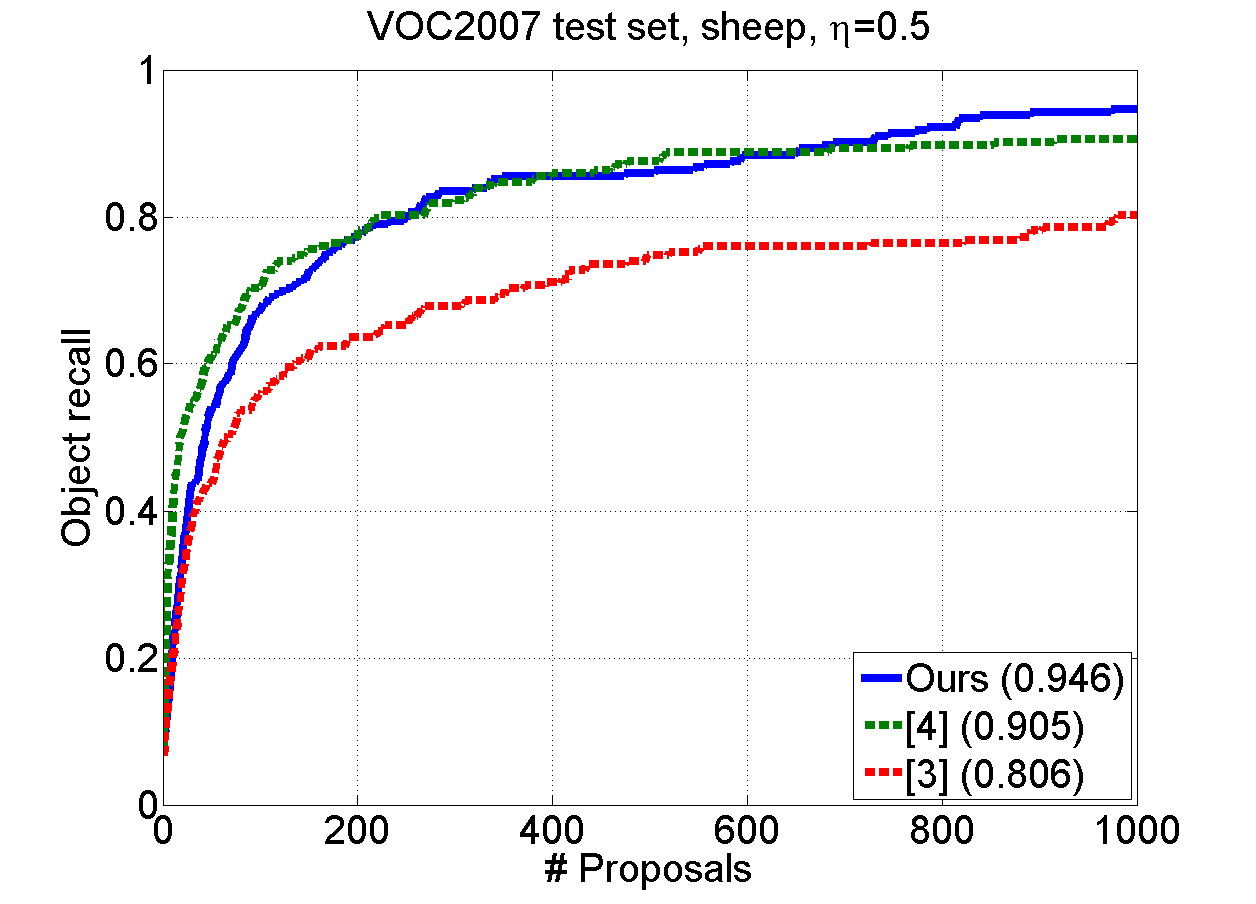}}
 \end{center}
\end{minipage}
\begin{minipage}[b]{0.245\linewidth}
 \begin{center}
 \centerline{\includegraphics[width=1.05\columnwidth]{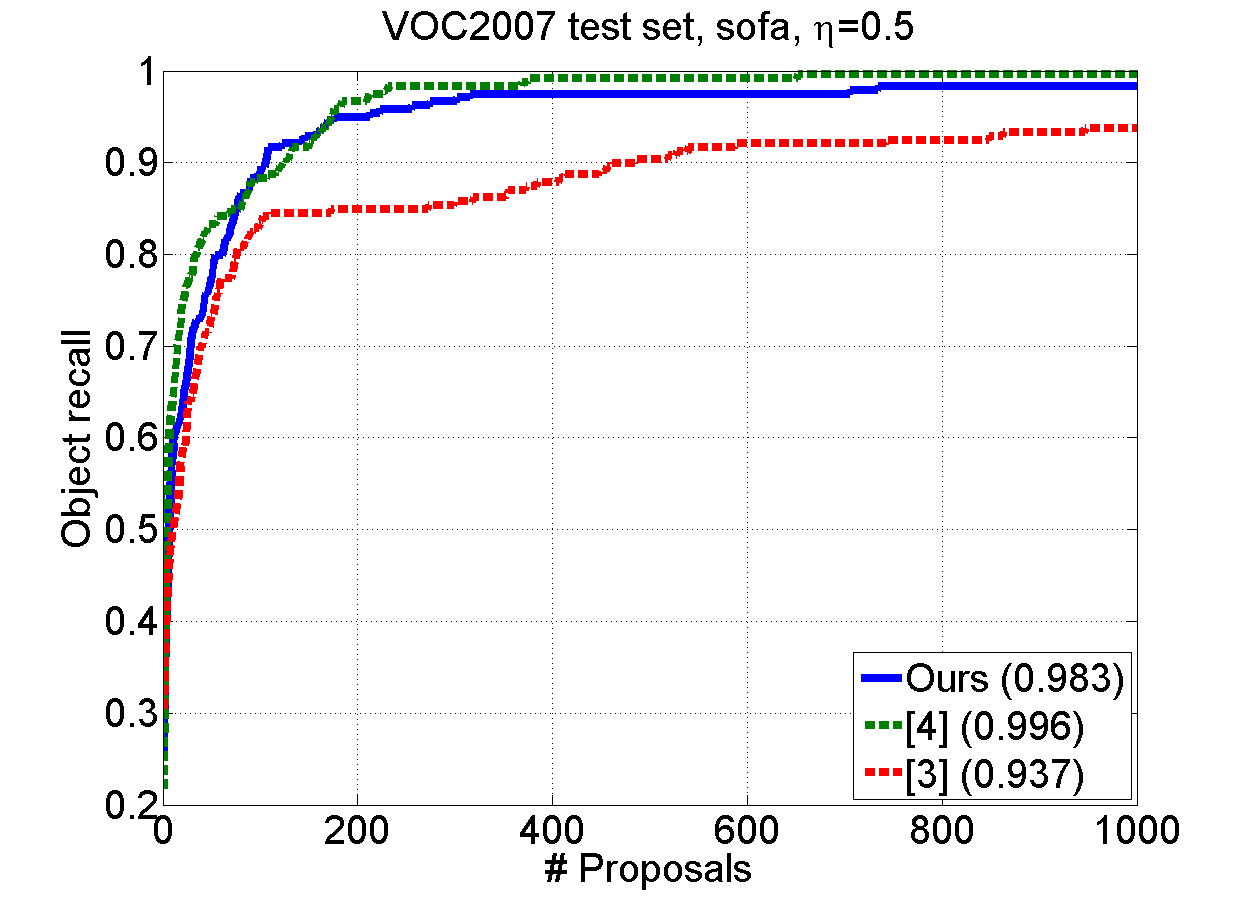}}
 \end{center}
\end{minipage}
\begin{minipage}[b]{0.245\linewidth}
 \begin{center}
 \centerline{\includegraphics[width=1.05\columnwidth]{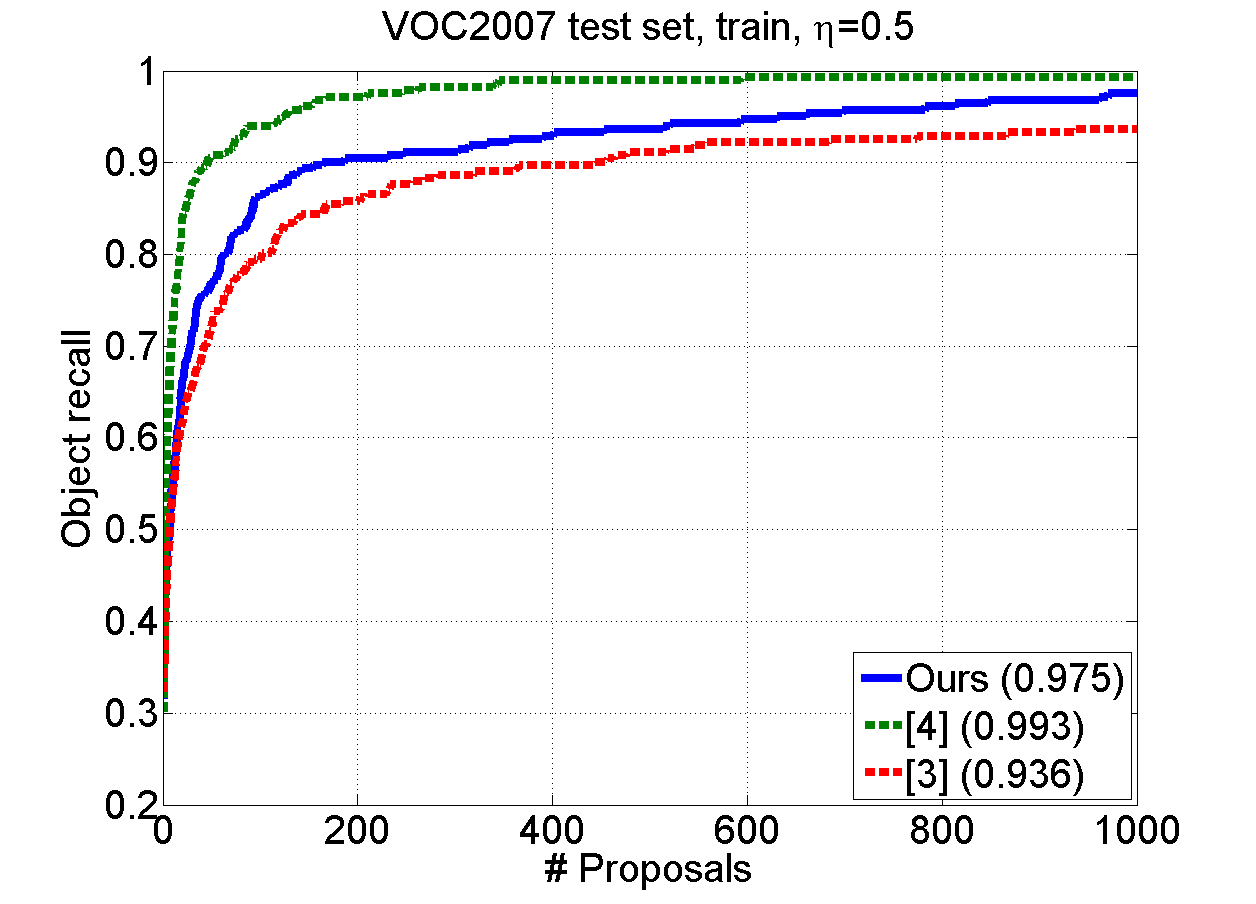}}
 \end{center}
\end{minipage}
\begin{minipage}[b]{0.245\linewidth}
 \begin{center}
 \centerline{\includegraphics[width=1.05\columnwidth]{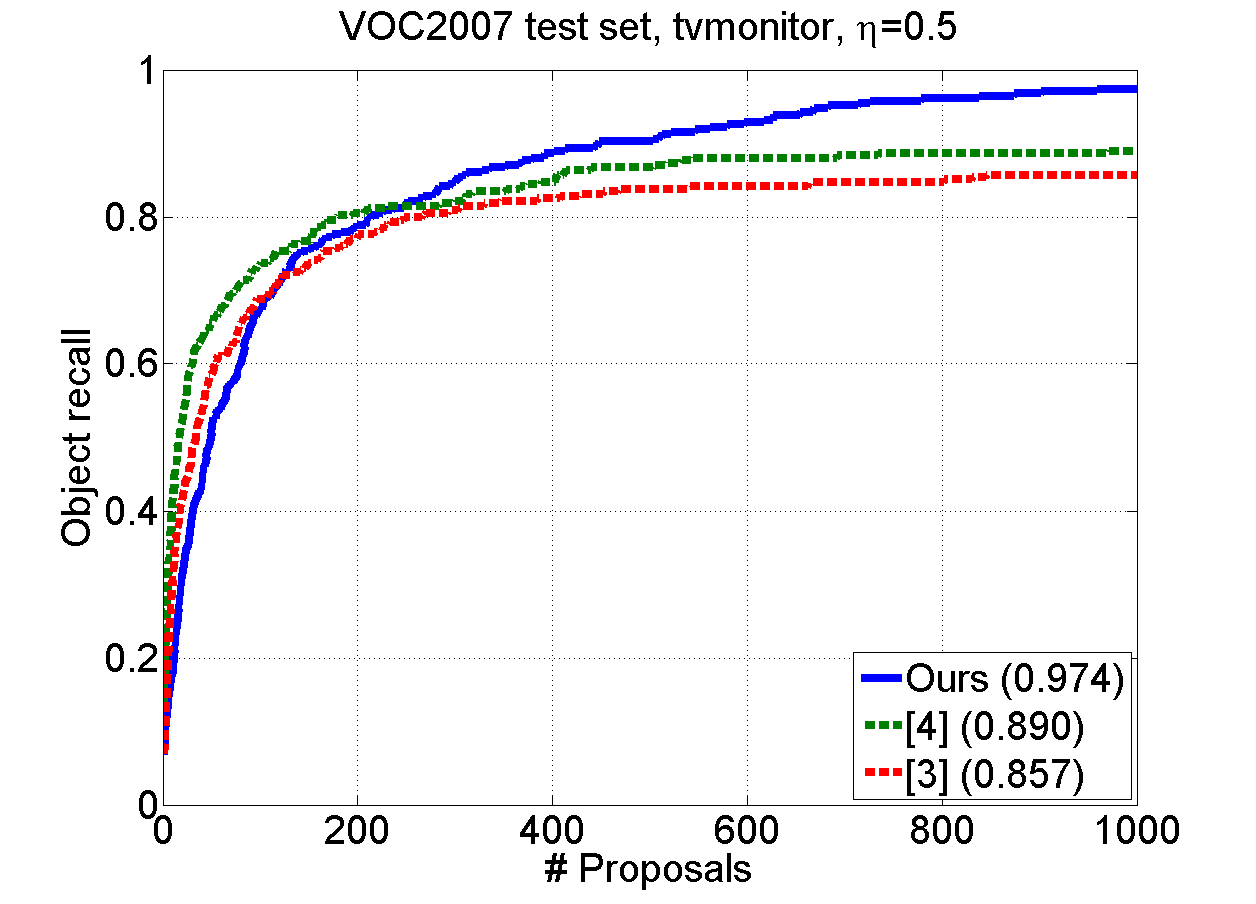}}
 \end{center}
\end{minipage}
\caption{{\footnotesize Comparison of recall-proposal curves using different methods and $\eta=0.5$ on each class in the test dataset of VOC2007. The numbers in brackets are the object recall values for each method using 1000 proposals. Still, in general our method ({\it i.e.} $\ell_1-o/r + \ell_1-o/r$) and \cite{Alexe2012pami} have similar behaviors, and both outperform \cite{Rahtu_iccv11} using 1000 proposals. The mean and standard deviation of the object recall values for our method, \cite{Alexe2012pami,Rahtu_iccv11} are $(95.1\pm 3.5)\%$, $(92.0\pm 6.7)\%$, and $(82.8\pm 12.8)\%$, respectively.}}\label{fig:recall-proposal-class-2007}
\vspace{-0mm}
\end{figure*}

\ifCLASSOPTIONcompsoc
  \section*{Acknowledgments}
\else
  \section*{Acknowledgment}
\fi

We acknowledge support of the EPSRC and financial support was provided by ERC grant ERC-2012-AdG 321162-HELIOS.

\ifCLASSOPTIONcaptionsoff
  \newpage
\fi



\bibliographystyle{IEEEtran}
\bibliography{egbib}
\end{document}